\documentclass{article}
\usepackage[utf8]{inputenc}
\usepackage{amsmath}
\usepackage{amsthm}
\usepackage{amssymb}
\usepackage[round]{natbib}
\usepackage[colorlinks,
            linkcolor=blue,
            anchorcolor=blue,
            citecolor=blue]{hyperref}
\usepackage{graphicx}
\usepackage{subcaption}
\usepackage{enumerate}
\usepackage{fullpage}
\usepackage{tikz}
\usetikzlibrary{angles,arrows,quotes,decorations.pathreplacing}
\usepackage{thmtools}
\usepackage{thm-restate}
\usepackage{cleveref}
\usepackage[ruled]{algorithm2e}
\usepackage{apptools}

\newtheorem{theorem}{Theorem}

\newtheorem{lemma}{Lemma}

\newtheorem{definition}{Definition}

\newtheorem{claim}{Claim}

\newtheorem{assumption}{Assumption}

\AtAppendix{\counterwithin{lemma}{section}}
\AtAppendix{\counterwithin{theorem}{section}}
\AtAppendix{\counterwithin{lemmma}{section}}
\AtAppendix{\counterwithin{claim}{section}}

\newcommand{\norm}[1]{\left\|#1\right\|}
\DeclareMathOperator{\sgn}{sgn}
\DeclareMathOperator{\erf}{erf}
\DeclareMathOperator{\OwenT}{OwenT}

\newcommand\E{\mathbb{E}}
\newcommand\R{\mathbb{R}}
\newcommand\Prob{\mathbb{P}}
\newcommand*{\dif}{\,\mathrm{d}}
\DeclareMathOperator{\arccot}{arccot}
\DeclareMathOperator{\relu}{ReLU}

\newcommand{\email}[1]{\href{mailto:#1}{\color{black} \texttt{#1}}}

\title{A Local Convergence Theory for Mildly Over-Parameterized Two-Layer Neural Network}

\author{
    Mo Zhou \\ Duke University \\ \email{mozhou@cs.duke.edu} \and 
    Rong Ge \\ Duke University \\ \email{rongge@cs.duke.edu} \and 
    Chi Jin \\ Princeton University \\ \email{chij@princeton.edu}}
\date{February 5, 2021}

\begin{document}

\maketitle
\begin{abstract}
While over-parameterization is widely believed to be crucial for the success of optimization for the neural networks, most existing theories on over-parameterization do not fully explain the reason---they either work in the Neural Tangent Kernel regime where neurons don't move much, or require an enormous number of neurons.
In practice, when the data is generated using a teacher neural network, even mildly over-parameterized neural networks can achieve 0 loss and recover the directions of teacher neurons. 
In this paper we develop a local convergence theory for mildly over-parameterized two-layer neural net. We show that as long as the loss is already lower than a threshold (polynomial in relevant parameters), all student neurons in an over-parameterized two-layer neural network will converge to one of teacher neurons, and the loss will go to 0. Our result holds for any number of student neurons as long as it is at least as large as the number of teacher neurons, and our convergence rate is independent of the number of student neurons. A key component of our analysis is the new characterization of local optimization landscape---we show the gradient satisfies a special case of Lojasiewicz property which is different from local strong convexity or PL conditions used in previous work.
\end{abstract}

\section{Introduction}
Recent years, deep learning has achieved great empirical success in a wide range of applications including speech recognition, image detection, natural language processing, game playing, etc. In practice, simple optimization algorithms such as gradient descent (GD) and stochastic gradient descent (SGD) typically already achieve zero training loss.
However, in theory, training deep neural networks remains a challenging problem, as it requires optimizing highly non-convex objective functions. Recent works suggest that over-parameterization is a key to the success of training for neural networks.

One line of work, known as the Neural Tangent Kernels (NTK) \citep{jacot2018neural, chizat2019lazy, du2018gradient,allen2018convergence}, shows that neural network training can get 0 training loss when the network is sufficiently over-parameterized. However, this theory also suggests that the neurons will not move very far from their initial positions, which is often not true for practical neural networks.

Another line of work uses a mean-field limit to analyze two-layer neural networks~\citep{chizat2018global, mei2018mean}. While this type of work would allow neurons to move far, the theoretical results often require the number of neurons to go to infinity, or be exponential in relevant parameters. 
\citet{chizat2019lazy} unified the two lines of work by showing that NTK is equivalent to a lazy training regime (Figure~\ref{fig-ntk}) where the initialization has a very large scale, while mean-field analysis can handle settings where the initialization is smaller.

\begin{figure}[t]
\centering
\begin{minipage}[t]{0.3\textwidth}
\centering
\includegraphics[width=\textwidth,trim=70 250 70 250,clip]{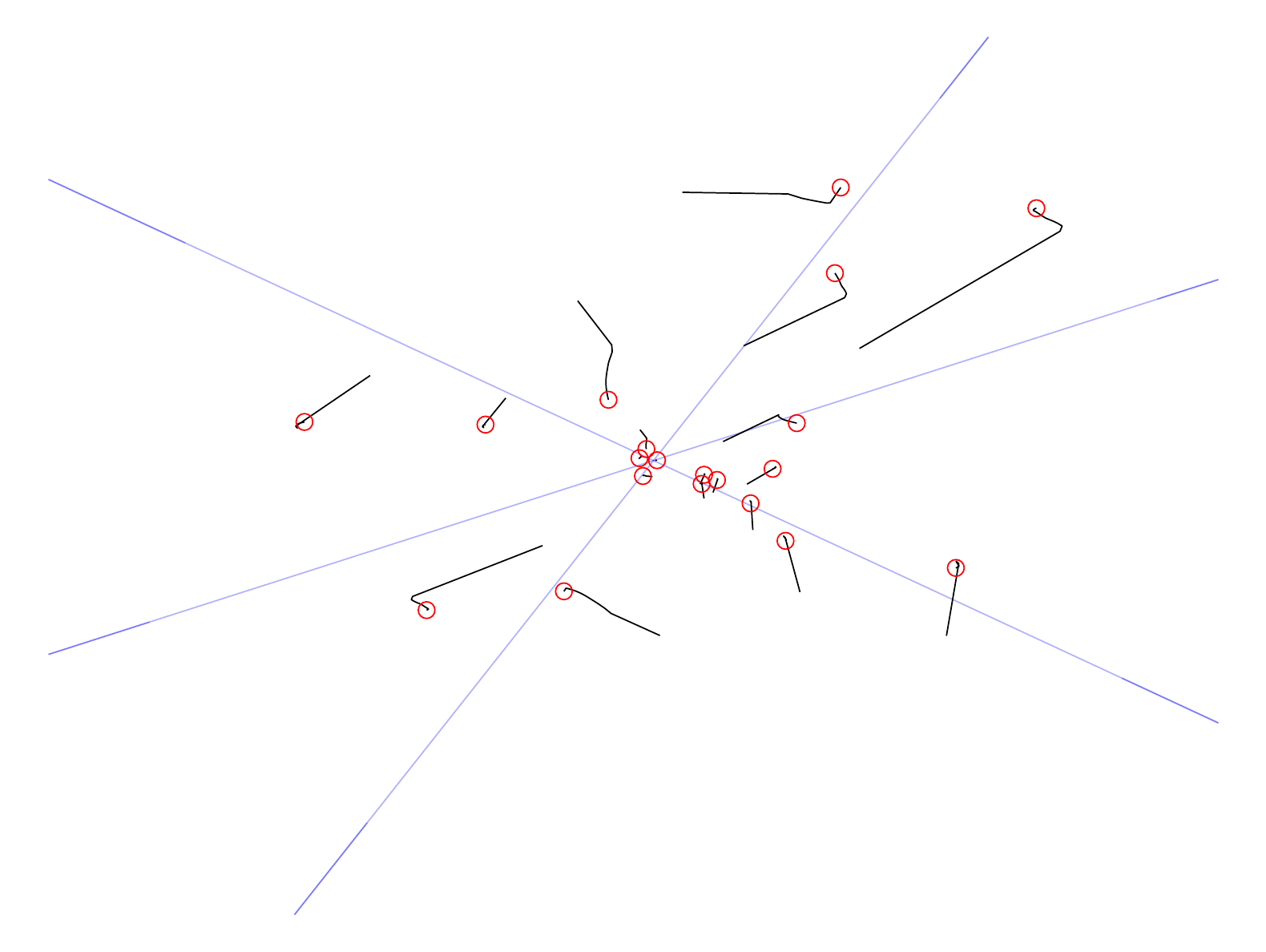}
\caption{NTK regime}
\label{fig-ntk}
\end{minipage}
\hspace{0.1in}
\begin{minipage}[t]{0.3\textwidth}
\centering
\includegraphics[width=\textwidth,trim=70 250 70 250,clip]{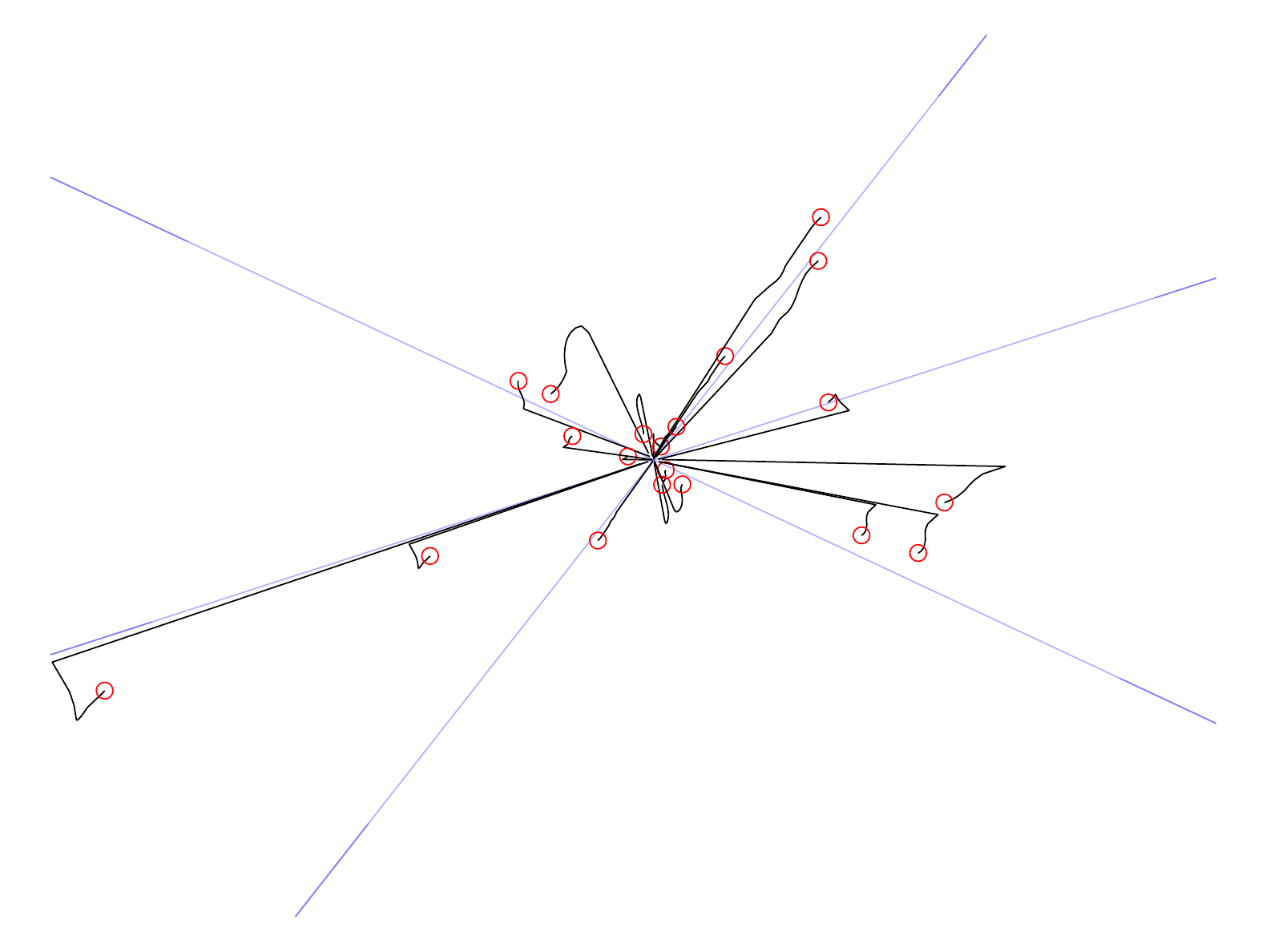}
\caption{Mean-field regime}
\label{fig-mean-field}
\end{minipage}
\hspace{0.1in}
\begin{minipage}[t]{0.3\textwidth}
\centering
\includegraphics[width=\textwidth,trim=70 250 70 250,clip]{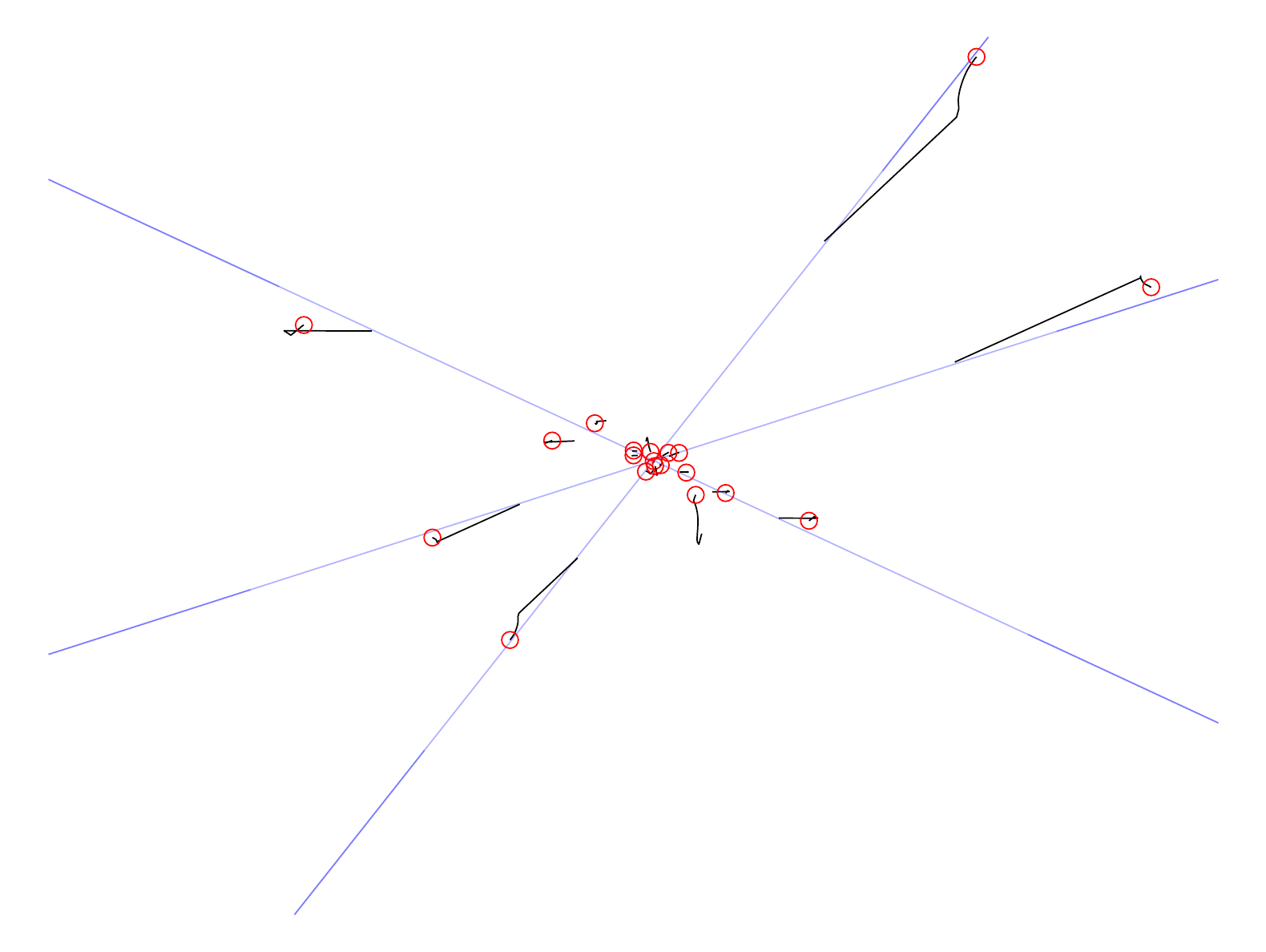}
\caption{Local convergence}
\label{fig-local-convergence}
\end{minipage}

\caption{Training two-layer neural networks in 2 dimension with $m=20$ student neurons and $r=3$ teacher neurons. Blue lines represent the direction of teacher neurons, 
black curves represent the trajectories for each student neuron, and red points represent their end positions.} 
\label{fig-lazy}
\end{figure}

In this paper we consider a simple teacher-student setting, where the training data $(x, y)$ is generated by sampling $x$ from a Gaussian and evaluating $y$ using a ground truth two-layer teacher network. The goal is to train a student network that mimics the behavior of the teacher. 
Figure~\ref{fig-lazy} illustrates the differences between two lines of work in the teacher student setting \--- in the NTK/lazy training regime (Figure~\ref{fig-ntk}) student neurons do not move much, while in the mean field regime (Figure~\ref{fig-mean-field}) student neurons converges to one of the directions of teacher neurons\footnote{Similar empirical observations of student-teacher neuron matching were also known for deeper networks \citep{tian2019student}.}. However, when the number of neurons is small, there are no analysis that shows why student neurons need to match the teacher neurons. 

In fact, the phenomenon that student neurons converge locally and match teacher neurons is not even understood in a simpler {over-parameterization} setting, where initially there are already student neurons close to each teacher neuron (Figure \ref{fig-local-convergence}). \citet{safran2020effects} observed that traditional techniques that rely on local strong convexity or PL conditions cannot be applied here. In this paper we focus on the following natural question:

\begin{center}
\textit{When the initial loss is small, will student neurons always match teacher neurons for an over-parameterized two-layer student net?}
\end{center}

We show that this is indeed true. In particular, we prove
\begin{theorem}[Informal] \label{thm:main:informal}
Given data generated by a two-layer teacher network\footnote{The specific architecture of the network is specified in Definition~\ref{def:teacherstudent} in Section~\ref{prelim}. } with $r$ neurons that are $\Delta$-separated (see Assumption~\ref{as:separate}).
There exists a threshold $\tau = \mbox{poly}(\Delta/r)$ such that when loss is smaller than $\tau$, gradient descent converges to a global optimum where all student neurons match (in direction) one of the teacher neurons.
\end{theorem}

Note that the threshold $\tau$ is independent of the student network size  \--- as long as the initial loss is low, even when the number of student neurons is equal or mildly larger than the number of teacher neurons, gradient descent will still converge to the global optimal solution where all student neurons match teacher neurons. In low dimensions or for simple teacher neurons, we can also give initialization procedures that efficiently finds an initialization with loss smaller than $\tau$.

\subsection{Related Work}
\paragraph{Neural Tangent Kernel (NTK)}
One line of the recent work connects the training of sufficiently over-parameterized neural networks with gradient descent to NTK \citep{jacot2018neural, chizat2019lazy, du2018gradient,du2019gradient,allen2018convergence,cao2019generalization,zou2020gradient,li2018learning,daniely2016toward,arora2019exact,arora2019fine,oymak2020towards,ghorbani2021linearized}. In NTK regime, training neural networks with gradient descent is essentially solving kernel regression with NTK. A key technique in NTK analysis is to restrict the neurons to stay around initialization by choosing large enough width of the neural network. As pointed out in \citet{chizat2019lazy}, neural networks essentially degenerate to linear function and makes the optimization become convex. Instead, our result allows neurons to move away from initialization and recover the ground truth neurons.

\paragraph{Mean-field analysis}
Another line of research uses mean-field approach to analyze the training of infinite-width neural networks \citep{chizat2018global, mei2018mean,mei2019mean,wei2019regularization,nguyen2020rigorous,araujo2019mean,nitanda2017stochastic,sirignano2020mean,rotskoff2018trainability,lu2020mean,fang2020modeling}. In mean-field analysis, they focus on the dynamics of the distribution of neurons, and as the number of hidden neurons goes to infinite, gradient descent becomes Wasserstein gradient flow. Different from NTK regime, neurons can move away from initialization. However, these works often either require exponential (or infinite) number of neurons or only provide exponential convergence rate. 

\paragraph{Local landscape analysis}
Several works have studied the local landscape property around the global minima in the teacher-student two-layer neural network setting. \citet{zhong2017recovery,zhang2019learning} studied the problem in the exact-parameterization case and showed the Hessian around global minima is positive-definite. \citet{chizat2019sparse} studied the over-parameterized neural network with a regularization term. They showed that when loss is small, it satisfies PL condition. Roughly speaking, their analysis relies on several kernels to be positive definite. However, they need to require some non-degenerate conditions that are difficult to verify in the case of neural networks. Further, these kernels become degenerate when the regularization term tends to zero. \citet{safran2020effects} studied the over-parameterization case with orthogonal teacher neurons and showed that neither convexity nor PL condition could hold even in the local region of global minima. This indicates that the analysis discussed above cannot be applied in our over-parameterization setting. In contrast, we could show a slightly different version of PL condition holds when loss is small.

\paragraph{Student-teacher neuron matching and the lottery ticket hypothesis} The lottery ticket hypothesis \citep{frankle2018lottery} showed that it is possible to prune a neural network such that even if training is done only on a small subset of randomly initialized neurons, the final network still achieves good accuracy. Our local convergence result gives a partial explanation of this phenomenon for two-layer teacher/student setting \--- as long as the initialization contains student neurons that are close to each teacher neuron, the training process can converge to a global optimal solution. See more discussions in Section~\ref{sec:lottery}.

\subsection{Outline} In Section~\ref{prelim} we formally define the neural network architecture that we work with. Then in Section~\ref{sec:main_results} we summarize our main results, including the formal version of Theorem~\ref{thm:main:informal}, potential initialization algorithms and generalizations in the setting of polynomial sample sizes. {In Section~\ref{sec:over_challenge}, we illustrate the unique challenges in establishing local convergence results for overparameterized setting}. In Section~\ref{sec-pf-sketch-descent-dir} we sketch the proof of a main lemma that lowerbounds the norm of the gradient, which is the main contribution of this paper. We also show how the main lemma can be used to prove Theorem~\ref{thm:main:informal}. Finally we conclude in Section~\ref{sec:conclusion}.

\section{Preliminaries}\label{prelim}
%\subsection{Problem Setup}
\paragraph{Teacher/student setting for two-layer neural network}
We consider the standard teacher-student setting with Gaussian input $x\sim N(0,I_d)$. We parameterize the teacher/student networks according to the following definition:
\begin{definition}[Teacher-student setup]\label{def:teacherstudent}
Teacher network is parameterized as
$f^*(x) = \sum_{i=1}^r |w_i^{*\top} x|$, where $\{w_i^*\}_{i=1}^r$ ($w_i^*\in \R^d$) are the $r$ teacher neurons. Student network is parameterized as
$f(x) = \sum_{i=1}^{m}\norm{w_i}|w_{i}^\top x|$, where  $\{w_i\}_{i=1}^m$ ($w_i\in \R^d$) are $m$ student neurons ($m\ge r$). 
Denote $W = (w_1, \ldots, w_m)$ as the weight matrix formed by student neurons. The loss function we optimize is the population square loss:
\begin{equation}\label{model}
    \begin{aligned}
    \min_W L(W) = \E_{x\sim N(0,I)}\left[ \frac{1}{2}\left(\sum_{i=1}^{m}\norm{w_i}|w_{i}^\top x| - \sum_{i=1}^r |w_i^{*\top} x|\right)^2\right].
    \end{aligned}
\end{equation}
\end{definition}

\paragraph{Choice of the neural network architecture} Note that we use $\norm{w_i}$ as the top layer weight in the student network. This parameterization of the student network ensures the smoothness of loss function, which helps our analysis as discussed in the later section. The same model was also used in \citet{li2020learning}. Since the two-layer neural network is 2-homogeneous, this restriction of the top layer weight is equivalent to requiring all top-layer weights to be nonnegative. Nonnegativity is important to our result as without this assumption, there might be two student neurons that completely cancel out each other and they may not converge to the direction of any teacher neuron.

We also remark that absolute value function is used as activation function in both teacher network and student network. 
We use absolute value function instead of ReLU (where $\relu(x) = \max\{x,0\}$) to ensure identifiability of our model. 
For ReLU activation, even at global minima, there may have student neuron that does not correspond to any teacher neuron. See the Claim below and more discussions in Section \ref{append-pre}.

\begin{restatable}{claim}{claimrelu}\label{claim-relu}
For problem \eqref{model} with ReLU activation, when loss is zero, there may exist a student neuron whose direction does not match direction of any teacher neuron.
\end{restatable}

On the other hand, using absolute value function as activation we can achieve student-teacher matching at global minima. It is directly obtained by setting $\epsilon=0$ in Lemma~\ref{lem-high-dim-ls-small-improve}. 
\begin{theorem}\label{claim-abs}
For problem \eqref{model} with absolute value activation, when loss is zero, every student neuron's direction must match one of the teacher neuron's direction.
\end{theorem}

We remark that we identify $w$ and $-w$ as the same direction and measure the angle between directions up to a sign (i.e., $\angle(w,v)=\arccos[|w^\top v|/(\norm{w}\norm{v})]$) when using the absolute value function as the activation. It is because absolute value function is an even function so there is no difference between $|w^\top x|$ and $|-w^\top x|$. Note that when ReLU is the activation, we cannot identify $w$ and $-w$ as the same direction as $\relu(w^\top x)\neq \relu(-w^\top x)$.

Note that when the input data is Gaussian (or just symmetric), after first fitting the optimal linear function to the data, a teacher network with ReLU activation becomes a teacher network with absolute value activation. See more discussions in Section \ref{append-relu}.

\paragraph{Assumptions}
We make following two assumptions throughout this paper.

\begin{assumption}[$\Delta$-separation] \label{as:separate}
The teacher neurons $w_1^*, w_2^*,\ldots,w_r^*\in\mathbb{R}^d$ are $\Delta$-separated, i.e., $\angle(w^*_i,w^*_j)\geq \Delta$ for all $i\neq j\in[r]$, where $\angle(w^*_i,w^*_j)=\arccos[|w_i^{*\top} w^*_j|/(\norm{w^*_i}\norm{w^*_j})]$.
\end{assumption}
Intuitively, $\Delta$ indicates the level of difficulty to distinguish the weight vectors $w^*_i$ separately. When $\Delta$ is large, all weight vectors $w^*_i$ of the teacher network are well-separated. When $\Delta$ is small, there may exist two weight vectors that are close to each other. In this case, it may be hard to distinguish between them.

\begin{assumption}[Norm Bounded] \label{as:norm}
The teacher neurons $w_1^*,w_2^*,\ldots,w_r^*\in\mathbb{R}^d$ are norm bounded, i.e., $0<w_{min}\leq \norm{w_i^*}\leq w_{max}$ for all $i\in[r]$.
\end{assumption}

Casual readers may assume that $\Delta,w_{min},w_{max},r$ are constants. In this case, teacher neurons are well-separated and their norms are approximately in the same order. 
\paragraph{Notations}
We will use $[n]$ to denote the set $\{1,2,\ldots,n\}$. For vector $x\in \mathbb{R}^d$, we use $\norm{x}=(\sum_{i=1}^d x_i^2)^{1/2}$ to represent the 2-norm of $x$, and $\bar{x}=x/\norm{x}$ as the unit vector in the direction of $x$. For matrix $A$, we use $\norm{A}_F$ to denote the Frobenius norm of $A$. For vectors $w,v \in \mathbb{R}^d$, denote $\angle(w,v)=\arccos[|w^\top v|/(\norm{w}\norm{v})]\in [0,\pi/2]$ as the angle between $w$ and $v$ (up to a sign).
We will use $\mathbb{I}_S$ as the indicator of the set $S$. Denote the inner product and norm between two function $f$ and $g$ on set $S$ as 
$\langle f, g \rangle_S = \mathbb{E}_{x\sim N(0,I)} [f(x)g(x)\mathbb{I}_S],$ and 
$\norm{f-g}_S^2 = \mathbb{E}_{x\sim N(0,I)} [(f(x) - g(x))^2\mathbb{I}_S].$ When $S=\mathbb{R}^d$, we will omit $S$ and denote it as $\norm{f-g}$.

\section{Main Results}\label{sec:main_results}

In this section we give a formal version of Theorem~\ref{thm:main:informal} and talk about generalizations. We discuss proof ideas for the main results in the Section~\ref{sec-pf-sketch-descent-dir}.

\subsection{Gradient Lower Bound} %\label{sec:main_gradient}
We first present a result that characterizes the landscape of the loss function when the loss is small.
We show that the gradient norm can be lower bounded by the value of the loss function, which is a special case of Łojasiewicz property \citep{lojasiewicz1963propriete}~\footnote{The general Łojasiewicz property is known as $\norm{\nabla f(x)}\geq C (f-f^*)^\alpha$, where $f^*$ is the optimal objective value. Theorem \ref{lem-repara-high-dim-grad-norm} corresponds to the case $\alpha=1$. We remark that the well-known Polyak-Łojasiewicz (PL) property is also a special case of Łojasiewicz property with $\alpha=1/2$, which has been discussed in many earlier works \cite[see e.g.][]{karimi2016linear}}.

\begin{restatable}[Gradient Lower Bound]{theorem}{lemReparaHighDimGradNorm}
\label{lem-repara-high-dim-grad-norm}
    For network and loss function $L(\cdot)$ defined in Definition~\ref{def:teacherstudent}, under Assumptions \ref{as:separate}, \ref{as:norm}, 
    there exists a threshold $\epsilon_0=poly(\Delta,r^{-1},w_{max}^{-1},w_{min})$ such that for any $W$ such that loss $L(W)\leq\epsilon_0$, we have
    \begin{align*}
        \norm{\nabla_{W} L(W)}_F \geq \kappa L(W),
    \end{align*}
    where $\kappa=\Theta(r^{-1/2} w_{max}^{-1/2})$. 
\end{restatable}

This result indicates that when the loss is smaller than certain threshold, gradient is zero only if the loss is zero. Hence there are no spurious stationary points. 

\subsection{Local Convergence} %\label{sec:main_local}
Now we are ready to state the formal version of Theorem~\ref{thm:main:informal}, which gives local convergence for network defined in Definition~\ref{def:teacherstudent}. 
\begin{restatable}[Main Result]{theorem}{lemHighDimLocalConvergence}
\label{lem-repara-high-dim-converge}
    For network and loss function $L(\cdot)$ defined in Definition~\ref{def:teacherstudent}, under Assumption~\ref{as:separate},~\ref{as:norm}, suppose we run gradient descent on objective \eqref{model}:
    \begin{align*}
        w_{i}^{(t+1)} = w_{i}^{(t)} - \eta\nabla_{w_{i}}L(W), \text{~~~for any~~} i\in[m]
    \end{align*}
    there exists a threshold $\epsilon_0=poly(\Delta,r^{-1},w_{max}^{-1},w_{min})$ and $\eta_0=O(r^{-1}w^{-1}_{max})$ such that for any $\epsilon>0$, if initial loss $L(W^{(0)})\leq\epsilon_0$ and step size $\eta\leq\eta_0$, then we have
    $L(W^{(T)})\leq \epsilon$ in $T=O(rw_{max}/(\epsilon\eta))$ steps.
\end{restatable}

The proof relies on two geometric conditions about the landscape of loss function \--- the gradient lowerbound as in Theorem~\ref{lem-repara-high-dim-grad-norm}, and a ``smoothness'' type guarantee for the loss function (Lemma~\ref{lem-high-dim-smooth}). 
See the proof in Section \ref{sec-pf-sketch-local-converge}.

In Section~\ref{sec:over_challenge} we show mild over-parameterization introduces significant difficulty for local convergence results, but the theorem shows that gradient descent can still achieve local convergence despite over-parameterization.
The result only has mild requirement on the over-parameterization. As long as $m\geq r$ and the initial solution has low loss, our local convergence result holds. Moreover, neither  the initial loss requirement or the convergence rate depends on the number of student neuron $m$. They only depend on the intrinsic quantities of the teacher network. This demonstrates that gradient descent could leverage the structure of a teacher network without explicitly knowing information such as its number of neurons $r$ or separation $\Delta$. 

\subsection{Initialization} \label{sec:main_inital}
Next we talk about how one can find a good initialization to get into the local convergence regime.
We first consider a simple random initialization (see Algorithm~\ref{alg:init-1} in Section~\ref{sec-pf-init}), where the directions of neurons are chosen randomly, and the norm of neurons are fitted by least-squares. We also give a more complicated initialization algorithm (subspace initialization, see Algorithm~\ref{alg:init-2} in Section~\ref{sec-pf-init}), which first estimates a subspace spanned by the teacher neurons and then randomly initialize student neurons in the subspace.

\begin{restatable}{theorem}{thmInit}\label{thm:init}
    For network and loss function $L(\cdot)$ defined in Definition~\ref{def:teacherstudent}, under Assumption \ref{as:separate}, \ref{as:norm}, set $\epsilon_{init}=\epsilon_0=poly(\Delta,r^{-1},w_{max}^{-1},w_{min})$, to achieve a good initialization that has $\epsilon_{init}$ initial loss, random initialization requires $O\left((rw_{max}/\sqrt{\epsilon_{init}})^d \cdot r\log(1/\delta)\right)$ student neurons and subspace initialization  requires $O\left((rw_{max}/\sqrt{\epsilon_{init}})^r\cdot r\log(1/\delta)\right)$ student neurons with $\widetilde{O}(dr^6w_{max}^4/\epsilon_{init}^2)$ samples. Suppose we run GD from either initialization, there exists a threshold $\eta_0=O(r^{-1}w_{max}^{-1})$ such that for any $\epsilon>0$, if step size $\eta\leq\eta_0$, with probability $1-\delta$ we have
    $L(W^{(T)})\leq \epsilon$ in $T=O\left(rw_{max}/(\epsilon\eta)\right)$ steps.
\end{restatable}

When $\Delta = \Omega(1/r), w_{min},w_{max} = \Theta(1)$, 
random initialization gives global convergence in polynomial time for $d = O(1)$; while subspace initialization gives global convergence in polynomial time when $r \le O(\log d/\log \log d)$. 

\subsection{Sample Complexity}
%\subsection{Sample Complexity} \label{sec:main_sample}
Previous theorems require us to optimize the population least squares loss directly, which is impossible in practice with finitely many samples. Here we consider a setting where 
we have access to $N$ data points $\{(x_k,y_k)\}_{k=1}^N$ where $x_k$ are i.i.d. sampled from $N(0,I)$ and $y_k=\sum_{i=1}^r|w_i^{*\top} x_k|$. In this case, we can define an empirical loss:
\begin{align*}
    \widehat{L}(W) = \frac{1}{2N} \sum_{k=1}^N \left(\sum_{i=1}^m \norm{w_i}|w_i^\top x_k| - \sum_{i=1}^r|w_i^{*\top} x_k|\right)^2.
\end{align*}

We can show the gradient on the empirical loss is close to the gradient on the population loss when the number of data points $N$ is large enough, which is formalized in Lemma \ref{lem-sample-complexity}. Based on this, we can extend Theorem \ref{lem-repara-high-dim-converge} to stochastic GD with mini-batch on training samples. See the proof and more discussions in Section \ref{sec-pf-sample-complexity}.

\begin{restatable}{theorem}{thmSample}\label{thm-sample}
    For network and loss function $L(\cdot)$ defined in Definition~\ref{def:teacherstudent}, under Assumption \ref{as:separate}, \ref{as:norm}, suppose we run stochastic GD on $W$ with $N$ fresh samples within each mini-batch at every iteration, i.e., for any $i\in[m]$
    \begin{align*}
        w_{i}^{(t+1)} = w_{i}^{(t)} - \eta\nabla_{w_{i}}\widehat{L}_t(W),
    \end{align*}
    where $\widehat{L}_t(\cdot)$ is the empirical loss using the N samples at iteration $t$. Then, there exists a threshold $\epsilon_0=poly(\Delta,r^{-1},w_{max}^{-1},w_{min})$ and $\eta_0=O(r^{-1}w_{max}^{-1})$ such that for any $\epsilon>0$, if initial loss $L(W^{(0)})\leq\epsilon_0$, step size $\eta\leq\eta_0$ and batch size $N\geq O\left(r^5w_{max}^5d^2\eta^{-1}\epsilon^{-3}\delta^{-1}\right)$, then with probability $1-\delta$ we have
    $L(W^{(T)})\leq \epsilon$ in $T=O\left(rw_{max}/(\epsilon\eta)\right)$ steps.
\end{restatable}

This shows that our local convergence results are robust even when we only use polynomially many samples.

\section{Exact-Parameterization v.s. Over-Parameterization} \label{sec:over_challenge}
Before talking about our proof ideas, we first highlight the key difference between the exact-parameterization setting and the over-parameterization setting, and the unique challenges in the latter setting. 
Consider a simple example where there is only one teacher neuron $w^*$. In the exact-parameterization setting, there is only one student neuron $w$. Suppose $w$ is close to $w^*$ within the distance $\delta$ in the sense $\angle(w,w^*)= \delta$ under the normalization $\norm{w} = \norm{w^*} = 1$. Then it is not difficult to show the loss \eqref{model} is $\Theta(\delta^2)$. This illustrates that the Hessian around global minima is non-degenerate and positive definite (which is formally established in \citet{zhong2017recovery,zhang2019learning}). Such properties immediately imply that the loss is locally strongly convex around the global minima, therefore gradient descent initialized in a small local region can converge to the global minima.

In sharp contrast, over-parameterization introduces a significantly more complicated geometry around the global minimum. Consider again the simple example in the over-parameterization setting where there are two student neurons and one teacher neuron (Figure \ref{fig-warm-up}). The claim below shows that when there are two student neurons that are $\delta$-close to teacher neuron, the loss can be as small as $\Theta(\delta^3)$ instead.

\begin{restatable}{claim}{claimoverpara} \label{claim:over-para}
Suppose teacher neuron $w^*$ and two student neurons $w_1$, $w_2$ belong to the same hyperplane with position shown in Figure \ref{fig-warm-up}. If $\norm{w^*}=1$, $\angle(w_1,w^*)=\angle(w_2,w^*)=\delta$, $\norm{w_1}=\norm{w_2}=1/\sqrt{2\cos\delta}$ (so that $\norm{w_1}w_1+\norm{w_2}w_2=w^*$) for a small enough constant $\delta$, then we have loss $L(W)=\Theta(\delta^3)$.
\end{restatable}

\begin{figure}
    \centering
        \begin{tikzpicture}
            \draw[->,-latex] (0,0) -- (0,2)
            node[right] {$w^*$};
            \draw[->,-latex] (0,0) -- (-0.7,1.87)
            node[left] {$w_1$};
            \draw[->,-latex] (0,0) -- (0.7,1.87)
            node[right] {$w_2$};
            \draw (0,2) coordinate (A) - - (0,0) coordinate (B) - - (-0.7,1.87) coordinate (C) pic [draw,"$\delta$",angle eccentricity=1.5] {angle};
            \draw (0.7,1.87) coordinate (A) - - (0,0) coordinate (B) - - (0,2) coordinate (C) pic [draw,"$\delta$",angle eccentricity=1.5] {angle};
        \end{tikzpicture}
    \caption{Warm-up Example: one teacher, two students}
    \label{fig-warm-up}
\end{figure}
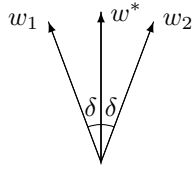

Claim \ref{claim:over-para} indicates that the over-parameterization case is clearly different from the exact-parameterization case. The scale $L(W)=\Theta(\delta^3)$ implies that the Hessian at the global minima must be degenerate, therefore the loss is not locally strongly convex. Intuitively, this is because that in the exact-parameterization case, we only have a set of isolated global minima. While in the over-parameterization setting, different global minima are connected and the loss could be small when the ``average'' student neuron ($\norm{w_1}w_1+\norm{w_2}w_2$) matches the teacher neuron. 
The observation above showcases the unique challenge in the over-parameterization setting. A novel geometric characterization as well as corresponding analyses is necessary, which will be developed in Section~\ref{sec-pf-sketch-descent-dir}.

\section{Proof Overview: Local Geometry and Gradient Lower Bound}
\label{sec-pf-sketch-descent-dir}

In this section, we provide proof sketch for Theorem \ref{lem-repara-high-dim-grad-norm}. First, we explain how the gradient lowerbound can be reduced to constructing a descend direction (Section~\ref{subsec:gradlowerbound}). Then in Section~\ref{subsec:descentdirection} we explain the intuitions for why the descent direction works. To prove our descent direction works, first we show that small loss implies that every teacher neuron has at least one nearby student neuron by constructing test functions (Section~\ref{subsec:testfunction}), then we exploit the notion of ``average neuron'' to decompose the residual into two terms (Section~\ref{subsec:residualdecomp}) and bound them separately. 
At the end (Section~\ref{sec-pf-sketch-local-converge}) we show how one can use Theorem~\ref{lem-repara-high-dim-grad-norm} to prove the local convergence result in Theorem~\ref{lem-repara-high-dim-converge}.

We also introduce following notations about the partition of student neurons that will be used in our analysis. For every teacher neuron $w_i^*$, denote $T_i$ as the set of student neurons that are closer to $w_i^*$ than other teacher neurons (break the tie arbitrarily), i.e.,  $T_i=\{j\in[m]|\angle(w_j,w_i^*)\leq \angle(w_j,w_k^*) \text{ for all } k \neq i \}$. It is easy to see that $\cup_{i\in [r]}T_i$ is the set of all student neurons, and $T_i\cap T_j=\varnothing$ for all $i\neq j\in[r]$. Also denote $T_{i}(\delta)=\{j\in T_i|\angle(w_{j},w_i^*)\leq \delta\}$ as the set of student neurons $w_{j}$ that are $\delta$-close to teacher neuron $w_i^*$.  Finally we use  $\delta_{j} = \angle(w_i^*,w_{j})$ to denote the angle between $w_i^*$ and $w_j$ for all $j\in T_i$.

\subsection{Gradient Lowerbound (Theorem \ref{lem-repara-high-dim-grad-norm}): Constructing Descent Direction}
\label{subsec:gradlowerbound}

We discuss how to obtain the gradient lower bound (Theorem \ref{lem-repara-high-dim-grad-norm}). 
We construct a ``descent direction $g(W)$'' that is correlated to the gradient and prove that $\langle \nabla L(W), g(W) \rangle \ge L(W) \ge0$, this directly implies $\norm{ \nabla L(W)} \ge L(W)/\norm{g(W)}$, which provides the result in form of Theorem \ref{lem-repara-high-dim-grad-norm}. Intuitively, we group the neurons into $r+1$ categories, where $T_i(\delta_{\max})$ represents student neurons that are within angle $\delta_{\max}$ with $i$-th teacher neuron (where $\delta_{\max}$ is a parameter chosen later), and the $r+1$-th group consists of neurons not close to any teacher neuron. The direction we construct will move student neurons closer to their corresponding teacher neurons, or 0 if it is in the last group. Formally, we prove the following lemma on descent direction.

\begin{restatable}[Descent Direction]{lemmma}{lemReparaHighDimSmallLossDescentDir}
\label{lem-repara-high-dim-small-loss-descent-dir}
    For network and loss function $L(\cdot)$ defined in Definition~\ref{def:teacherstudent}, under Assumption~\ref{as:separate},~\ref{as:norm}, there exists a threshold $\epsilon_0=poly(\Delta,r^{-1},w_{max}^{-1},w_{min})$ such that for any $W$ satisfying loss $\epsilon\triangleq L(W) \leq\epsilon_0$, we have
    \begin{align*}
        \sum_{i=1}^r\sum_{j\in T_i} \langle\nabla_{w_{j}}L(W),(I+\bar{w}_j\bar{w}_j^\top)^{-1}(w_{j}-q_{ij}\sgn(w_j^\top w_i^*)w_i^*)\rangle
        \geq L(W),
    \end{align*}
    where $\{q_{ij}\}_{i \in [r], j \in[m]}$ is any sequence that satisfies (1) $q_{ij}\geq 0$ for all $(i, j)$,  (2) $q_{ij}=0$ if $j\not\in  T_{i}(\delta_{max})$; and (3) $\sum_{j\in T_{i}(\delta_{max})} q_{ij}\norm{w_j}=1$ for all $i\in[r]$, where $\delta_{\max} = \Theta(rw_{max}w_{min}^{-5/3} \cdot \epsilon^{1/3})$.
\end{restatable}

In Lemma \ref{lem-repara-high-dim-small-loss-descent-dir}, the scalar $q_{ij}$ describes how much fraction of teacher neuron $w_i^*$ that student $w_j$ should target to approximate. Recall that our activation is absolute value function, so there is a symmetry between $w$ and $-w$ and we identify them as the same neuron. Therefore, $(w_{j}-q_{ij}\sgn(w_j^\top w_i^*)w_i^*)$ can be unsterstood as the difference between the current student neuron and its ``optima''. We further multiply it by matrix $(I+\bar{w}_j\bar{w}_j^\top)^{-1}$ due to technical reason raised by our parameterization of student neural network as $f(x) = \sum_{i=1}^{m}\norm{w_i}|w_{i}^\top x|$ instead of $f(x) = \sum_{i=1}^{m}|w_{i}^\top x|$ (see Section \ref{prelim} for more details).

\subsection{Descent Direction (Lemma \ref{lem-repara-high-dim-small-loss-descent-dir}): High-level Ideas}
\label{subsec:descentdirection}

We briefly explain why the descent direction we constructed would reduce loss function in this section. First, we introduce a quantity which plays a key role in the remaining of Section \ref{sec-pf-sketch-descent-dir}:
\begin{equation} \label{eq:average_distance}
\hat{w}_i \triangleq \sum_{j\in T_i}\norm{w_j}w_{j}; \quad
v_i\triangleq \hat{w}_i - w_i^*    \text{~~for~~} i \in [r]
\end{equation}
Intuitively, $\hat{w}_i$ represents the ``average neuron'' for student neurons close to teacher neuron $w_i$; $v_i$ represents the difference between a teacher neuron and its corresponding average student neuron. Note that in the above definition of ``average neuron'' $\hat{w}_i$, there is an ambiguity of the direction of the neurons $w_j$ due to the symmetry of the absolute value activation. To break this ambiguity, we will assume the direction of $w_j$ always has positive correlation with $w_i^*$, i.e., $w_j^\top w_i^*\ge 0$. This is without of generality because absolute value function is an even function so that intuitively we can replace $w_j$ with $-w_j$ when $w_j^\top w_i^*<0$ without affecting anything else. In fact, we can show if there is a descent direction $(g(w_1),\ldots,g(w_m))$ for $(w_1,\ldots,w_m)$, then for arbitrary choice of $a_i\in\{\pm 1\}$ there is a descent direction $(a_1g(w_1),\ldots,a_mg(w_m))$ for $(a_1w_1,\ldots,a_m w_m)$. We defer the proof to the beginning of Section~\ref{sec-pf-descent-dir}.

In order to prove the Lemma \ref{lem-repara-high-dim-small-loss-descent-dir}, a main challenge is to properly exploit the precondition that loss $L(W)$ is small. Our strategy is to establish an intermediate result, and prove the lemma by the following three steps:
\begin{enumerate}
    \item Show that when the loss $L(W) \le \epsilon$ , for each teacher neuron there is at least one student neuron that is $\delta_{\max} = \Omega(rw_{max}w_{min}^{-5/3} \cdot \epsilon^{1/3})$ close to the teacher neuron.
    \item Show that loss $L(W)$ is small implies that the $\norm{v_i}$ is small for all $i\in[r]$ (average neuron close to teacher neuron).
    \item Show that 1 and 2 imply the conclusion of Lemma \ref{lem-repara-high-dim-small-loss-descent-dir}.
\end{enumerate}

In fact, the third step directly follows from our choice of ``descent direction'', and algebraic computation. We refer readers to Lemma \ref{lem-high-dim-correlation-lower-bound} in Appendix \ref{sec-pf-correlation-lower-bound} for details.

For the first step, we formalize it as the following lemma:

\begin{restatable}{lemmma}{lemHighDimLossDelta}
\label{lem-high-dim-loss-delta}
Under Assumption \ref{as:separate}, \ref{as:norm}, there exists a threshold $\epsilon_0=poly(\Delta,r^{-1},w_{max}^{-1},w_{min})$ and an absolute constant $C$ such that for any $W$ satisfying loss $\epsilon\triangleq L(W) \leq\epsilon_0$, if we choose $\delta_{\max} \ge C \cdot  rw_{max}w_{min}^{-5/3} \cdot \epsilon^{1/3}$, then we have 
$|T_{i}(\delta_{max})| \ge 1$ and $\sum_{j\in T_i(\delta_{max})}\norm{w_j}^2\geq \frac{1}{2}\norm{w_i^*}$ for all $i\in [r]$.
\end{restatable}

The proof uses idea of test functions to lowerbound the loss, and will be discussed later in Section~\ref{subsec:testfunction}. 

Now that each teacher neuron has at least one nearby student neuron, average neuron $\hat{w}_i$ is never 0. Our second step shows that the average student neuron is always close to teacher neuron when the loss is small:

\begin{restatable}[Average Student Is Close to Teacher]{lemmma}{lemHighDimSumStudentSmall} \label{lem-high-dim-sum-student-small}
Under Assumption \ref{as:separate}, \ref{as:norm},
there exists a threshold $\epsilon_0=poly(\Delta,r^{-1},w_{max}^{-1},w_{min})$ such that for any $W$ satisfying loss $\epsilon\triangleq L(W) \leq\epsilon_0$, we have $\{v_i\}_{i\in[r]}$ (defined in~\eqref{eq:average_distance}) satisfy $\norm{v_i} \le poly(r, w_{\max}, \Delta^{-1}) \epsilon^{3/8}$ for all $i\in[r]$.
\end{restatable}

The proof of this relies on a characterization of the loss, which we describe in Section~\ref{subsec:residualdecomp}.

\subsection{Local Property I (Lemma \ref{lem-high-dim-loss-delta}): Lowerbounding Loss using Test Functions}
\label{subsec:testfunction}

To prove Lemma~\ref{lem-high-dim-loss-delta}, we show its contraposition \--- if there is a teacher neuron that does not have a nearby student neuron, then the loss must be large.

In order to lowerbound the loss, we use the idea of constructing a test function. The similar idea has also been used in maximum mean discrepancy (MMD) \citep{gretton2012kernel}, where they use test function to distinguish two distribution. Here, we try to explicitly construct a test function so that its correlation with the residual can be lower bounded. Formally, when there is a teacher neuron without any student neuron within angle $\delta$, we construct a test function $h$ such that
$\langle R(x), h(x)\rangle/\|h(x)\| \geq \Omega(w_{min}^{5/2}r^{-3/2}w_{max}^{-3/2}\cdot\delta^{3/2})$, where $R(x)=\sum_{i=1}^m \norm{w_i}|w_i^\top x| - \sum_{i=1}^r |w_i^{*\top} x|$ is the residual. This directly implies that $L^{1/2}(W)=\norm{R}/\sqrt{2}\geq\Omega(w_{min}^{5/2}r^{-3/2}w_{max}^{-3/2}\cdot\delta^{3/2})$.

Our choice of test function crucially relies on the nonlinearity of absolute function. More specifically, absolution function $|w^\top x|$ is a linear function everywhere else except the hyperplane $w^\top x = 0$. We call the neighborhood of this hyperplane ($\{x| |w^\top x | \le \tau\}$ for some small $\tau$) the nonlinear region. The key observation here is that the loss in this nonlinear region would be small only if there is another neuron $w'$ such that $w' \approx w$, and they cancel each other. See Figure \ref{fig:nonlinear} for an illustration. This gives the basic component for constructing test functions for Lemma \ref{lem-high-dim-loss-delta}.

\begin{figure}
    \centering
    \begin{minipage}[t]{0.4\textwidth}
    \centering
    \begin{tikzpicture}
    \draw (-2.5,0) -- (2.5,0) node[right]{x axis};
    \draw (-1.5,1.5) -- (0,0) node[yshift=-7pt]{$w^\perp$};
    \draw (0,0) -- (1.5,1.5) node[right]{$|w^\top x|$};
    \draw (-2,1.5) -- (-0.5,0) node[yshift=-7pt]{$v^\perp$};
    \draw (-0.5,0) -- (1,1.5) node[left,xshift=-7pt]{$|v^\top x|$};
    \end{tikzpicture}
     
    \end{minipage}
    \hfil
    \begin{minipage}[t]{0.4\textwidth}
    \centering
    \begin{tikzpicture}
    \draw (-3,0) -- (3,0) node[right]{x axis};
    \draw[blue,line width=1pt] (-2,0) -- (2,0) node[yshift=-7pt]{$S$};
    \draw (-2,1) -- (0,-1);
    \draw (0,-1) -- (2,1) node[right]{$|w^{*\top} x|$};
    \draw[dashed] (0,-1) -- (0,0) node[yshift=7pt]{$w^{*\perp}$};
    \draw[dashed] (2,1) -- (2,0);
    \draw[dashed] (-2,1) -- (-2,0);
    \end{tikzpicture}
    \end{minipage}
    
    \caption{Left: Illustration for nonlinearity and nonlinear region. Here $w^\perp$ and $v^\perp$ represent the vector that is orthogonal to $w$ and $v$. Right: Test function $h(x)$ in Lemma \ref{lem-high-dim-loss-delta}, where $S=\{x||w^{*\top} x|\leq\tau\}$.}
    \label{fig:nonlinear}
\end{figure}
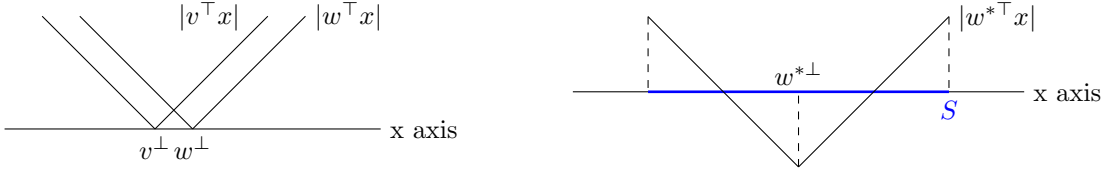

Suppose $w^*$ is a teacher neuron without any close-by student neurons. We focus on the nonlinear region of $w^*$, and construct test function $h(x)=\left(|w^{*\top}x|-\E_x[|w^{*\top}x|\mathbb{I}_S]/\E_x[\mathbb{I}_S]\right)\mathbb{I}_S$, where $S=\{x||w^{*\top}x|\leq\tau\}$ with a small enough $\tau$ (see also Figure \ref{fig:nonlinear}). Such a test function has almost zero correlation with any function that is linear in region $S$ (i.e. terms contributed by student and teacher neurons far away from $w^*$). On the other hand, it has a large positive correlation with teacher neuron $|w^{*\top}x|$ term in the residual. See Section~\ref{sec-pf-lem-high-dim-loss-delta} in appendix for detailed bounds.

\subsection{Local Property II (Lemma \ref{lem-high-dim-sum-student-small}): Residual Decomposition}
\label{subsec:residualdecomp}
We first introduce the notion of residual and its decomposition, which turns out to the key in the proof of Lemma \ref{lem-high-dim-sum-student-small}. For any fixed student neurons $W$, recall the residual function $R(x) = \sum_{i=1}^m \norm{w_i}|w_i^\top x| - \sum_{i=1}^r |w_i^{*\top} x|$. 
Intuitively, the residual is the difference between the outputs of the student network and the teacher network. 
Clearly our loss $L(W) = \E_x[R(x)^2]/2$. We also note the gradient of loss is closely related to the residual as $\nabla_{w_j}L(W)=\E_x[R(x)\norm{w_j}(I+\bar{w}_j\bar{w}_j^\top)x\sgn(w_j^\top x)]$. We decompose the residual as follows:
\begin{equation} \label{eq:residual_decomp}
R_1(x) \triangleq \sum_{i=1}^r v_i^\top x\sgn(w_i^{*\top} x), \ \ 
R_2(x) \triangleq \sum_{i=1}^r\sum_{j\in T_i} \norm{w_j}w_{j}^\top x \left(\sgn(w_{j}^\top x) - \sgn(w_i^{*\top} x)\right).
\end{equation}
where $v_i$ is defined in \eqref{eq:average_distance}. It is clear that $R(x)=R_1(x)+R_2(x)$. As discussed at the beginning of Section~\ref{subsec:descentdirection}, we can fix the direction of student neurons to have positive correlation with their corresponding teacher neuron, so there is no ambiguity in the definition of ``average neuron'' $\hat{w_i}$ and thus no ambiguity in the decomposition $R_1$ and $R_2$ defined above.

Intuitively, each neuron has % (absolute function), 
a linear weight component and a nonlinear activation pattern. $R_1$ describes the difference in linear weights, while $R_2$ describes the difference in activation pattern. 
In some sense, if one focuses only on $R_1$, then the model becomes {\em exactly parameterized} again because $R_1$ only depends on the average neurons $\hat{w}_i$ and does not depend on individual neurons. Lemma~\ref{lem-approximate-hessian} makes this observation more precise.

Using the decomposition and characterizations of $R_1$ and $R_2$, we can now discuss the proof strategy for Lemma \ref{lem-high-dim-sum-student-small}.

\paragraph{High-level proof strategy for Lemma \ref{lem-high-dim-sum-student-small}} 
To prove Lemma \ref{lem-high-dim-sum-student-small}, i.e., to show that $\norm{v_i}$ is small for all $i\in[r]$, we do following two steps.
\begin{enumerate}
\item Show that $\norm{R_1}^2$ is strongly convex in $(v_1, \ldots, v_r)$. Thus to show $\norm{v_i}$ is small for all $i$, it is sufficient to show that $\norm{R_1}^2$ is small.
\item Show $\norm{R_2}^2$ is small. Since $R_1 = R - R_2$, and we know $\norm{R}^2 = L(W)$ is small, this implies $\norm{R_1}^2$ is small as we need.
\end{enumerate}

\paragraph{Step 1} let $v = (v_1, \ldots, v_r)$, we observe that 
$\norm{R_1}^2 = v^{\top}Mv = \sum_{i, j} v_i^\top M_{ij} v_j$ where
\begin{equation} \label{eq:M_def}
    M_{ij} = \E_{x\sim N(0,I)} \left[xx^\top \sgn(w_i^{*\top} x) \sgn(w_j^{*\top} x)\right].
\end{equation}
In fact, $M\in\mathbb{R}^{dr\times dr}$ also corresponds to the Hessian at global minima in the exact-parameterization case. 
We can show $\lambda_{min}(M)=\Omega(\Delta^3/r^3)$, which ensures that $\norm{R_1}^2$ is strongly convex in $v$. See the proof in Section \ref{sec-pf-approximate-hessian}.

\begin{restatable}{lemmma}{lemApproximateHessian}
\label{lem-approximate-hessian}
    Under Assumption \ref{as:separate}, we have $\norm{R_1}^2 \ge \Omega(\Delta^3/r^3) \norm{v}^2$.
\end{restatable}

\paragraph{Step 2} we first observe following upper bound on $R_2$. 
\begin{equation*}\label{eq-r2}
    \begin{aligned}
    \norm{R_2}^2
    =&\E_{x}\left[\left(\sum_{i=1}^r\sum_{j\in T_i} \norm{w_j}w_{j}^\top x \left(\sgn(w_{j}^\top x) - \sgn(w_i^{*\top} x)\right)\right)^2\right]
    \leq r\sum_{i=1}^r\left(\sum_{j\in T_i}\norm{w_j}^2\delta_j^{3/2}\right)^2
    \end{aligned}
\end{equation*}
We obtain the upperbound for $\sum_{i=1}^m\norm{w_i}^2$ and $\sum_{i=1}^m\norm{w_i}^2\delta_i^2$ in Lemma~\ref{lem-sum-norm-bound} and Lemma~\ref{lem-high-dim-ls-small-improve}. The bound for the latter term uses similar ideas of constructing test functions as Lemma~\ref{lem-high-dim-loss-delta}, but the test function is more complicated. Combining the two upperbounds we have the following result:

\begin{restatable}{lemmma}{lemRtwoBound}\label{lem-r2-bound}
Under Assumption \ref{as:separate}, \ref{as:norm}, 
there exists a threshold $\epsilon_0=poly(\Delta,r^{-1},w_{max}^{-1},w_{min})$ such that for any $W$ satisfying loss $\epsilon\triangleq L(W) \leq\epsilon_0$, we have $\norm{R_2}^2=O(r^{5/2}w_{max}^{1/2}\epsilon^{3/4})$.
\end{restatable}

\subsection{Proof Sketch of Local Convergence Theorem (Theorem \ref{lem-repara-high-dim-converge})}
\label{sec-pf-sketch-local-converge}
Theorem \ref{lem-repara-high-dim-grad-norm} proves that the gradient norm is lower bounded by the value of loss function when loss is small. In order to establish the convergence result, we need to additionally characterize the local ``smoothness'' of the loss function.

\begin{restatable}[Smoothness]{lemmma}{lemHighDimSmooth}
\label{lem-high-dim-smooth}
    Under Assumption \ref{as:norm}, if loss $L(W)=O(r^2w_{max}^2)$, then
    \begin{align*}
        L(W+U)\leq& L(W) + \langle \nabla_W L(W), U\rangle
        +O(r^{1/4}w_{max}^{1/4})L^{1/2}(W)\norm{U}_F^{3/2} + O(rw_{max})\norm{U}_F^2\\ &+ O(1)\norm{U}_F^4.
    \end{align*}
\end{restatable}

We note Lemma \ref{lem-high-dim-smooth} is slightly different from the standard notion of smoothness in the optimization literature, but it turns out to be also sufficient to prove our result. See the proof in Section \ref{sec-pf-smooth}. 
Furthermore, we can also prove that the function is Lipschitz when loss is small.
\begin{restatable}[Lipschitz]{lemmma}{lemGradUpperBound}
\label{lem-grad-upper-bound}
Under Assumption \ref{as:norm}, if loss $L(W)=O(r^2w_{max}^2)$, then $\norm{\nabla_{W}L(W)}_F^2=O(r^3w_{max}^3)$.
\end{restatable}

Finally, following the standard linear algebra calculations, Theorem \ref{lem-repara-high-dim-grad-norm}, Lemma \ref{lem-high-dim-smooth} and Lemma~\ref{lem-grad-upper-bound} jointly establish the convergence result (Theorem \ref{lem-repara-high-dim-converge}). We defer the detailed proof to Section \ref{sec:proof_main_theorem}.

\section{Conclusion}\label{sec:conclusion}
In this paper, we develop a local convergence theory for mildly over-parameterized two-layer neural networks. By characterizing the local landscape and showing gradient satisfies a special case of Łojasiewicz property, we prove that as long as initial loss is below a threshold that is polynomial in relevant parameters, gradient descent could converge to zero loss. Our result is different from NTK analysis and mean-field analysis, since student neurons converge to the ground-truth teacher neuron and we only have mild requirement on the over-parameterization. One immediate open question is when can gradient descent find such a good initialization. We hope our result could lead to stronger optimization results for mildly over-parameterized neural networks.

\section*{Acknowledgement}
Rong Ge and Mo Zhou are supported in part by NSF-Simons Research Collaborations on the Mathematical and Scientific Foundations of Deep Learning (THEORINET), NSF Award CCF-1704656, CCF-1845171 (CAREER), CCF-1934964 (Tripods), a Sloan Research Fellowship, and a Google Faculty Research Award. Part of the work was done when Rong Ge and Chi Jin were visiting Instituted for Advanced Studies for “Special Year on
Optimization, Statistics, and Theoretical Machine Learning” program.

\bibliographystyle{apa}
\bibliography{ref}

\newpage
\appendix

\section{Discussions about ReLU Network and Absolute Network}
\subsection{Two-Layer ReLU network and Two-Layer Absolute network}\label{append-relu}
We discuss two ways that reduce learning a two-layer teacher network with ReLU activation to a two-layer teacher network with absolute value activation with infinite number of data. Consider the following problem for learning two-layer ReLU net in the teacher-student setting with Gaussian input (or symmetric input\footnote{We call $x$ follows a symmetric distribution $\mathcal{D}$ if for any $x$, the probability of observing $x\sim\mathcal{D}$ is the same as the probability of observing $-x\sim\mathcal{D}$. } as long as $\E_x[xx^\top]$ is full rank) 
\begin{align}
\label{relu-loss}
    \min_{\{w_i\}_{i=1}^m} \E_{x}\left[\frac{1}{2}\left(\sum_{i=1}^m \norm{w_i}\relu(w_i^\top x) - \sum_{i=1}^r \relu(w_i^{*\top} x)\right)^2\right].
\end{align}

The first way is to first fit the optimal linear function to the data, that is
\begin{align*}
    \beta^*=\arg\min_{\beta} \E_{x}\left[\frac{1}{2}\left(\beta^\top x - \sum_{i=1}^r \relu(w_i^{*\top} x)\right)^2\right].
\end{align*}
Since $\E_x[x\relu(w^\top x)] = \frac{1}{2}\E_x[xx^\top] w$ \citep{goel2018learning,ge2018learning}, we know $$\sum_{i=1}^r \relu(w_i^{*\top} x) - \beta^{*\top} x = \frac{1}{2}\sum_{i=1}^r |w_i^{*\top} x|.$$ This implies once we first find the optimal linear function $\beta^{*\top}x$, we can then reduced the problem to learn a teacher network with absolute value activation.

The second way follows similar idea. We claim that the problem \eqref{relu-loss} is equivalent to problem \eqref{model} in the way described below.
\begin{align}
\label{relu-linear-loss}
    \min_{\beta,\{w_i\}_{i=1}^m} \E_{x}\left[\frac{1}{2}\left(\beta^\top x + \sum_{i=1}^m \norm{w_i}\text{ReLU}(w_i^\top x) - \sum_{i=1}^r \text{ReLU}(w_i^{*\top} x)\right)^2\right],
\end{align}
where we have an additional linear term in the two-layer ReLU student net.

To see the equivalence, let us first optimize $\beta$ which is essentially a least square problem, we will obtain $\beta = \frac{1}{2}\sum_{i=1}^r w_i^* - \frac{1}{2}\sum_{i=1}^m \norm{w_i}w_i$. Then, plugging in the expression of $\beta$ into \eqref{relu-linear-loss}, we have 
\begin{align*}
    \min_{\{w_i\}_{i=1}^m} \E_{x}\left[\frac{1}{2}\left(\frac{1}{2}\sum_{i=1}^m 
    \norm{w_i}|w_i^\top x| - \frac{1}{2}\sum_{i=1}^r |w_i^{*\top} x|\right)^2\right],
\end{align*}
which is exactly the same as the problem \eqref{model} where we use absolute value function as activation. Therefore, our results can also be extended to the ReLU activation case with some modifications as discussed in the above way.

\subsection{Proof of Claim \ref{claim-relu}}\label{append-pre}
\claimrelu*
\begin{proof}
    We give one example to show Claim \ref{claim-relu} is true. Consider a teacher network that is parameterized as $f^*(x)=\sum_{i=1}^r\relu(w_i^{*\top}x)$ with $\sum_{i=1}^r w_i^*=0$. For the student network $f(x)=\sum_{i=1}^m \norm{w_i}\relu(w_i^\top x)$, let $m=r$ and $\norm{w_i}w_i=-w_i^*$ for $i\in[r]$. Using the fact that $\relu(x)-\relu(-x)=x$, we have
    \begin{align*}
        f(x) - f^*(x) = \sum_{i=1}^r\relu(\norm{w_i}w_i^{\top}x) - \sum_{i=1}^r\relu(w_i^{*\top}x)
        = -\sum_{i=1}^r w_i^{*\top}x=0.
    \end{align*}
    This indicates that when loss is 0, it is possible that student neuron does not match the direction of any teacher neuron.
\end{proof}

\subsection{Proof of Claim~\ref{claim:over-para}}
\claimoverpara*
\begin{proof}
    Since $\norm{w_1}w_1+\norm{w_2}w_2=w^*$, we know $\norm{w_1}|w_1^\top x|+\norm{w_1}|w_1^\top x| = |w^{*\top}x|$ when $\sgn(w_1^\top x)=\sgn(w_2^\top x)$. Thus,
    \begin{align*}
        L(W) 
        =& \E_x\left[\left(\norm{w_1}|w_1^\top x|+\norm{w_1}|w_1^\top x|-|w^{*\top}x|\right)^2\mathbb{I}_{\sgn(w_1^\top x)\neq\sgn(w_2^\top x)}\right]\\
        =&\E_{\tilde{x}}\left[\left(\norm{w_1}|w_1^\top \tilde{x}|+\norm{w_1}|w_1^\top \tilde{x}|-|w^{*\top}\tilde{x}|\right)^2\mathbb{I}_{\sgn(w_1^\top \tilde{x})\neq\sgn(w_2^\top \tilde{x})}\right],
    \end{align*}
    where $\tilde{x}$ is a 3-dimensional Gaussian since the expectation only depends on three vectors $w_1$, $w_2$ and $w^*$. Note that when $\sgn(w_1^\top \tilde{x})\neq\sgn(w_2^\top \tilde{x})$, $\norm{w_1}|w_1^\top \tilde{x}|+\norm{w_1}|w_1^\top \tilde{x}|-|w^{*\top}\tilde{x}|=O(\delta\norm{\tilde{x}})$. Therefore, we know $L(W)=O(\delta^3)$.
\end{proof}

\section{Proof of Gradient Lower Bound (Theorem \ref{lem-repara-high-dim-grad-norm})}\label{sec-pf-grad-norm}

Recall that in Lemma \ref{lem-repara-high-dim-small-loss-descent-dir} we construct a descent direction and show that it has a positive correlation with the gradient. Then using Lemma \ref{lem-repara-high-dim-small-loss-descent-dir} with $q_{ij}=\frac{\norm{w_{ij}}}{\sum_{j\in T_{i}(\delta_{max})}\norm{w_{ij}}^2}$ for $j\in T_{i}(\delta_{max})$ and $q_{ij}=0$ otherwise, we can show a lower bound on gradient norm.

\lemReparaHighDimGradNorm*

\begin{proof}
     By Lemma~\ref{lem-repara-high-dim-small-loss-descent-dir}, we know 
    \begin{align*}
        \sum_{i=1}^r\sum_{j\in T_i} \langle\nabla_{w_{j}}L(W),(I+\bar{w}_{j}\bar{w}_{j})^{-1}(w_{j}-q_{ij}\sgn(w_j^\top w_i^*)w_i^*)\rangle
        \geq L(W),
    \end{align*}
    where $\{q_{ij}\}_{i \in [r], j \in[m]}$ is any sequence that satisfies (1) $q_{ij}\geq 0$ for all $(i, j)$,  (2) $q_{ij}=0$ if $j\not\in  T_{i}(\delta_{max})$; and (3) $\sum_{j\in T_{i}(\delta_{max})} q_{ij}\norm{w_j}=1$ for all $i\in[r]$, where $\delta_{\max} = \Theta(rw_{max}w_{min}^{-5/3} \cdot \epsilon^{1/3})$. Hence,
    \begin{equation}\label{eq-repara-high-dim-grad-norm-1}
    \begin{aligned}
        &\left(\sum_{i=1}^r\sum_{j\in T_i} \norm{\nabla_{w_{j}}L(W)}^2\right)^{1/2}
        \left(\sum_{i=1}^r\sum_{j\in T_i}\norm{(I+\bar{w}_{j}\bar{w}_{j})^{-1}(w_{j}-q_{ij}\sgn(w_j^\top w_i^*)w_i^*)}^2\right)^{1/2}\\
        \geq& \sum_{i=1}^r\sum_{j\in T_i} \norm{\nabla_{w_{j}}L(W)}\norm{(I+\bar{w}_{j}\bar{w}_{j})^{-1}(w_{j}-q_{ij}\sgn(w_j^\top w_i^*)w_i^*)}
        \geq L(W).
    \end{aligned}
    \end{equation}

    To lower bound the gradient norm, we have to give a upper bound for the following term.
    \begin{align*}
        \sum_{i=1}^r\sum_{j\in T_i}\norm{(I+\bar{w}_{j}\bar{w}_{j})^{-1}(w_{j}-q_{ij}\sgn(w_j^\top w_i^*)w_i^*)}^2
        \leq& \sum_{i=1}^r\sum_{j\in T_i}\norm{w_{j}-q_{ij}\sgn(w_j^\top w_i^*)w_i^*}^2\\
        \leq& 2\sum_{i=1}^r\sum_{j\in T_i}\left(\norm{w_{j}}^2+\norm{q_{ij}w_i^*}^2\right)\\
        =& 2\sum_{i=1}^r\sum_{j\in T_i}\norm{w_{j}}^2
        + 2\sum_{i=1}^r\sum_{j\in T_i}q_{ij}^2\norm{w_i^*}^2.
    \end{align*}
    
    Let $q_{ij}=\frac{\norm{w_{j}}}{\sum_{j\in T_{i}(\delta_{max})}\norm{w_{j}}^2}$ for $j\in T_{i}(\delta_{max})$ and $q_{ij}=0$ otherwise. We have
    \begin{align*}
        \sum_{i=1}^r\sum_{j\in T_i}q_{ij}^2\norm{w_i^*}^2
        = \sum_{i=1}^r\sum_{j\in T_{i}(\delta_{max})}\frac{\norm{w_{j}}^2\norm{w_i^*}^2}{\left(\sum_{j\in T_{i}(\delta_{max})}\norm{w_{j}}^2\right)^2}
        = \sum_{i=1}^r\frac{\norm{w_i^*}^2}{\sum_{j\in T_{i}(\delta_{max})}\norm{w_{j}}^2}.
    \end{align*}
    Using Lemma \ref{lem-high-dim-loss-delta} with $\delta_{max}=\Theta(rw_{max}w_{min}^{-5/3} \cdot \epsilon^{1/3})$, we have
    $\sum_{j\in T_{i}(\delta_{max})}\norm{w_{j}}^2 \geq \frac{1}{2}\norm{w_i^*}$.
    
    Thus, with Lemma \ref{lem-sum-norm-bound}, we have
    \begin{align*}
        \sum_{i=1}^r\sum_{j\in T_i}\norm{(I+\bar{w}_{j}\bar{w}_{j})^{-1}(w_{j}-q_{ij}\sgn(w_j^\top w_i^*)w_i^*)}^2
        \leq& 2\sum_{i=1}^r\sum_{j\in T_i}\norm{w_{j}}^2
        + 2\sum_{i=1}^r\sum_{j\in T_i}q_{ij}^2\norm{w_i^*}^2\\
        \leq& O(rw_{max}) + 2\sum_{i=1}^r 2\norm{w_i^*}
        = O(rw_{max}).
    \end{align*}
    Together with \eqref{eq-repara-high-dim-grad-norm-1}, we have
    \begin{align*}
        \norm{\nabla_W L(W)}_F=\left(\sum_{i=1}^r\sum_{j\in T_i} \norm{\nabla_{w_{j}}L(W)}^2\right)^{1/2}
        \geq \kappa L(W),
    \end{align*}
    where $\kappa=\Theta\left(r^{-1/2} w_{max}^{-1/2}\right)$.
\end{proof}

%%%%%%%%%%%%%%%%%%%%%%%%%%%%%%%%%%%%%%%%%%%%%%%%%%%%%%%%%%%%%%%
%% Descent Direction Lemma
%%%%%%%%%%%%%%%%%%%%%%%%%%%%%%%%%%%%%%%%%%%%%%%%%%%%%%%%%%%%%%%
\section{Proof of Descent Direction Lemma (Lemma \ref{lem-repara-high-dim-small-loss-descent-dir})}\label{sec-pf-descent-dir}

We first give the proof of Lemma \ref{lem-repara-high-dim-small-loss-descent-dir}, then give the proof of each step (Lemma~\ref{lem-high-dim-loss-delta}, Lemma~\ref{lem-high-dim-sum-student-small}, Lemma~\ref{lem-high-dim-correlation-lower-bound}) in the proof sketch accordingly in the following subsections.

Before presenting the proof of descent direction lemma, we first show that we can arbitrarily change the sign of student neurons and still have the direction of improvement. As discussed at the beginning of Section~\ref{subsec:descentdirection}, this allows us to assume $w_j^\top w_i^*\ge 0$ for any $j\in T_i$. In the rest of this section, we assume the above always hold.
\begin{lemma}\label{lem:symmetric}
Suppose there exists $\{g(w_j)\}_{j\in[m]}$ such that $W=(w_1,\ldots,w_m)$ satisfies 
    \begin{align*}
    \sum_{j\in [m]} \langle\nabla_{w_{j}}L(W),g(w_{j})\rangle
    \geq L(W).
    \end{align*}
Then, for any $a_i\in\{\pm 1\}$ if we replace $W$ with $W^\prime=(a_1w_1,\ldots,a_mw_m)$, we have
    \begin{align*}
    \sum_{j\in [m]} \langle\nabla_{w_{j}^\prime}L(W^\prime),a_jg(w_{j})\rangle
    \geq L(W^\prime).
    \end{align*}

\end{lemma}
\begin{proof}
    To specify the dependency on $W$, let $R_W(x) \triangleq \sum_{i=1}^m \norm{w_i}|w_i^\top x| - \sum_{i=1}^r |w_i^{*\top} x|$. Denote $w_i^\prime = a_iw_i$. Using $R_W=R_{W^\prime}$, we have
    \begin{align*}
        \langle\nabla_{w_{j}}L(W),g(w_j)\rangle
        =& \E_x\left[ R_W(x)\norm{w_{j}}g(w_j)^\top(I+\bar{w}_{j}\bar{w}_{j}) x\sgn(w_{j}^\top x)
        \right]\\
        =&\E_x\left[ R_{W^\prime}(x)\norm{w_{j}^\prime}g(w_j)^\top(I+\bar{w}_{j}^\prime\bar{w}_{j}^\prime) x\sgn(w_{j}^{\prime\top} x a_j)
        \right]\\
        =&\E_x\left[ R_{W^\prime}(x)\norm{w_{j}^\prime}g(w_j)^\top(I+\bar{w}_{j}^\prime\bar{w}_{j}^\prime) a_jx\sgn(w_{j}^{\prime\top} x)
        \right]\\
        =&\langle\nabla_{w_{j}^\prime}L(W^\prime),a_jg(w_{j})\rangle
    \end{align*}    
    Since $L(W)=L(W^\prime)$, we know the result holds.
\end{proof}

\subsection{Proof of Descent Direction Lemma (Lemma \ref{lem-repara-high-dim-small-loss-descent-dir})}
We need the follow lemma to prove Lemma \ref{lem-repara-high-dim-small-loss-descent-dir}. In fact, this lemma corresponds to the third step of proof sketch as described in Section~\ref{subsec:descentdirection}. It shows that as long as every teacher neuron has close-by student neuron and the difference between average neuron $\hat{w}_i$ and $w_i^*$ is small, the direction we construct is indeed a descent direction. The proof is provided in Section \ref{sec-pf-correlation-lower-bound}.

\begin{restatable}{lemmma}{lemHighDimCorrelationLowerBound} \label{lem-high-dim-correlation-lower-bound}
    Under Assumption \ref{as:norm}, if $\norm{v_i}\leq\alpha$ and $|T_i(\delta_{max})|\geq 1$ for all $i\in[r]$, then 
    \begin{align*}
        \sum_{i=1}^r\sum_{j\in T_i} \langle\nabla_{w_{j}}L(W),(I+\bar{w}_j\bar{w}_j^\top)^{-1}(w_{j}-q_{ij}\sgn(w_j^\top w_i^*)w_i^*)\rangle
        \geq \norm{R}^2 - 
        O(r^2w_{max}\alpha\delta_{max}^2),
    \end{align*}
    where $\{q_{ij}\}_{i \in [r], j \in[m]}$ is any sequence that satisfies (1) $q_{ij}\geq 0$ for all $(i, j)$,  (2) $q_{ij}=0$ if $j\not\in  T_{i}(\delta_{max})$; and (3) $\sum_{j\in T_{i}(\delta_{max})} q_{ij}\norm{w_j}=1$ for all $i\in[r]$.
\end{restatable}

Now we are ready to prove Lemma \ref{lem-repara-high-dim-small-loss-descent-dir} by following the three-step proof sketch in Section~\ref{subsec:descentdirection}.

\lemReparaHighDimSmallLossDescentDir*

\begin{proof}
    Using Lemma \ref{lem-high-dim-loss-delta}, we know there exists such $\delta_{max}=\Theta\left(rw_{max}w_{min}^{-5/3}\cdot\epsilon^{1/3}\right)$ satisfies $|T_i(\delta_{max})|\geq 1$. From Lemma \ref{lem-high-dim-sum-student-small}, we know $\norm{v_i}\leq \alpha=O(r^{11/4}w_{max}^{1/4}\Delta^{-3/2}\cdot \epsilon^{3/8})$ for all $i\in[r]$.
    
    Then, by Lemma \ref{lem-high-dim-correlation-lower-bound} we have
    \begin{align*}
        \sum_{i=1}^r\sum_{j\in T_i} \langle\nabla_{w_{j}}L(W),(I+\bar{w}_{j}\bar{w}_{j})^{-1}(w_{j}-q_{ij}\sgn(w_j^\top w_i^*)w_i^*)\rangle
        \geq \norm{R}^2 -O(r^2 w_{max}\alpha \delta_{max}^2).
    \end{align*}

    By the choice of $\delta_{max}$, we have $O(r^2 w_{max}\alpha \delta_{max}^2)=O(r^{27/4}w_{max}^{13/4}\Delta^{-3/2}w_{min}^{-10/3}\cdot\epsilon^{25/24})$. Then with the fact that $\norm{R}^2=2L(W)=2\epsilon$, when $\epsilon\leq\epsilon_0=O(\Delta^{36}w_{min}^{80}r^{-162}w_{max}^{-78})$ we have
    \begin{align*}
        \sum_{i=1}^r\sum_{j\in T_i} \langle\nabla_{w_{j}}L(W),(I+\bar{w}_{ij}\bar{w}_{ij})^{-1}(w_{ij}-q_{ij}\sgn(w_j^\top w_i^*)w_i^*)\rangle
        \geq& 2\epsilon - O\left(\frac{r^{27/4}w_{max}^{13/4}}{\Delta^{3/2}w_{min}^{10/3}}\epsilon^{25/24}\right)\\
        \geq& \epsilon.
    \end{align*}
\end{proof}

%%%%%%%%%%%%%%%%%%%%%%%%%%%%%%%%%%%%%%%%%%%%%%%%%%%%%%%%%%%%%%%%%%%%%%
% loss small, student close to teacher
%%%%%%%%%%%%%%%%%%%%%%%%%%%%%%%%%%%%%%%%%%%%%%%%%%%%%%%%%%%%%%%%%%%%%%
\subsection{Proof of Lemma \ref{lem-high-dim-loss-delta}}\label{sec-pf-lem-high-dim-loss-delta}
We give the proof of Lemma \ref{lem-high-dim-loss-delta}, which is the first step of the proving Lemma \ref{lem-repara-high-dim-small-loss-descent-dir} as described in Section~\ref{subsec:descentdirection}. It shows every teacher neuron has at least one close-by student neuron when loss is small. The proof follows the idea of test function as described in Section \ref{subsec:testfunction}.

\lemHighDimLossDelta*

\begin{proof}
     The following proof has two parts: In Part 1, we aim to show $|T_i(\delta_{max})|\geq 1$; In Part 2, we aim to show $\sum_{j\in T_i(\delta_{max})}\norm{w_j}^2\geq \frac{1}{2}\norm{w_i}$. Denote $\erf(x)=\frac{2}{\sqrt{\pi}}\int_0^x e^{-t^2}\dif t$ as the error function. We will use $\delta$ instead of $\delta_{max}$ to simplify the notation in the proof.
    WLOG, assume $\delta\leq \Delta$ as we can choose small enough $\epsilon_0=poly(\Delta,r^{-1},w_{max}^{-1},w_{min})$.
    
    \paragraph{Part 1:} We aim to show every teacher neuron has at least one close student neuron. Assume toward contradiction that there is a teacher neuron $w_i^*$ such that for all student neurons $w_j$, we have $\angle(w_i^*,w_j)\geq \delta$. Denote set $S=\{x\in\mathbb{R}^d | |w_i^{*\top} x|\leq\tau\}$. To simplify notation, we will use $w^*$ to represent $w_i^*$. Let $w$ be a normalized vector satisfying $\phi\triangleq\angle(w^*,w)\geq \delta$. In following part (i)-(iii), WLOG, assume $\norm{w}=\norm{w^*}=1$, $w^*=(1,0,\cdots,0)^\top$ and $w=(\cos\phi,\sin\phi,0,\cdots,0)^\top$. We are going to focus on $|w^\top x|$, so we can further assume $\phi\in[\delta,\frac{\pi}{2}]$. Also, we will assume $\tau,\frac{\tau}{\phi}\leq c$ for a sufficiently small constant $c$.
    
    \paragraph{(i) Estimate $\langle|w^{*\top}x|, |w^\top x|\rangle_S$.}
    
    We have
    \begin{align*}
        \langle |w^{*\top} x|, |w^\top x|\rangle_S 
        =& \frac{1}{2\pi}\int_{-\tau}^\tau |x_1| e^{-x_1^2/2}\int_{-\infty}^\infty |x_1\cos\phi+x_2\sin\phi|e^{-x_2^2/2} \dif x_2 \dif x_1.
    \end{align*}
    Note that since $\phi\in (0,\frac{\pi}{2}]$, we know $x_1\cos\phi+x_2\sin\phi\geq 0$ is equivalent to $x_2\geq -x_1\cot\phi$. Thus,
    \begin{equation}\label{lem-loss-delta-eq-1}
    \begin{aligned}
        &\int_{-\infty}^\infty |x_1\cos\phi+x_2\sin\phi|e^{-x_2^2/2} \dif x_2\\
        =& \int_{-x_1\cot\phi}^\infty (x_1\cos\phi+x_2\sin\phi)e^{-x_2^2/2} \dif x_2
        - \int_{-\infty}^{-x_1\cot\phi} (x_1\cos\phi+x_2\sin\phi)e^{-x_2^2/2} \dif x_2\\
        =& x_1\cos\phi \int_{-x_1\cot\phi}^{x_1\cot\phi} e^{-x_2^2/2} \dif x_2
        + 2\sin\phi \int_{|x_1\cot\phi|}^\infty x_2 e^{-x_2^2/2} \dif x_2\\
        =& \sqrt{2\pi}x_1\cos\phi\erf\left(\frac{x_1\cot\phi}{\sqrt{2}}\right)
        + 2\sin\phi e^{-x_1^2\cot^2\phi/2}.
    \end{aligned}
    \end{equation}
    This leads to
    \begin{equation}\label{lem-loss-delta-eq-2}
    \begin{aligned}
        \langle |w^{*\top} x|, |w^\top x|\rangle_S 
        =& \frac{1}{2\pi}\int_{-\tau}^\tau |x_1| e^{-x_1^2/2} \left(\sqrt{2\pi}x_1\cos\phi\erf\left(\frac{x_1\cot\phi}{\sqrt{2}}\right)
        + 2\sin\phi e^{-x_1^2\cot^2\phi/2}\right) \dif x_1\\
        =& \frac{\cos\phi}{\sqrt{2\pi}} \int_{-\tau}^\tau |x_1| x_1 e^{-x_1^2/2} \erf\left(\frac{x_1\cot\phi}{\sqrt{2}}\right)\dif x_1
        + \frac{\sin\phi}{\pi} \int_{-\tau}^\tau |x_1|e^{-x_1^2/2\sin^2\phi} \dif x_1,
    \end{aligned}    
    \end{equation}

    For the first term in \eqref{lem-loss-delta-eq-2}, we have
    \begin{align*}
        \int_{-\tau}^\tau |x_1| x_1 e^{-x_1^2/2} \erf\left(\frac{x_1\cot\phi}{\sqrt{2}}\right)\dif x_1
        = 2\int_0^\tau x_1^2 e^{-x_1^2/2} \erf\left(\frac{x_1\cot\phi}{\sqrt{2}}\right)\dif x_1.
    \end{align*}
    
    To get the upper bound, using the fact that $x_1\cot\phi \geq 0$ and $\erf(x)\leq \frac{2x}{\sqrt{\pi}}$ for $x\geq 0$, we have
    \begin{align*}
        \int_{-\tau}^\tau |x_1| x_1 e^{-x_1^2/2} \erf\left(\frac{x_1\cot\phi}{\sqrt{2}}\right)\dif x_1
        \leq& 2\sqrt{\frac{2}{\pi}}\cot\phi\int_0^\tau x_1^3 e^{-x_1^2/2} \dif x_1\\
        =& 2\sqrt{\frac{2}{\pi}}\cot\phi(2-(2+\tau^2)e^{-\tau^2/2}).
    \end{align*}
    
    To get the lower bound, using the fact that $x_1\cot\phi\leq\frac{\tau}{\phi}\leq c$ and $\erf(x)\geq \frac{x}{\sqrt{\pi}}$ for $x\in[0,1]$, we have
    \begin{align*}
        \int_{-\tau}^\tau |x_1| x_1 e^{-x_1^2/2} \erf\left(\frac{x_1\cot\phi}{\sqrt{2}}\right)\dif x_1
        \geq& \sqrt{\frac{2}{\pi}}\cot\phi\int_0^\tau x_1^3 e^{-x_1^2/2} \dif x_1\\
        =& \sqrt{\frac{2}{\pi}}\cot\phi(2-(2+\tau^2)e^{-\tau^2/2}).
    \end{align*}
    
    For the second term in \eqref{lem-loss-delta-eq-2}, we have
    \begin{align*}
        \int_{-\tau}^\tau |x_1|e^{-x_1^2/2\sin^2\phi} \dif x_1
        = 2\int_0^\tau x_1 e^{-x_1^2/2\sin^2\phi} \dif x_1
        = 2\sin^2\phi (1-e^{-\tau^2/2\sin^2\phi}).
    \end{align*}
    
    Therefore, combining the above two term, we have the upper bound
    \begin{align*}
        \langle |w^{*\top} x|, |w^\top x|\rangle_S 
        &\leq \frac{\cos\phi}{\sqrt{2\pi}} 2\sqrt{\frac{2}{\pi}}\cot\phi(2-(2+\tau^2)e^{-\tau^2/2})
        + \frac{2}{\pi} \sin^3\phi (1-e^{-\tau^2/2\sin^2\phi})\\
        &\leq \frac{2\cos^2\phi}{\pi\sin\phi} \left(2-(2+\tau^2)(1-\frac{\tau^2}{2})\right)
        + \frac{2}{\pi} \sin^3\phi \left(\frac{\tau^2}{2\sin^2\phi}-\frac{\tau^4}{16\sin^4\phi}\right)\\
        &= \frac{\cos^2\phi}{\pi\sin\phi}\tau^4 + \frac{\sin\phi}{\pi}\tau^2 - \frac{\tau^4}{8\pi\sin\phi},
    \end{align*}
    where the second inequality is because $\tau\leq c$, $\frac{\tau}{\sin\phi}\leq \frac{2\tau}{\phi}\leq c$ and $1-e^{-x^2}\leq x^2-\frac{x^4}{4}\leq x^2$ for $x\in[0,1]$.
    
    Also, we have the lower bound
    \begin{align*}
        \langle |w^{*\top} x|, |w^\top x|\rangle_S 
        &\geq \frac{\cos\phi}{\sqrt{2\pi}} \sqrt{\frac{2}{\pi}}\cot\phi(2-(2+\tau^2)e^{-\tau^2/2})
        + \frac{2}{\pi} \sin^3\phi (1-e^{-\tau^2/2\sin^2\phi})\\
        &\geq \frac{\cos^2\phi}{\pi\sin\phi} \left(2-(2+\tau^2)(1-\frac{\tau^2}{2}+\frac{\tau^4}{8})\right)
        + \frac{2}{\pi} \sin^3\phi \left(\frac{\tau^2}{2\sin^2\phi}-\frac{\tau^4}{8\sin^4\phi}\right)\\
        &= \frac{\cos^2\phi}{\pi\sin\phi} \left(\frac{\tau^4}{4}-\frac{\tau^6}{8}\right) + \frac{\sin\phi}{\pi}\tau^2 - \frac{\tau^4}{4\pi\sin\phi},
    \end{align*}
    where the second inequality is because $\tau,\frac{\tau}{\sin\phi}\geq 0$ and $1-e^{-x^2}\geq x^2-\frac{x^4}{2}$ for $x\geq 0$.
    
    Thus, we have the estimation
    \begin{align*}
        \langle |w^{*\top} x|, |w^\top x|\rangle_S = \frac{\sin\phi}{\pi}\tau^2 \pm \Theta\left(\frac{\tau^4}{\phi}\right).
    \end{align*}

    \paragraph{(ii) Estimate $\langle 1, |w^\top x|\rangle_S$.}
    We have
    \begin{equation}\label{lem-loss-delta-eq-3}
    \begin{aligned}
        \langle 1, |w^\top x|\rangle_S 
        =& \frac{1}{2\pi}\int_{-\tau}^\tau
        e^{-x_1^2/2}\int_{-\infty}^\infty |x_1\cos\phi+x_2\sin\phi|e^{-x_2^2/2} \dif x_2 \dif x_1\\
        =& \frac{1}{2\pi}\int_{-\tau}^\tau 
        e^{-x_1^2/2} \left(\sqrt{2\pi}x_1\cos\phi\erf\left(\frac{x_1\cot\phi}{\sqrt{2}}\right)
        + 2\sin\phi e^{-x_1^2\cot^2\phi/2}\right) \dif x_1\\
        =& \frac{\cos\phi}{\sqrt{2\pi}} \int_{-\tau}^\tau 
        x_1 e^{-x_1^2/2} \erf\left(\frac{x_1\cot\phi}{\sqrt{2}}\right)\dif x_1
        + \frac{\sin\phi}{\pi} \int_{-\tau}^\tau e^{-x_1^2/2\sin^2\phi} \dif x_1,
    \end{aligned}
    \end{equation}
    where we use \eqref{lem-loss-delta-eq-1} in the second line.
    
    For the first term in \eqref{lem-loss-delta-eq-3}, we have
    \begin{align*}
        \int_{-\tau}^\tau x_1 e^{-x_1^2/2} \erf\left(\frac{x_1\cot\phi}{\sqrt{2}}\right)\dif x_1
        = 2\int_0^\tau x_1 e^{-x_1^2/2} \erf\left(\frac{x_1\cot\phi}{\sqrt{2}}\right)\dif x_1.
    \end{align*}
    
    To get the upper bound, using the fact that $x_1\cot\phi\geq 0$ and $\erf(x)\leq \frac{2x}{\sqrt{\pi}}$ for $x\geq 0$, we have
    \begin{align*}
        \int_{-\tau}^\tau  x_1 e^{-x_1^2/2} \erf\left(\frac{x_1\cot\phi}{\sqrt{2}}\right)\dif x_1
        \leq& 2\sqrt{\frac{2}{\pi}}\cot\phi\int_0^\tau x_1^2 e^{-x_1^2/2} \dif x_1\\
        =& 2\sqrt{\frac{2}{\pi}}\cot\phi 
        \left(-\tau e^{-\tau^2/2}+\sqrt{\frac{\pi}{2}}\erf\left(\frac{\tau}{\sqrt{2}}\right)\right).
    \end{align*}
    
    To get the lower bound, using the fact that $x_1\cot\phi\leq\frac{\tau}{\phi}\leq c$ and $\erf(x)\geq \frac{x}{\sqrt{\pi}}$ for $x\in[0,1]$, we have
    \begin{align*}
        \int_{-\tau}^\tau x_1 e^{-x_1^2/2} \erf\left(\frac{x_1\cot\phi}{\sqrt{2}}\right)\dif x_1
        \geq& \sqrt{\frac{2}{\pi}}\cot\phi\int_0^\tau x_1^2 e^{-x_1^2/2} \dif x_1\\
        =& \sqrt{\frac{2}{\pi}}\cot\phi
        \left(-\tau e^{-\tau^2/2}+\sqrt{\frac{\pi}{2}}\erf\left(\frac{\tau}{\sqrt{2}}\right)\right).
    \end{align*}
    
    For the second term in \eqref{lem-loss-delta-eq-3}, we have
    \begin{align*}
        \int_{-\tau}^\tau e^{-x_1^2/2\sin^2\phi} \dif x_1
        = 2\int_0^\tau e^{-x_1^2/2\sin^2\phi} \dif x_1
        = \sqrt{2\pi}\sin\phi\erf\left(\frac{\tau}{\sqrt{2}\sin\phi}\right).
    \end{align*}
    
    Therefore, combining the two terms above, we have the upper bound
    \begin{align*}
        \langle 1, |w^\top x|\rangle_S 
        &\leq \frac{\cos\phi}{\sqrt{2\pi}} 2\sqrt{\frac{2}{\pi}}\cot\phi 
        \left(-\tau e^{-\tau^2/2}+\sqrt{\frac{\pi}{2}}\erf\left(\frac{\tau}{\sqrt{2}}\right)\right)
        + \frac{\sin\phi}{\pi} \sqrt{2\pi}\sin\phi\erf\left(\frac{\tau}{\sqrt{2}\sin\phi}\right)\\
        &\leq \frac{2\cos^2\phi}{\pi\sin\phi} \left(-\tau(1-\frac{\tau^2}{2})+\tau\right)
        + \sqrt{\frac{2}{\pi}}\sin^2\phi \sqrt{\frac{2}{\pi}}\frac{\tau}{\sin\phi}\\
        &= \frac{\cos^2\phi}{\pi\sin\phi}\tau^3 + \frac{2\sin\phi}{\pi}\tau,
    \end{align*}
    where the second inequality is because $\tau\leq c$, $\frac{\tau}{\sin\phi}\leq \frac{2\tau}{\phi}\leq c$, $1-e^{-x^2}\leq x^2-\frac{x^4}{4}\leq x^2$ for $x\in[0,1]$ and $\erf(x)\leq \frac{2x}{\sqrt{\pi}}$ for $x\geq 0$.
    
    Also, we have the lower bound,
    \begin{align*}
        \langle 1, |w^\top x|\rangle_S 
        &\geq \frac{\cos\phi}{\sqrt{2\pi}} \sqrt{\frac{2}{\pi}}\cot\phi 
        \left(-\tau e^{-\tau^2/2}+\sqrt{\frac{\pi}{2}}\erf\left(\frac{\tau}{\sqrt{2}}\right)\right)
        + \frac{\sin\phi}{\pi} \sqrt{2\pi}\sin\phi\erf\left(\frac{\tau}{\sqrt{2}\sin\phi}\right)\\
        &\geq \frac{\cos^2\phi}{\pi\sin\phi} \left(-\tau(1-\frac{\tau^2}{2}+\frac{\tau^4}{8})+\tau - \frac{\tau^3}{6}\right)
        + \sqrt{\frac{2}{\pi}}\sin^2\phi \left(\sqrt{\frac{2}{\pi}}\frac{\tau}{\sin\phi}-\frac{2}{3\sqrt{\pi}}\frac{\tau^3}{2\sqrt{2}\sin^3\phi}\right)\\
        &= \frac{\cos^2\phi}{\pi\sin\phi}\left(\frac{\tau^3}{3}-\frac{\tau^5}{8}\right) + \frac{2\sin\phi}{\pi}\tau - \frac{\tau^3}{3\pi\sin\phi},
    \end{align*}
    where the second inequality is because $1-e^{-x^2}\geq x^2-\frac{x^4}{2}$ for $x\geq 0$ and $\erf(x)\geq \frac{2x}{\sqrt{\pi}}-\frac{2}{3\sqrt{\pi}}x^3$ for $x\geq 0$.
    
    Thus, we have the estimation
    \begin{align*}
        \langle 1, |w^\top x|\rangle_S = \frac{2\sin\phi}{\pi}\tau \pm \Theta\left(\frac{\tau^3}{\phi}\right).
    \end{align*}
    
    \paragraph{(iii) Determine $g$.}
    
    Let $g = \frac{\langle |w^{*\top} x|, 1\rangle_S}{\langle 1, 1\rangle_S},$
    which implies that
    \begin{align*}
        \langle|w^{*\top}x|-g, g\rangle_S &= 0.
    \end{align*}
    
    Note that
    \begin{align*}
        &\langle |w^{*\top} x|, 1\rangle_S 
        = \frac{1}{\sqrt{2\pi}}\int_{-\tau}^\tau |x_1| e^{-x_1^2/2}\dif x_1
        = \sqrt{\frac{2}{\pi}}(1-e^{-\tau^2/2}),\\
        &\langle 1, 1\rangle_S
        = \frac{1}{\sqrt{2\pi}}\int_{-\tau}^\tau e^{-x_1^2/2}\dif x_1
        = \erf\left(\frac{\tau}{\sqrt{2}}\right).
    \end{align*}
    Thus,
    \begin{align*}
        g = \frac{\langle |w^{*\top} x|, 1\rangle_S}{\langle 1, 1\rangle_S}
        = \sqrt{\frac{2}{\pi}}\frac{1-e^{-\tau^2/2}}{\erf(\frac{\tau}{\sqrt{2}})} = \frac{\tau}{2} + O(\tau^3).
    \end{align*}
    
    \paragraph{(iv) Complete the proof.}
    
    Note that
    \begin{align*}
        &\langle|w^{*\top}x|, |w^{*\top} x|\rangle_S
        =\frac{1}{\sqrt{2\pi}}\int_{-\tau}^\tau x_1^2 e^{-x_1^2/2} \dif x_1
        = \sqrt{\frac{2}{\pi}}\left(-\tau e^{-\tau^2/2}+\sqrt{\frac{\pi}{2}}\erf\left(\frac{\tau}{\sqrt{2}}\right)\right),\\
        &\langle 1, |w^{*\top} x|\rangle_S
        =\frac{1}{\sqrt{2\pi}}\int_{-\tau}^\tau |x_1| e^{-x_1^2/2} \dif x_1
        = \sqrt{\frac{2}{\pi}} (1-e^{-\tau^2/2}).
    \end{align*}
    Now, combining with the results in (i)(ii)(iii), we have
    \begin{align*}
    &\langle|w^{*\top}x|-g, g\rangle_S = 0,\\
    &\langle|w^{*\top}x|-g, |w^{*\top} x|\rangle_S 
    = 
    \langle|w^{*\top}x|, |w^{*\top} x|\rangle_S - g\langle 1, |w^{*\top} x|\rangle_S
    = \Theta(\tau^3),\\
    &\left|\langle|w^{*\top}x|-g, |w^\top x|\rangle_S\right| 
    = \left|\frac{\sin\phi}{\pi}\tau^2 \pm \Theta\left(\frac{\tau^4}{\phi}\right) - \left(\frac{\tau}{2} + O(\tau^3)\right)\left(\frac{2\sin\phi}{\pi}\tau \pm \Theta\left(\frac{\tau^3}{\phi}\right)\right)\right| \leq \Theta\left(\frac{\tau^4}{\phi}\right).
    \end{align*}
    
    In the following, we will no longer assume $\norm{w^*}=1$, and will directly use $\norm{w^*}$. Recall that $\angle(w_i,w^*)\geq \delta$ and $R(x) = \sum_{i=1}^m \norm{w_i}|w_i^\top x| - \sum_{i=1}^r |w_i^{*\top} x|$. We have
    \begin{align*}
        \langle |w^{*\top}x|-g, R(x)\rangle_S 
        \geq& \Theta(\tau^3) \norm{w^*}^2- \sum_{i=1}^m\norm{w_i}^2\norm{w^*}\Theta\left(\frac{\tau^4}{\delta}\right)-\sum_{j\neq i}\norm{w_j^*}\norm{w^*}\Theta\left(\frac{\tau^4}{\Delta}\right)\\
        =& \left(\Theta(\tau^3)\norm{w^*} - \Theta\left(\frac{rw_{max}\tau^4}{\delta}\right)\right)\norm{w^*},
    \end{align*}
    where we use $\delta\leq \phi$, $\delta\leq \Delta$ and Lemma \ref{lem-sum-norm-bound}.
    
    Hence, when $\tau\leq \frac{c_1w_{min}}{rw_{max}}\delta $ with a sufficiently small constant $c_1$, we have $\langle |w^{*\top}x|-g, R(x)\rangle_S 
    \geq \Theta(\tau^3)\norm{w^*}^2$. Thus,
    \begin{align*}
        \norm{R}_S^2\geq \frac{\langle |w^{*\top}x|-g, R(x)\rangle_S ^2}{\langle|w^{*\top}x|-g, |w^{*\top}x|-g\rangle_S }
        =\Theta(\tau^3)\norm{w^*}^2=\Theta\left(\frac{w_{min}^3}{r^3w_{max}^3}\delta^3\right)\norm{w^*}^2,
    \end{align*}
    where we choose $\tau=\frac{c_1w_{min}}{rw_{max}}\delta$. Since $\delta^3\geq \frac{C_1r^3w_{max}^3}{w_{min}^5}\epsilon$ with a sufficiently large constant $C_1$, we have $$\epsilon > \mathbb{E}_{x\sim N(0,I)}\left[\left(\sum_{i=1}^m \norm{w_i}|w_i^\top x| - \sum_{i=1}^r |w_i^{*\top} x|\right)^2\right]=\norm{R}^2\geq\norm{R}_S^2\geq\epsilon,$$ which is a contradiction.

    \paragraph{Part 2:}
    We aim to show that the total mass of student neuron is large. Assume toward contradiction that there exists a teacher neuron $w_i^*$ such that $\sum_{j\in T_i(\delta)}\norm{w_j} < \frac{1}{2}\norm{w_i^*}$. We follow the same notations and assumptions as mentioned at the beginning of Part 1. We also assume $\phi\leq c$ for a sufficient small constant $c$.
    
    We are going to show that when $\phi\leq\delta$,
    $h(\phi)\triangleq\langle|\bar{w}^{*\top}x|-g, |w^\top x|\rangle_S$ is non-increasing, i.e., $h^\prime(\phi)\leq 0$. With \eqref{lem-loss-delta-eq-2}\eqref{lem-loss-delta-eq-3}, we have
    \begin{align*}
        h^\prime (\phi)
        =& \left(\left(2 e^{-\tau^2/2} (e^{-\tau^2/2} - 1) + \sqrt{2\pi} e^{-\tau^2/2} \tau \erf\left(\frac{\tau}{\sqrt{2}}\right) \right) \frac{\erf(\frac{\tau\cot\phi} {\sqrt{
        2}})}{\pi \erf(\frac{\tau}{\sqrt{2}})} \right.\\
        &\left.+ 4 \OwenT(\tau, \cot\phi) -1 +\frac{2\phi}{\pi}\right) \sin\phi\\
        =& \underbrace{\left(2 e^{-\tau^2/2} (e^{-\tau^2/2} - 1) + \sqrt{2\pi} e^{-\tau^2/2} \tau \erf\left(\frac{\tau}{\sqrt{2}}\right) - \sqrt{\frac{\pi}{2}}\tau\erf\left(\frac{\tau}{\sqrt{2}}\right)\right)}_{h_1(\tau)} \frac{\erf(\frac{\tau\cot\phi}{\sqrt{2}})} {\pi\erf(\frac{\tau}{\sqrt{2}})} \sin\phi\\
        &+ \underbrace{\left(4 \OwenT(\tau, \cot\phi) -1 +\frac{2\phi}{\pi} + \frac{\tau}{\sqrt{2\pi}}\erf\left(\frac{\tau\cot\phi} {\sqrt{
        2}}\right)\right)}_{h_2(\tau,\phi)} \sin\phi,
    \end{align*}
    where $\OwenT(x,a) = \frac{1}{2\pi}\int_0^a e^{-x^2(1+t^2)/2}\frac{1}{1+t^2}\dif t$.
    
    For $h_1(\tau)$, we have
    \begin{align*}
        h_1(\tau) 
        =& 2 e^{-\tau^2/2} (e^{-\tau^2/2} - 1) + \sqrt{2\pi} e^{-\tau^2/2} \tau \erf\left(\frac{\tau}{\sqrt{2}}\right) - \sqrt{\frac{\pi}{2}}\tau\erf\left(\frac{\tau}{\sqrt{2}}\right)\\
        =& 2 e^{-\tau^2/2} (e^{-\tau^2/2} - 1) + \sqrt{\frac{\pi}{2}} e^{-\tau^2/2} \tau \erf\left(\frac{\tau}{\sqrt{2}}\right) + \sqrt{\frac{\pi}{2}}\tau\erf\left(\frac{\tau}{\sqrt{2}}\right)(e^{-\tau^2/2}-1)\\
        \leq& 2\left(-\frac{\tau^2}{2}+\frac{\tau^4}{8}\right) + \sqrt{\frac{\pi}{2}}\tau\frac{2}{\sqrt{\pi}}\frac{\tau}{\sqrt{2}}
        + \sqrt{\frac{\pi}{2}}\tau\frac{2}{\sqrt{\pi}}\frac{\tau}{\sqrt{2}}\left(-\frac{\tau^2}{2}+\frac{\tau^4}{8}\right)\\
        =& -\frac{\tau^4}{4} + \frac{\tau^6}{8}
        \leq 0,
    \end{align*}
    where in the first inequality we use $\erf(x)\leq \frac{2x}{\sqrt{\pi}}$, $e^{-x}-1\leq -x+\frac{x^2}{2}$ for $x\geq 0$ and $\tau\leq c$ in the last line.
    
    For $h_2(\tau,\phi)$, by Lemma \ref{lem-loss-delta-calculation-1} we have $h_2(\tau,\phi)\leq 0$ when $\tau,\phi\leq c$. 
    
    Therefore, we have $h^\prime(\phi)\leq 0$ when $\tau,\phi\leq c$. This implies that when $\tau,\phi\leq c$,
    \begin{align*}
        h(\phi)=\langle|\bar{w}^{*\top}x|-g, |w^\top x|\rangle_S\leq \langle|\bar{w}^{*\top}x|-g, |\bar{w}^{*\top} x|\rangle_S=h(0).
    \end{align*}
    
    In Part 1, we know that when $\tau\leq c, \frac{\tau}{\phi}\leq c$,
    $\left|\langle|\bar{w}^{*\top}x|-g, |w^\top x|\rangle_S\right| 
    \leq \Theta\left(\frac{\tau^4}{\phi}\right).$ Thus, if $\tau\leq \min\{c,c^2\}$ and $\phi\geq c$, we have
    \begin{align*}
    &\left|\langle|\bar{w}^{*\top}x|-g, |w^\top x|\rangle_S\right| 
    \leq \Theta\left(\frac{\tau^4}{\phi}\right)
    \leq \langle|\bar{w}^{*\top}x|-g, |\bar{w}^{*\top} x|\rangle_S = \Theta(\tau^3).
    \end{align*}
    
    In summary, we have when $\tau\leq \min\{c,c^2\}$,
    \begin{align*}
        \langle|\bar{w}^{*\top}x|-g, |w^\top x|\rangle_S\leq \langle|\bar{w}^{*\top}x|-g, |\bar{w}^{*\top} x|\rangle_S.
    \end{align*}
    
    Recall that $R(x) = \sum_{i=1}^m \norm{w_i}|w_i^\top x| - \sum_{i=1}^r |w_i^{*\top} x|$, we have
    \begin{align*}
        \langle |w^{*\top}x|-g, R(x)\rangle_S 
        \geq& 
        \norm{w^*}\left(\norm{w^*}-\sum_{j\in T_i(\delta)} \norm{w_j}^2\right)\langle|\bar{w}^{*\top}x|-g, |\bar{w}^{*\top} x|\rangle_S \\
        &- \sum_{j\in [m]\setminus T_i(\delta)}\norm{w_j}^2\norm{w^*} \Theta\left(\frac{\tau^4}{\delta}\right)-\sum_{j\neq i}\norm{w_j^*}\norm{w^*}\Theta\left(\frac{\tau^4}{\Delta}\right)\\
        =& \left(\Theta(\tau^3)\norm{w^*} - \Theta\left(\frac{rw_{max}\tau^4}{\delta}\right)\right)\norm{w^*},
    \end{align*}
    where we use $\delta\leq \Delta$ and Lemma \ref{lem-sum-norm-bound}.
    
    Then, following the same argument at the end of Part 1, we get the contradiction
    $$\epsilon > \mathbb{E}_{x\sim N(0,I)}\left[\left(\sum_{i=1}^m \norm{w_i}|w_i^\top x| - \sum_{i=1}^r |w_i^{*\top} x|\right)^2\right]=\norm{R}^2\geq\norm{R}_S^2\geq\epsilon,$$
    which finishes the proof.

\end{proof}

\subsection{Lemma \ref{lem-high-dim-small-loss-residual}: Property of The Residual}
Before we give the proof of second step (Lemma \ref{lem-high-dim-sum-student-small}) and third step (Lemma \ref{lem-high-dim-correlation-lower-bound}) of the proof sketch, we first present the following result that characterize the property of the residual. This lemma will be useful in the later analysis. Recall that we have the residual decomposition that $R(x)=R_1(x)+R_2(x)$. The two terms have different properties \--- one can further show that $R_1$ is ``flat" (whose value is uniformly bounded for all $x$) while $R_2$ is ``spiky" (whose value is large only in a local region), see also Figure \ref{fig-high-dim-residual}. More precisely, we have

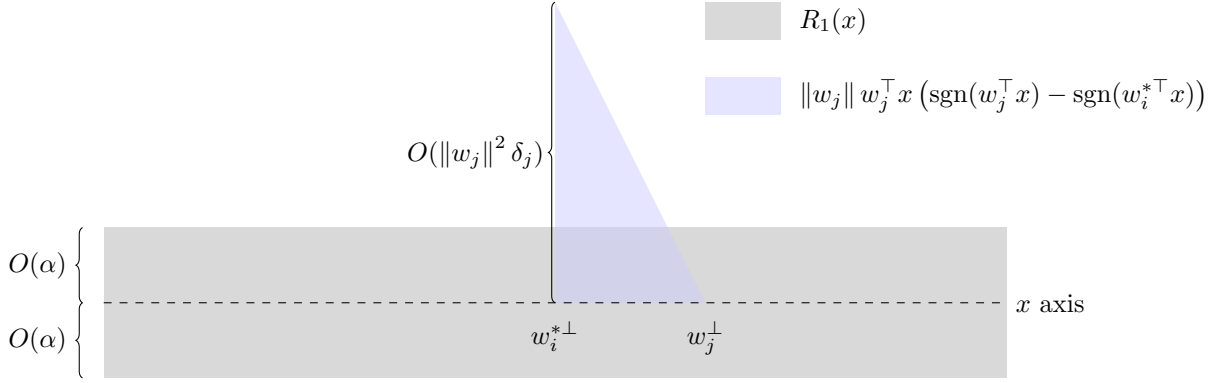
\begin{figure}[t]
    \centering
    \begin{tikzpicture}
    \filldraw[fill=gray!30!white, draw=none] (-6,-1) rectangle (6,1);
    \filldraw[draw=none, fill=blue!20, opacity=0.5] (0,0) -- (2,0) -- (0,4) -- cycle;
    
    \draw[dashed] (-6,0) -- (6,0) node[black,right]{$x$ axis};
    
    \node[label=below:$w_i^{*\perp}$] (a) at (0,0) {};
    \node[label=below:$w_j^{\perp}$] (a) at (2,0) {};
    %\node[label=below:$w_{i1}^\perp$] (b) at (-2,0) {};
    %\node[label=below:$w_{i2}^\perp$] (c) at (2,0) {};
    
    \draw[decorate,decoration={brace,raise=8pt}] (-6,-1) -- (-6,0) node[black,midway,xshift=-25pt]{$O(\alpha)$};
    \draw[decorate,decoration={brace,raise=8pt}] (-6,0) -- (-6,1) node[black,midway,xshift=-25pt]{$O(\alpha)$};
    \draw[decorate,decoration={brace,raise=0pt}] (0,0) -- (0,4) node[black,midway,xshift=-30pt]{$O(\norm{w_j}^2\delta_j)$};
    
    %legend
    \filldraw[fill=gray!30!white, draw=none] (2,3.5) rectangle (3,4);
    \filldraw[draw=none, fill=blue!20, opacity=0.5] (2,2.5) rectangle (3,3);
    \node[label=right:$R_1(x)$] (a) at (3,3.75) {};
    \node[label=right:$\norm{w_j}w_{j}^\top x \left(\sgn(w_{j}^\top x) - \sgn(w_i^{*\top} x)\right)$] (a) at (3,2.75) {};
    
    \end{tikzpicture}

    \caption{Illustration for the decomposition of residual $R(x)$ in Lemma \ref{lem-high-dim-small-loss-residual} (assume $\norm{x}=1$). Here $w^\perp$ represents a vector that is orthogonal to $w$, and $\norm{\sum_{j\in T_{i}}\norm{w_j}w_{j}-w_i^*}\leq\alpha$ for all $i\in[r]$.}
    \label{fig-high-dim-residual}
\end{figure}

%%%%%%%%%%%%%%%%%%%%%%%%%%%%%%%%%%%%%%%%%%%%%%%%%%%%%%%%%%%%%%%%%%%
% residual property
%%%%%%%%%%%%%%%%%%%%%%%%%%%%%%%%%%%%%%%%%%%%%%%%%%%%%%%%%%%%%%%%%%%%
\begin{restatable}[Property of The Residual]{lemmma}{lemHighDimSmallLossResidual} \label{lem-high-dim-small-loss-residual}
    Under Assumption \ref{as:norm}, if $\norm{v_i}\leq\alpha$ for all $i\in[r]$, then $R_1, R_2$ as defined in \eqref{eq:residual_decomp} satisfy $|R_1(x)| \leq \norm{x}r\alpha$ and $R_2(x)\geq 0$, for all $x$.
\end{restatable}

%\lemHighDimSmallLossResidual*
\begin{proof}
    It is easy to verify $R(x)=R_1(x)+R_2(x)$ for all $x$.
    
    For $R_1(x)$, we have
    \begin{align*}
        |R_1(x)| 
        =& \left|\sum_{i=1}^r v_i^\top x\sgn(w_i^{*\top}x)\right|
        \leq \sum_{i=1}^r \left|v_i^\top x\sgn(w_i^{*\top}x)\right|
        \leq \sum_{i=1}^r \norm{v_i}\norm{x}
        \leq r\alpha\norm{x}.
    \end{align*}
    
    For $R_2(x)$, note that for any $x$, $i\in[r]$ and $j\in T_i$, we have
    \begin{align*}
        w_{j}^\top x( \sgn(w_{j}^\top x) - \sgn(w_i^{*\top}x) )
        \geq 0. 
    \end{align*}
    
    Therefore,
    \begin{align*}
        R_2(x) = \sum_{i=1}^r\sum_{j\in T_i} \norm{w_j}w_{j}^\top x \left(\sgn(w_{j}^\top x) - \sgn(w_i^{*\top} x)\right)
        \geq 0.
    \end{align*}

\end{proof}

%%%%%%%%%%%%%%%%%%%%%%%%%%%%%%%%%%%%%%%%%%%%%%%%%%%%%%%%%%%%%%%%%%%
% average student close to teacher
%%%%%%%%%%%%%%%%%%%%%%%%%%%%%%%%%%%%%%%%%%%%%%%%%%%%%%%%%%%%%%%%%%%%
\subsection{Proof of Lemma \ref{lem-high-dim-sum-student-small}}\label{sec-pf-high-dim-sum-student-small}
In this subsection, we first following the proof sketch in Section~\ref{subsec:residualdecomp} to give a proof of Lemma~\ref{lem-high-dim-sum-student-small} (the second step of proving Lemma~\ref{lem-repara-high-dim-small-loss-descent-dir}). Then, we give the proof of Lemma~\ref{lem-approximate-hessian} and Lemma \ref{lem-r2-bound} respectively. 
In this subsection, we focus on the residual which is invariant with the sign change of student neuron. Recall the discussion at the beginning of Section~\ref{subsec:descentdirection}, we can assume $w_j^\top w_i^*\ge 0$ for all $j\in T_i$ due to the symmetry of the absolute value function.

\lemHighDimSumStudentSmall*

\begin{proof}
    For $R_1$, from Lemma \ref{lem-approximate-hessian} we have $\norm{R_1}^2 
    =\Omega\left(\frac{\Delta^3}{r^3}\right)\norm{v}^2.$
    
    For $R_2$, from Lemma \ref{lem-r2-bound} we have $\norm{R_2}^2=O(r^{5/2}w_{max}^{1/2}\epsilon^{3/4})$.
    
    With $\norm{R}^2=2L(W)\leq2\epsilon$, we have
    \begin{align*}
        \Omega\left(\frac{\Delta^{3/2}}{r^{3/2}}\right)\norm{v}
        \leq \norm{R_1}
        \leq \norm{R}+\norm{R_2}
        = O(\epsilon^{1/2}+r^{5/4}w_{max}^{1/4}\epsilon^{3/8}),
    \end{align*}
    which leads to $\norm{\sum_{j\in T_{i}}\norm{w_j}w_{j}-w_i^*}^2=\norm{v_i}^2\leq\norm{v}^2=O\left(r^{11/2}w_{max}^{1/2}\Delta^{-3}\cdot\epsilon^{3/4}\right)$ with our choice of $\epsilon$.
\end{proof}

%%%%%%%%%%%%%%%%%%%%%%%%%%%%%%%%%%%%%%%%%%%%%%%%%%%%%%%%%%%%%%%%%%%%%%%%
% Positive-definite matrix
%%%%%%%%%%%%%%%%%%%%%%%%%%%%%%%%%%%%%%%%%%%%%%%%%%%%%%%%%%%%%%%%%%%%%%%%
\subsubsection{Proof of Lemma \ref{lem-approximate-hessian}}\label{sec-pf-approximate-hessian}

Before proving Lemma \ref{lem-approximate-hessian} (the first step of proving Lemma \ref{lem-high-dim-sum-student-small}), we need the following lemma that shows $M$ is positive definite. The proof relies on the fact that teacher neurons are $\Delta$-separated. Intuitively, this matrix $M$ corresponds to the Hessian matrix at global minima in the exact-parameterization case. When teacher neurons are well-separated, one can imagine this matrix is positive definite.
\begin{lemma}\label{lem-appendix-hessian}
Under Assumption \ref{as:separate}, for matrix $M\in\R^{dr\times dr}$ defined in \eqref{eq:M_def}, we have $\lambda_{min}(M)=\Omega(\Delta^3/r^3)$.
\end{lemma}

Now given the relation between $\norm{R_1}$ and matrix $M$, we are ready to prove Lemma \ref{lem-approximate-hessian}.

\lemApproximateHessian*

\begin{proof}
    Recall that $R_1(x)=\sum_{i=1}^r v_i^\top x \sgn(w_i^{*\top} x)$. We have
    \begin{align*}
        \norm{R_1}^2 
        = \E_{x\sim N(0,I)}\left[\left(\sum_{i=1}^r
        v_i^\top x\sgn(w_i^{*\top} x)\right)^2\right]
        = v^\top M v,
    \end{align*}
    where $M\in\R^{dr\times dr}$ is a matrix with block entry $M_{ij}=\E_x[xx^\top\sgn(w_i^{*\top}x)\sgn(w_j^{*\top}x)]$ as defined in \eqref{eq:M_def} and $v=(v_1,\ldots,v_r)\in\mathbb{R}^{dr}$. Therefore, we have $\norm{R_1}^2\geq \Omega(\Delta^3/r^3)\norm{v}^2$.
\end{proof}

We now give the proof of Lemma \ref{lem-appendix-hessian}. The high-level idea of this proof is that when teacher neurons are well-separated, it is hard to cancel the nonlinearity induced by each teacher neuron. 
\begin{proof}[Proof of Lemma \ref{lem-appendix-hessian}]
    
    WLOG, assume $\norm{w_i^*}=1$. Denote vector $v\in\mathbb{R}^{dr}$ with $\norm{v}=1$ as $v=(v_1,\ldots,v_r)^\top$, where $v_i\in\mathbb{R}^d$. It suffices to give lower bound on $v^\top M v$.
    \begin{equation}\label{approximate-hession-eq-1}
        \begin{aligned}
        v^\top M v 
        =& \E_{x} \left[\left(\sum_{j=1}^r v_j^\top x \sgn(w_j^{*\top} x) \right)^2\right]
        \geq \sum_{i=1}^r \E_{x} \left[\left(\sum_{j=1}^r v_j^\top x \sgn(w_j^{*\top} x) \right)^2\mathbb{I}_{S_i}\right],
        \end{aligned}
    \end{equation}
    where set $S_i$ will be determined later and satisfies $S_i\cap S_j=\varnothing$ for $i\neq j$. 
    
    Denote $S(w,\delta)=\{x\in\mathbb{R}^d||w^\top x|\leq \delta\}$ as the  nonlinear region of neuron $w$. We will use $S_i^*(\delta)$to represent $S(w_i^*,\delta)$. Also denote 
    \begin{align*}
        &A(\alpha,\beta,\delta_1,\delta_2)\\
        &=\{x||\alpha^\top x|\leq \delta_1, \text{there exists $y$ such that } |\alpha^\top y|\leq \delta_1, |\beta^\top y|\leq \delta_2, (I-\alpha\alpha^\top)x=(I-\alpha\alpha^\top)y\}
    \end{align*} 
    as the projection of $S(\alpha,\delta_1)\cap S(\beta,\delta_2)$ onto $S(\alpha,\delta_1)$ in the direction of $\alpha$. We will use $A_i^*(\beta,\delta_1,\delta_2)$ to represent $A(w_i^*,\beta,\delta_1,\delta_2)$.
    
    Let set $$S_i = S_i^*(\delta_1)\setminus \left(\cup_{j\in[r],j\neq i} A_i^*(w_j^*,\delta_1,\delta_1)\right)$$ with $\delta_1=O(\frac{\Delta}{r})$ and $\delta_1\leq 1$. It is easy to see that $S_i\cap S_j=\varnothing$ for $i\neq j$, since $S_i^*(\delta_1)\cap S_j^*(\delta_1)\subseteq S_i^*(\delta_1)\cap A_i^*(w_j^*,\delta_1,\delta_1)$. By Lemma \ref{lem-nonlinear-region-size}, we know $$\E_x[\mathbb{I}_{S_i^*(\delta_1)}]\geq \sqrt{\frac{2}{\pi}}e^{-\frac{1}{2}}\delta_1,\quad
    \E_x[\mathbb{I}_{A_i^*(w_j^*,\delta_1,\delta_1)}]\leq \frac{\delta_1^2(1+\cos\Delta)}{\pi\sin\Delta}.$$ 
    Together with $\sin x\geq \frac{2x}{\pi}$ for $0\leq x \leq \frac{\pi}{2}$, we have
    \begin{align*}
        \E_x[\mathbb{I}_{S_i}]\geq \Omega\left(\delta_1-\frac{r\delta_1^2}{\Delta}\right)=c_1\delta_1,
    \end{align*}
    where $c_1$ is a constant.
    
    In the following, we focus on one term in \eqref{approximate-hession-eq-1}, and show it can be lower bounded. That is,
    \begin{align*}
        \E_{x\sim N(0,I)} \left[\left(\sum_{j=1}^r v_j^\top x \sgn(w_j^{*\top} x) \right)^2\mathbb{I}_{S_i}\right].
    \end{align*}

    For any $x^+\in S_i^+\triangleq S_i\cap S_i^*(\frac{\delta_1}{2})\cap\{x|w_i^{*\top}x\geq 0\}$, 
    let 
    \begin{align*}
        x^- = x^+ - \frac{\delta_1}{2}w_i^*,\quad 
        x^{++} = x^+ + \frac{\delta_1}{2}w_i^*,\quad
        x^{--} = x^+ - \delta_1 w_i^*.
    \end{align*}
    We know that when $x^{+}\in S_i$, $x^-,x^{++},x^{--}\in S_i$. Further, we have
    \begin{align*}
        x^-\in S_i^-&\triangleq S_i\cap S_i^*(\frac{\delta_1}{2})\cap\{x|w_i^{*\top}x\leq 0\},
        \\
        x^{++}\in S_i^{++}&\triangleq S_i\cap \left(S_i^*(\delta_1)\setminus S_i^*(\frac{\delta_1}{2})\right)\cap\{x|w_i^{*\top}x\geq 0\},\\
        x^{--}\in S_i^{--}&\triangleq S_i\cap \left(S_i^*(\delta_1)\setminus S_i^*(\frac{\delta_1}{2})\right)\cap\{x|w_i^{*\top}x\leq 0\},
    \end{align*}
    and $S_i^+\cup S_i^{++}\cup S_i^- \cup S_i^{--}=S_i$.
    
    Consider $g(x)=\sum_{j=1}^r v_j^\top x \sgn(w_j^{*\top} x)$ on these four data points $x^+,x^-,x^{++},x^{--}$. Note that when data has the form $x=x^{+}+\gamma w_i^*$ for any fixed $x^+\in S_i$, we have $g(x)= \alpha+\beta^\top x + v_i^\top x \sgn(w_i^{*\top} x)$, by the choice of $S_i$. Here $\alpha,\beta$ only depend on the choice of $x^+$. Thus,
    \begin{align*}
        &0 \cdot g(x^{++})-g(x^{--})-g(x^+)+2g(x^-)\\
        =&  0\cdot \left(\alpha+\beta^\top (x^+ + \frac{\delta_1}{2}w_i^*) + v_i^\top (x^+ + \frac{\delta_1}{2}w_i^*)\right)
         - \left(\alpha+\beta^\top (x^+ - \delta_1 w_i^*) - v_i^\top (x^+ - \delta_1 w_i^*)\right)\\
         &- (\alpha+\beta^\top x^+ + v_i^\top x^+)
         +2\left(\alpha+\beta^\top (x^+ - \frac{\delta_1}{2}w_i^*) - v_i^\top (x^+ - \frac{\delta_1}{2}w_i^*)\right)\\
        =& -2v_i^\top x^+,
    \end{align*}
    which indicates that $\max\{|g(x^{++})|,|g(x^{--})|,|g(x^+)|,|g(x^-)|\}\geq \frac{1}{2}|v_i^\top x^+|$. 
    
    Note that
    \begin{align*}
        \E_x[g^2(x)\mathbb{I}_{S_i^-}]
        =& \frac{1}{(2\pi)^{d/2}}\int_{S_i^-} g^2(x) e^{-\norm{x}^2/2}\dif x\\
        =& \frac{1}{(2\pi)^{d/2}}\int_{S_i^+} g^2(x^+-\frac{\delta_1}{2}w_i^*) e^{-\norm{x^+-\frac{\delta_1}{2}w_i^*}^2/2}\dif x^+\\
        \geq& \frac{1}{(2\pi)^{d/2}}\int_{S_i^+} g^2(x^+-\frac{\delta_1}{2}w_i^*) e^{-\left(\norm{x^+}^2+\frac{\delta_1^2}{4}\right)/2}\dif x^+\\
        \geq& e^{-1/8} \E_{x^+}[g^2(x^-)\mathbb{I}_{S_i^+}],
    \end{align*}
    where the second line is because we change $x$ to $x^+-\frac{\delta_1}{2}w_i^*$, the third line is because $w_i^{*\top}x^+\geq 0$, and the last line is due to $\delta_1\leq 1$ and $x^-=x^+-\frac{\delta_1}{2}w_i^*$. Similarly, we have $\E_x[g^2(x)\mathbb{I}_{S_i^{--}}]\geq e^{-1/2} \E_{x^+}[g^2(x^{--})\mathbb{I}_{S_i^+}]$ and $\E_x[g^2(x)\mathbb{I}_{S_i^{++}}]\geq e^{-3/8} \E_{x^+}[g^2(x^{++})\mathbb{I}_{S_i^+}]$.
    
    Therefore,
    \begin{align*}
        \E_x[g^2(x)\mathbb{I}_{S_i}]
        =& \E_x[g^2(x)(\mathbb{I}_{S_i^+}+ \mathbb{I}_{S_i^-} + \mathbb{I}_{S_i^{++}}+\mathbb{I}_{S_i^{--}})]\\
        \geq& e^{-1/2}\E_x[(g^2(x^+)+g^2(x^-)+g^2(x^{++})g^2(x^{--}))\mathbb{I}_{S_i^+}]\\
        \geq& \frac{1}{4\sqrt{e}}\E_{x^+}[|v_i^\top x^+|^2\mathbb{I}_{S_i^+}],
    \end{align*}
    where we use $\max\{|g(x^{++})|,|g(x^{--})|,|g(x^+)|,|g(x^-)|\}\geq \frac{1}{2}|v_i^\top x^+|$.
    
    Recall that $$S_i^+ = S_i\cap S_i^*(\frac{\delta_1}{2})\cap\{x|w_i^{*\top}x\geq 0\} = \left(S_i^*(\frac{\delta_1}{2})\cap\{x|w_i^{*\top}x\geq 0\}\right)\setminus \left(\cup_{j\in[r],j\neq i} A_i^*(w_j^*,\delta_1,\delta_1)\right).$$ By Lemma \ref{lem-nonlinear-region-size}, we know 
    \begin{align*}
        \E_x[\mathbb{I}_{S_i^+}]
        \geq \frac{1}{2}\left(\E_x[\mathbb{I}_{S_i(\frac{\delta_1}{2})}] - \sum_{j\in[r],j\neq i} \E_x[\mathbb{I}_{A_i^*(w_j^*,\delta_1,\delta_1)}]\right) 
        \geq \Omega\left(\delta_1 - \frac{r\delta_1^2}{\Delta}\right)
        = c_2\delta_1.
    \end{align*}
    Let $S_i^{+'}= S(\bar{v}_i, \sqrt{\frac{\pi}{8}}c_2\delta_1)$. By Lemma \ref{lem-nonlinear-region-size}, we know $\E_x[\mathbb{I}_{S_i^{+'}}]\leq \frac{c_2}{2}\delta_1$. Hence, $\E[\mathbb{I}_{S_i^+\setminus S_i^{+'}}]\geq \frac{c_2}{2}\delta_1$ and $|v_i^\top x^+|\geq \sqrt{\frac{\pi}{8}}c_2\delta_1\norm{v_i}$ for $x^+\in S_i^+\setminus S_i^{+'}$. Therefore,
    \begin{align*}
        \E_x[g^2(x)\mathbb{I}_{S_i}]
        \geq& \frac{1}{4\sqrt{e}}\E_{x^+}[|v_i^\top x^+|^2\mathbb{I}_{S_i^+\setminus S_i^{+'}}]
        \geq \frac{1}{4\sqrt{e}}\frac{\pi}{8}c_2^2\delta_1^2\norm{v_i}^2\E_{x^+}[\mathbb{I}_{S^+\setminus S^{+'}}]\\
        \geq& \frac{\pi}{64\sqrt{e}}c_2^3\delta_1^3\norm{v_i}^2
        =\frac{c_3\Delta^3}{r^3}\norm{v_i}^2.
    \end{align*}
    
    Therefore, by \eqref{approximate-hession-eq-1}, we have
    \begin{align*}
        v^\top M v 
        \geq \sum_{i=1}^r\E_x[g^2(x)\mathbb{I}_{S_i}]
        \geq \sum_{i=1}^r \frac{c_3\Delta^3}{r^3}\norm{v_i}^2
        =\Omega\left(\frac{\Delta^3}{r^3}\right)\norm{v}^2,
    \end{align*}
    which implies that $\lambda_{min}(M)=\Omega\left(\frac{\Delta^3}{r^3}\right)$.
\end{proof}

\subsubsection{Proof of Lemma \ref{lem-r2-bound}}
We need the following two lemmas to prove Lemma \ref{lem-r2-bound} (the second step of proving Lemma~\ref{lem-high-dim-sum-student-small}). 
Lemma~\ref{lem-sum-norm-bound} claims that when loss is small, the sum of norm of the student neurons would be bounded by a fixed quantity. Lemma~\ref{lem-high-dim-ls-small-improve} claims student neurons should not be far away from teacher neuron when loss is small. See more discussions and corresponding proofs in the next two subsections.

\begin{restatable}{lemmma}{lemSumNormBound}
\label{lem-sum-norm-bound}
    Under Assumption \ref{as:norm}, if loss $L(W)=O(r^2w_{max}^2)$, then $\sum_{i=1}^r\sum_{j\in T_i}\norm{w_{j}}^2=O(rw_{max}).$
\end{restatable}

\begin{restatable}{lemmma}{lemHighDimLsSmallImprove}
\label{lem-high-dim-ls-small-improve}
Under Assumption \ref{as:separate}, \ref{as:norm}, 
there exists a threshold $\epsilon_0=poly(\Delta,r^{-1},w_{max}^{-1},w_{min})$ such that for any $W$ satisfying loss $\epsilon\triangleq L(W) \leq\epsilon_0$, we have $\sum_{i=1}^r\sum_{j\in T_i} \norm{w_j}^2\delta_j^2=O(\epsilon^{1/2})$.
\end{restatable}

Now we are ready to prove Lemma \ref{lem-r2-bound} by combining above two lemmas and Holder's inequality.

\lemRtwoBound*
\begin{proof}
    Recall that $R_2(x)=\sum_{i=1}^r\sum_{j\in T_i} \norm{w_j}w_{j}^\top x \left(\sgn(w_{j}^\top x) - \sgn(w_i^{*\top} x)\right)$. We have
    \begin{align*}
        \norm{R_2}^2 
        =& \E_{x\sim N(0,I)}\left[\left(\sum_{i=1}^r\sum_{j\in T_i} \norm{w_j}w_{j}^\top x \left(\sgn(w_{j}^\top x) - \sgn(w_i^{*\top} x)\right)\right)^2\right]\\
        \leq& r\sum_{i=1}^r \E_{x\sim N(0,I)}\left[ \left(\sum_{j\in T_i} \norm{w_j}w_{j}^\top x \left(\sgn(w_{j}^\top x) - \sgn(w_i^{*\top} x)\right) \right)^2\right].
    \end{align*}
    
    In the following, we focus on one term in the above sum. Since we use absolute value function as activation, WLOG assume $w_j^\top w_i^*\ge 0$ for all $j\in T_i$ as discussed at the beginning of Section~\ref{subsec:descentdirection}. We have
    \begin{align*}
        &\E_{x\sim N(0,I)}\left[ \left(\sum_{j\in T_{i}} \norm{w_j}w_{j}^\top x \left(\sgn(w_{j}^\top x) - \sgn(w_i^{*\top} x)\right) \right)^2\right]\\
        =& \sum_{j,k\in T_{i}} \E_{x\sim N(0,I)}\left[ \norm{w_j}w_{j}^\top x \cdot \norm{w_k}w_{k}^\top x \left(\sgn(w_{j}^\top x) - \sgn(w_i^{*\top} x)\right) \left(\sgn(w_{k}^\top x) - \sgn(w_i^{*\top} x)\right) \right]\\
        \leq& \sum_{j,k\in T_{i}} c_1 \norm{w_{j}}^2\norm{w_{k}}^2\delta_{j}\delta_{k}\min\{\delta_{j},\delta_{k}\}\\
        \leq& c_1 \left(\sum_{j\in T_{i}} \norm{w_{j}}^2\delta_{j}^{3/2}\right)^2\\
        \leq& c_1 \left(
        \left(\sum_{j\in T_{i}} \norm{w_{j}}^2\right)^{1/4}
        \left(\sum_{j\in T_{i}} \norm{w_{j}}^2\delta_{j}^{2}\right)^{3/4}\right)^2\\
        =& O(r^{1/2}w_{max}^{1/2}\epsilon^{3/4}),
    \end{align*}
    where we use Lemma \ref{lem-sum-student-small-calculation-1} in the first inequality, Holder's inequality in the last ineqality, and Lemma \ref{lem-sum-norm-bound} and Lemma \ref{lem-high-dim-ls-small-improve} in the last line.
    
    Thus, we have $\norm{R_2}^2=O(r^{5/2}w_{max}^{1/2}\epsilon^{3/4})$.
\end{proof}

%%%%%%%%%%%%%%%%%%%%%%%%%%%%%%%%%%%%%%%%%%%%%%%%%%%%%%%%%%%%%%%%%%%%%%%%
% sum norm bound
%%%%%%%%%%%%%%%%%%%%%%%%%%%%%%%%%%%%%%%%%%%%%%%%%%%%%%%%%%%%%%%%%%%%%%%%
\subsubsection{Proof of Lemma \ref{lem-sum-norm-bound}}
Intuitively, when loss is small, the norm of student neurons are bounded. The proof follows by simple calculations.
\lemSumNormBound*

\begin{proof}
    We have 
    \begin{align*}
        \E_x\left[\left(\sum_{i=1}^r\sum_{j\in T_i}|w_{j}^\top x|\right)^2\right]
        \leq& 2\E_x\left[\left(\sum_{i=1}^r\sum_{j\in T_i}|w_{j}^\top x|-\sum_{i=1}^{r}|w_i^{*\top}x|\right)^2\right]
        + 2\E_x\left[\left(\sum_{i=1}^{r}|w_i^{*\top}x|\right)^2\right]\\
        =& O(r^2w_{max}^2).
    \end{align*}
    Also, we have
    \begin{align*}
        \E_x\left[\left(\sum_{i=1}^r\sum_{j\in T_i}|w_{j}^\top x|\right)^2\right]
        \geq c \left(\sum_{i=1}^r\sum_{j\in T_i}\norm{w_{j}}\right)^2,
    \end{align*}
    where $c$ is a constant.
    Thus, $\sum_{i=1}^r\sum_{j\in T_i}\norm{w_{j}}=O(rw_{max})$.
\end{proof}

%%%%%%%%%%%%%%%%%%%%%%%%%%%%%%%%%%%%%%%%%%%%%%%%%%%%%%%%%%%%%%%%%%%%%%%%
% far away student are small
%%%%%%%%%%%%%%%%%%%%%%%%%%%%%%%%%%%%%%%%%%%%%%%%%%%%%%%%%%%%%%%%%%%%%%%%
\subsubsection{Proof of Lemma \ref{lem-high-dim-ls-small-improve}}\label{sec-pf-lem-high-dim-ls-small-improve}
Recall that $\delta_{j} = \angle(w_i^*,w_{j})$ as the angle between $w_i^*$ and $w_j$ for all $j\in T_i$.
Lemma \ref{lem-high-dim-ls-small-improve} upper-bounds the weighted norm of student neurons by the value of the loss, where the weights depend on the angle between student and teacher neuron. Since more weights are assigned to the student neurons that are far away from any teacher neuron, this lemma implies that  if there are many far-away neurons with large weights, then the loss is also large. The proof idea is to use test function as described in Section~\ref{subsec:testfunction}. We construct a test function such that it has large correlation with far-away student neurons and almost zero correlation with teacher neurons and near-by student neurons.

\lemHighDimLsSmallImprove*

\begin{proof}
    Since we use absolute value function as activation, WLOG assume $w_j^\top w_i^*\ge 0$ for all $j\in T_i$ as discussed at the beginning of Section~\ref{subsec:descentdirection}. We will first choose our test function in (i) and then complete the proof in (ii).
    
    \paragraph{(i) Choose test function.}
    
    Let $h(x) = \sqrt{\frac{\pi}{2}} - \sum_{k=1}^r \frac{1}{\hat{\sigma_l}}h_l(\bar{w}_k^{*\top}x)$ be the test function, where $h_l(x)$ is $l$-th Hermite polynomial, $\hat{\sigma}_l$ is $l$-th Hermite coefficient of absolute value function and $l=2\max\{\lceil \log_{\frac{1}{\cos(\Delta/2)}}\frac{1}{\epsilon}\rceil,1\}$. See Section \ref{sec-hermite} for the definition of Hermite polynomial and some properties. By Lemma \ref{lem-hermite-coefficient}, we know $\hat{\sigma}_l^2=\Theta(l^{-3/2})$. We are going to estimate $\langle h(x), R(x)\rangle$, where $R(x)=\sum_{i=1}^m \norm{w_i}|w_i^\top x| - \sum_{i=1}^r|w_i^{*\top}x|$ is the residual.

    By Lemma \ref{lem-hermite-coefficient} and Claim \ref{lem-hermite-product}, for any $w$ with $\norm{w}=1$, we have
    \begin{align*}
        &\langle h_l(\bar{w}_i^{*\top}x),|w^\top x|\rangle
        = \sum_{k=0}^\infty \hat{\sigma}_k \E[ h_l(\bar{w}_i^{*\top}x)h_k(w^\top x)]
        = \hat{\sigma}_l \langle \bar{w}_i^*, w\rangle^l,\\
        &\langle 1,|w^\top x|\rangle = \sqrt{\frac{2}{\pi}}.
    \end{align*}

    Denote $\phi=\angle(w_i^*, w)$. When $\sin^2\phi\leq 2/l$, we have
    \begin{align*}
        \langle \frac{1}{\hat{\sigma}_l} h_l(\bar{w}_i^{*\top}x),|w^\top x|\rangle
        = \cos^l \phi = (1-\sin^2\phi)^{l/2}\leq 1-\frac{l}{4}\sin^2\phi\leq 1-\frac{\sin^2\phi}{4},
    \end{align*}
    where we use $(1-x)^n\leq 1-\frac{nx}{2}$ for $x\in[0,1/n]$. When $\sin^2\phi > 2/l$, we have
    \begin{align*}
        \langle \frac{1}{\hat{\sigma}_l} h_l(\bar{w}_i^{*\top}x),|w^\top x|\rangle
        = \cos^l \phi = (1-\sin^2\phi)^{l/2}\leq \left(1-\frac{2}{l}\right)^{l/2}\leq \frac{1}{e}\leq 1-\frac{\sin^2\phi}{4}.
    \end{align*}
    This implies that $\langle \frac{1}{\hat{\sigma}_l} h_l(\bar{w}_i^{*\top}x),|w^\top x|\rangle \leq 1-\frac{\sin^2\phi}{4}$ holds for any $\phi$.
    
    Further, when $\angle(w_i^*,w)\geq \frac{\Delta}{2}$, by the choice of $l$, we have
    \begin{align*}
        0\leq \langle \frac{1}{\hat{\sigma}_l} h_l(\bar{w}_i^{*\top}x),|w^\top x|\rangle \leq \cos^{l} \frac{\Delta}{2}\leq \epsilon.
    \end{align*}

    Note that for $j\in T_i$ and $k\neq i$, we have $\Delta\leq \angle(w_k^*,w_i^*)\leq \angle(w_k^*,w_j) + \angle(w_i^*,w_j)\leq 2\angle(w_k^*,w_j)$. This indicates that $\angle(w_k^*,w_j)\geq \Delta/2$ for all $k\neq i$. Thus, for $j\in T_i$, we have
    \begin{align*}
        \langle h(x), |w_j^\top x|\rangle
        &= \sqrt{\frac{\pi}{2}}\langle 1,|w_j^\top x|\rangle - \sum_{k=1}^r \langle \frac{1}{\hat{\sigma}_l} h_l(\bar{w}_k^{*\top}x),|w_j^\top x|\rangle\\
        &\geq \left(1 - \left(1-\frac{\sin^2\delta_j}{4}\right) - (r-1)\epsilon\right)\norm{w_j}\\
        &= \frac{1}{4}\norm{w_j}\sin^2\delta_j - (r-1)\norm{w_j}\epsilon.
    \end{align*}
    We also have for every teacher neuron $w_j^*$,
    \begin{align*}
        \langle h(x), |w_j^{*\top} x|\rangle
        &= \sqrt{\frac{\pi}{2}}\langle 1,|w_j^{*\top} x|\rangle - \sum_{k=1}^r \langle \frac{1}{\hat{\sigma}_l} h_l(\bar{w}_k^{*\top}x),|w_j^{*\top} x|\rangle\\
        &= \left(1- 1 - \sum_{k\neq j} \langle \frac{1}{\hat{\sigma}_l} h_l(\bar{w}_k^{*\top}x),|\bar{w}_j^{*\top}x|\rangle\right)\norm{w_j^*} \leq 0.
    \end{align*}
    
    Therefore, for residual $R(x)=\sum_{i=1}^m \norm{w_i}|w_i^\top x| - \sum_{i=1}^r|w_i^{*\top}x|$, we have
    \begin{align*}
        \langle h(x), R(x)\rangle \geq \frac{1}{4}\sum_{i=1}^m \norm{w_i}\sin^2\delta_i - (r-1)\sum_{i=1}^m \norm{w_i}\epsilon
        \geq  \frac{1}{4}\sum_{i=1}^m \norm{w_i}\sin^2\delta_i - O(r^2w_{max}\epsilon),
    \end{align*}
    where we use Lemma \ref{lem-sum-norm-bound} in the last inequality.
    
    \paragraph{(ii) Complete Proof.}
    
    Recall that $\norm{R}^2=2L(W)\leq 2\epsilon$. From (i) we could have
    \begin{align*}
        \norm{h}\sqrt{2\epsilon}\geq \norm{h}\norm{R}
        \geq \langle h(x), R(x)\rangle 
        \geq  \frac{1}{4}\sum_{i=1}^m \norm{w_i}\sin^2\delta_i - O(r^2w_{max}\epsilon),
    \end{align*}
    
    We now estimate the norm of $h$. Recall that $\hat{\sigma}_l^2=\Theta(l^{-3/2})$, we have
    \begin{align*}
        \norm{h}^2 &= \E\left[\left(\sqrt{\frac{\pi}{2}} - \sum_{k=1}^r \frac{1}{\hat{\sigma_l}}h_l(\bar{w}_k^{*\top}x)\right)^2\right]
        = \frac{\pi}{2} +  + \frac{1}{\hat{\sigma}_l^2}\sum_{k,j=1}^r \langle \bar{w}_k^*, \bar{w}_j^*\rangle^l\\
        &= \frac{\pi}{2} + O(l^{3/2} r^2 \epsilon)
        = O\left(1 + r^2 \epsilon\log^{3/2}_{\frac{1}{\cos(\Delta/2)}}\frac{1}{\epsilon}\right).
    \end{align*}
    
    Therefore, we have
    \begin{align*}
        \frac{1}{\pi^2}\sum_{i=1}^m \norm{w_i}\delta_i^2 
        \leq \frac{1}{4}\sum_{i=1}^m \norm{w_i}\sin^2\delta_i
        &\leq \norm{h}\sqrt{2\epsilon}+ O(r^2w_{max}\epsilon)\\
        &= O\left(\epsilon^{1/2} + r^2w_{max}\epsilon + r^2 \epsilon^{3/2}\log^{3/2}_{\frac{1}{\cos(\Delta/2)}}\frac{1}{\epsilon}\right),
    \end{align*}
    which finishes the proof.
\end{proof}

%%%%%%%%%%%%%%%%%%%%%%%%%%%%%%%%%%%%%%%%%%%%%%%%%%%%%%%%%%%%%%%%%%%%%%%%
% correlation lower bound
%%%%%%%%%%%%%%%%%%%%%%%%%%%%%%%%%%%%%%%%%%%%%%%%%%%%%%%%%%%%%%%%%%%%%%%%

\subsection{Proof of Lemma \ref{lem-high-dim-correlation-lower-bound}}\label{sec-pf-correlation-lower-bound}
Finally, we give the proof of Lemma~\ref{lem-high-dim-correlation-lower-bound}, which is the third step of proving Lemma~\ref{lem-repara-high-dim-small-loss-descent-dir}. It shows that when the first two steps (Lemma~\ref{lem-high-dim-loss-delta} and Lemma \ref{lem-high-dim-sum-student-small}) hold, the inner product between the gradient and descent direction can be lower bounded. The proof is involved and relies on algebraic computations.

\lemHighDimCorrelationLowerBound*

\begin{proof}
 We focus on the case where $w_j^\top w_i^*\ge 0$ for all $j\in T_i$. The general case directly follows from Lemma~\ref{lem:symmetric}.
 Recall that residual $R(x)=\sum_{i=1}^m\norm{w_i}|w_i^\top x|-\sum_{i=1}^r|w_i^{*\top} x|$. For any student neuron $w_{j}$ where $j\in T_i$, we have
    \begin{align*}
        &\langle\nabla_{w_{j}}L(W),(I+\bar{w}_{j}\bar{w}_{j})^{-1}(w_{j}-q_{ij}w_i^*)\rangle\\
        =& \E_x\left[ R(x)\norm{w_{j}}(w_{j}-q_{ij}w_i^*)^\top x\sgn(w_{j}^\top x)
        \right]\\
        =& \E_x\left[ R(x) \left(\norm{w_{j}}|w_{j}^\top x| - q_{ij}\norm{w_{j}}|w_i^{*^\top}x|\right)\right]\\
        &+ \E_x\left[R(x) q_{ij}\norm{w_{j}}w_i^{*\top}x \left(\sgn(w_i^{*\top}x)-\sgn(w_{j}^\top x)\right) \right].
    \end{align*}
    Sum over all student neurons $w_{j}$, we have
    \begin{equation}\label{eq-repara-high-dim-small-loss-descent-dir-1}
    \begin{aligned}
        &\sum_{i=1}^r\sum_{j\in T_i} \langle\nabla_{w_{j}}L(W),(I+\bar{w}_{j}\bar{w}_{j})^{-1}(w_{j}-q_{ij}w_i^*)\rangle\\
        =& \E_x\left[ R^2(x)  
        + R(x) \sum_{i=1}^r\sum_{j\in T_i} q_{ij}\norm{w_j}w_i^{*\top}x \left(\sgn(w_i^{*\top}x)-\sgn(w_{j}^\top x)\right) \right]\\
        =& \norm{R}^2 + \E_x\left[ R(x) 
        \underbrace{ \sum_{i=1}^r\sum_{j\in T_{i}(\delta_{max})} q_{ij}\norm{w_{j}}w_i^{*\top}x \left(\sgn(w_i^{*\top}x)-\sgn(w_{j}^\top x)\right) }_{I_1(x)} \right],
    \end{aligned}
    \end{equation}
    where the last line is because we set $q_{ij}=0$ if $j\not\in  T_{i}(\delta_{max})$.
    
    We are going to lower bound the second term $\E_x[R(x)I_1(x)]$ in the last line. Note that when $\sgn(w_i^{*\top}x) \neq \sgn(w_{j}^\top x)$, we have $w_i^{*\top}x \left(\sgn(w_i^{*\top}x)-\sgn(w_{ij}^\top x)\right) = 2|w_i^{*\top}x|$. Hence,
    \begin{align*}
        &\E_x\left[ R(x) q_{ij}\norm{w_{j}}w_i^{*\top}x \left(\sgn(w_i^{*\top}x)-\sgn(w_{j}^\top x)\right) \right]\\
        =& 2q_{ij}\norm{w_{j}}\E_x\left[ R(x)  |w_i^{*\top}x| \mathbb{I}_{\sgn(w_i^{*\top}x) \neq \sgn(w_{j}^\top x)} \right]\\
        =& 2q_{ij}\norm{w_{j}} \underbrace{\E_x\left[ R_1(x) \cdot |w_i^{*\top}x| \mathbb{I}_{\sgn(w_i^{*\top}x) \neq \sgn(w_{j}^\top x)} \right]}_{I_2}\\
        &+ 2q_{ij}\norm{w_{j}} \underbrace{\E_x\left[ R_2(x) \cdot |w_i^{*\top}x| \mathbb{I}_{\sgn(w_i^{*\top}x) \neq \sgn(w_{j}^\top x)} \right]}_{I_3},
    \end{align*}
    where $R_1(x)$ and $R_2(x)$ are the residual decomposition as defined in \eqref{eq:residual_decomp}. 
    
    According to Lemma \ref{lem-high-dim-small-loss-residual}, we have the residual $R(x)=R_1(x)+R_2(x)$ satisfies $|R_1(x)| \leq \norm{x}r\alpha$, $R_2(x)\geq 0$ for all $x$.

    For the first term $I_2$, recall that $R_1(x)=\sum_{i=1}^rv_i^\top x\sgn(w_i^{*\top} x)$ and $\norm{v_i}\leq \alpha$. We have
    \begin{align*}
        \E_x\left[ R_1(x) |w_i^{*\top}x| \mathbb{I}_{\sgn(w_i^{*\top}x) \neq \sgn(w_{j}^\top x)} \right]
        =& \sum_{i=1}^r \E_x\left[ v_i^\top x\sgn(w_i^{*\top} x) |w_i^{*\top}x| \mathbb{I}_{\sgn(w_i^{*\top}x) \neq \sgn(w_{j}^\top x)} \right]\\
        \geq& -\sum_{i=1}^r \norm{v_i}\norm{w_i^*}\E_x[|\bar{v}_i^\top x||\bar{w}_i^{*\top}x|\mathbb{I}_{\sgn(w_i^{*\top}x) \neq \sgn(w_{j}^\top x)}]\\
        \geq& -\sum_{i=1}^r \alpha\norm{w_i^*}\delta_j\E_{\tilde{x}}[\norm{\tilde{x}}^2\mathbb{I}_{\sgn(w_i^{*\top}\tilde{x}) \neq \sgn(w_{j}^\top \tilde{x})}]\\
        =& -c_1rw_{max}\alpha\delta_{j}^2,
    \end{align*}
    where in the second to last line $\tilde{x}$ represents a 3-dimensional Gaussian since the expression only depends on three vectors $v_i,w_j,w_i^*$, and we use Lemma \ref{lem-calculation-1} in the last line and $c_1$ is a constant.
    
    For the second term $I_3$, recall that $R_2(x) \geq 0$. Hence,
    \begin{align*}
        \E_x\left[ R_2(x) |w_i^{*\top}x| \mathbb{I}_{\sgn(w_i^{*\top}x) \neq \sgn(w_{j}^\top x)} \right]
        \geq 0.
    \end{align*}
    
    Therefore, we have
    \begin{align*}
        \E_x\left[ R(x) I_1(x) \right]
        \geq& - \sum_{i=1}^r \sum_{j\in  T_{i}(\delta_{max})} 2c_1q_{ij}\norm{w_{j}}rw_{max}\alpha\delta_{j}^2\\ 
        \geq& - 2c_1rw_{max}\alpha\delta_{max}^2 \sum_{i=1}^r \sum_{j\in  T_{i}(\delta_{max})} q_{ij}\norm{w_{j}}
        = - 2c_1 r^2 w_{max}\alpha \delta_{max}^2,
    \end{align*}
    where in the last line is because $\sum_{j\in T_{i}(\delta_{max})} q_{ij}\norm{w_{ij}} =1 $ for all $i\in[r]$.    
    
    Thus, with \eqref{eq-repara-high-dim-small-loss-descent-dir-1} we have
    \begin{align*}
        \sum_{i=1}^r\sum_{j\in T_i} \langle\nabla_{w_{j}}L(W),(I+\bar{w}_{j}\bar{w}_{j})^{-1}(w_{j}-q_{ij}w_i^*)\rangle
        \geq \norm{R}^2 -O(r^2 w_{max}\alpha \delta_{max}^2).
    \end{align*}
\end{proof}

%%%%%%%%%%%%%%%%%%%%%%%%%%%%%%%%%%%%%%%%%%%%%%%%%%%%%%%%%%%%%%%%%%%%%%%%
% Technical Lemmas
%%%%%%%%%%%%%%%%%%%%%%%%%%%%%%%%%%%%%%%%%%%%%%%%%%%%%%%%%%%%%%%%%%%%%%%%
\subsection{Technical Lemmas}

\begin{lemma}\label{lem-nonlinear-region-size}
     Consider $\alpha,\beta\in\mathbb{R}^d$ with $\norm{\alpha}=\norm{\beta}=1$ and $x\sim N(0,I)$. Denote $\angle(\alpha,\beta)=\phi$, then if $\sin\phi> 0$, we have
     \begin{align*}
        \sqrt{\frac{2}{\pi}}\delta_1e^{-\frac{\delta_1^2}{2}} \leq &\E_x[\mathbb{I}_{|\alpha^\top x|\leq \delta_1}]
        \leq \sqrt{\frac{2}{\pi}}\delta_1,\\
        \frac{\delta_1(\delta_2+\delta_1\cos\phi)}{\pi\sin\phi}e^{-\frac{\delta_1^2+\delta_2^2+2\delta_1\delta_2\cos\phi}{2\sin^2\phi}}
        \leq &\E_x[\mathbb{I}_{A}]
        \leq \frac{\delta_1(\delta_2+\delta_1\cos\phi)}{\pi\sin\phi},
     \end{align*}
     where $A=\{x||\alpha^\top x|\leq \delta_1, \text{there exists $y$ such that } |\alpha^\top y|\leq \delta_1, |\beta^\top y|\leq \delta_2, (I-\alpha\alpha^\top)x=(I-\alpha\alpha^\top)y\}$
\end{lemma}
\begin{proof}
    WLOG, assume $\alpha=(1,0,\ldots,0)^\top$ and $\beta=(\cos\phi,\sin\phi,0,\ldots,0)^\top$ with $\phi\in(0,\frac{\pi}{2}]$. We first prove the first inequality. We have
    \begin{align*}
        \E_x[\mathbb{I}_{|\alpha^\top x|\leq \delta_1}]
        =& \E_x[\mathbb{I}_{|x_1|\leq \delta_1}]
        = \frac{1}{\sqrt{2\pi}}\int_{|x_1|\leq \delta_1} e^{-\frac{x_1^2}{2}}dx_1
        = \sqrt{\frac{2}{\pi}}\int_{0\leq x_1 \leq \delta_1} e^{-\frac{x_1^2}{2}}dx_1.
    \end{align*}
    Therefore, we have the upper bound
    \begin{align*}
        \E_x[\mathbb{I}_{|\alpha^\top x|\leq \delta_1}]
        \leq& \sqrt{\frac{2}{\pi}}\delta_1,
    \end{align*}
    and the lower bound
    \begin{align*}
        \E_x[\mathbb{I}_{|\alpha^\top x|\leq \delta_1}]
        \geq& \sqrt{\frac{2}{\pi}}\delta_1e^{-\frac{\delta_1^2}{2}}.
    \end{align*}
    
    Now, we prove the second inequality. We have the set $A$ is
    \begin{align*}
        A=&\{x||x_1|\leq \delta_1, \text{there exists $y$ such that } |y_1|\leq \delta_1, |y_1\cos\phi+y_2\sin\phi|\leq \delta_2, x_2=y_2,\ldots,x_d=y_d\}\\
        =&\{x||x_1|\leq \delta_1, \text{there exists $y_1$ such that } |y_1|\leq \delta_1, |y_1\cos\phi+x_2\sin\phi|\leq \delta_2\}.
    \end{align*}
    When $\cos\phi=0$, we have $A=\{x||x_1|\leq\delta_1,|x_2\sin\phi|\leq\delta_2\}$. Thus, 
    \begin{align*}
        \E_x[\mathbb{I}_A]=\frac{1}{4\pi}\int_{|x_1|\leq\delta_1}e^{-\frac{x_1^2}{2}}dx_1 \int_{|x_2|\leq\delta_2}e^{-\frac{x_2^2}{2}}dx_2.
    \end{align*}
    We could have the upper bound
    \begin{align*}
        \E_x[\mathbb{I}_A]\leq \frac{1}{4\pi}\cdot 2\delta_1\cdot 2\delta_2
        =\frac{\delta_1\delta_2}{\pi},
    \end{align*}
    and lower bound
    \begin{align*}
        \E_x[\mathbb{I}_A]\geq \frac{1}{4\pi} \cdot 2\delta_1 e^{-\frac{\delta_1^2}{2}}\cdot 2\delta_2 e^{-\frac{\delta_2^2}{2}} =\frac{\delta_1\delta_2}{\pi}e^{-\frac{\delta_1^2+\delta_2^2}{2}}.
    \end{align*}
    
    When $\cos\phi\neq 0$, we have
    \begin{align*}
        A=&\{x||x_1|\leq \delta_1, \text{there exists $y_1$ such that } |y_1|\leq \delta_1, |y_1\cos\phi+x_2\sin\phi|\leq \delta_2\}\\
        =&\{x||x_1|\leq \delta_1, \frac{\delta_2-x_2\sin\phi}{\cos\phi}\geq -\delta_1, \frac{-\delta_2-x_2\sin\phi}{\cos\phi}\leq \delta_1 \}\\
        =&\{x||x_1|\leq \delta_1, |x_2|\leq \frac{\delta_2+\delta_1\cos\phi}{\sin\phi}\}.
    \end{align*}
    Thus, 
    \begin{align*}
        \E_x[\mathbb{I}_A]=\frac{1}{4\pi}\int_{|x_1|\leq\delta_1}e^{-\frac{x_1^2}{2}}dx_1 \int_{|x_2|\leq \frac{\delta_2+\delta_1\cos\phi}{\sin\phi}}e^{-\frac{x_2^2}{2}}dx_2.
    \end{align*}
    We could have the upper bound
    \begin{align*}
        \E_x[\mathbb{I}_A]\leq \frac{1}{4\pi}\cdot 2\delta_1\cdot \frac{2(\delta_2+\delta_1\cos\phi)}{\sin\phi}
        =\frac{\delta_1(\delta_2+\delta_1\cos\phi)}{\pi\sin\phi},
    \end{align*}
    and lower bound
    \begin{align*}
        \E_x[\mathbb{I}_A]\geq \frac{1}{4\pi} \cdot 2\delta_1 e^{-\frac{\delta_1^2}{2}}\cdot \frac{2(\delta_2+\delta_1\cos\phi)}{\sin\phi} e^{-\frac{(\delta_2+\delta_1\cos\phi)^2}{2\sin^2\phi}} =\frac{\delta_1(\delta_2+\delta_1\cos\phi)}{\pi\sin\phi}e^{-\frac{\delta_1^2+\delta_2^2+2\delta_1\delta_2\cos\phi}{2\sin^2\phi}}.
    \end{align*}
\end{proof}

\begin{lemma}\label{lem-loss-delta-calculation-1}
    Suppose $0\leq\tau,\phi\leq c$ for a sufficiently small constant. Then, we have
    \begin{align*}
        h(\tau,\phi)=4 \OwenT(\tau, \cot\phi) -1 +\frac{2\phi}{\pi} + \frac{\tau}{\sqrt{2\pi}}\erf\left(\frac{\tau\cot\phi} {\sqrt{
        2}}\right)\leq 0,
    \end{align*}
    where $\OwenT(x,a) = \frac{1}{2\pi}\int_0^a e^{-x^2(1+t^2)/2}\frac{1}{1+t^2}\dif t$ and $\erf(x)=\frac{2}{\sqrt{\pi}}\int_0^x e^{-t^2}\dif t$.
\end{lemma}

\begin{proof}
    We have
    \begin{align*}
        \frac{\dif h(\tau,\phi)}{\dif \phi} 
        = \frac{1}{\pi}\left(2-2e^{-\frac{\tau^2}{2\sin^2\phi}}-\frac{\tau^2}{\sin^2\phi}e^{-\frac{\tau^2\cot^2\phi}{2}}\right)
        \triangleq \frac{1}{\pi}h_1(\tau,\phi).
    \end{align*}
    Further,
    \begin{align*}
        \frac{\dif h_1(\tau,\phi)}{\dif\tau}
        =& \frac{\tau}{\sin^2\phi}\left(2e^{-\frac{\tau^2}{2\sin^2\phi}} -2e^{-\frac{\tau^2\cot^2\phi}{2}} + e^{-\frac{\tau^2\cot^2\phi}{2}} \tau^2\cot^2\phi \right)\\
        =& \frac{\tau}{\sin^2\phi}e^{-\frac{\tau^2\cot^2\phi}{2}}\left(2e^{-\frac{\tau^2}{2}} -2 + \tau^2\cot^2\phi \right)\\
        \geq& \frac{\tau}{\sin^2\phi}e^{-\frac{\tau^2\cot^2\phi}{2}}\left(2(1-\frac{\tau^2}{2}) -2 + \tau^2\cot^2\phi \right)\\
        =& \frac{\tau}{\sin^2\phi}e^{-\frac{\tau^2\cot^2\phi}{2}}\tau^2(\cot^2\phi -1)
        \geq 0,
    \end{align*}
    where we use $e^{-x}\geq 1-x$ when $x\geq 0$ in the first inequality and $\phi\leq\frac{\pi}{4}$ in the last line.
    
    Hence, when $\phi\leq \frac{\pi}{4}$, $h_1(\tau,\phi)\geq h_1(0,\phi)=0$. This implies that if $h(\tau,\phi_0)\leq 0$ for some $\phi_0\leq\frac{\pi}{4}$, then we have $h(\tau,\phi)\leq h(\tau,\phi_0)\leq 0$.
    
    We are going to show that when $\phi_0=\arccot 3\leq\frac{\pi}{4}$ and $\tau\leq c$, we have $h(\tau,\phi_0)\leq 0$.
    \begin{align*}
        \frac{\dif h(\tau,\phi_0)}{\dif \tau}
        =& e^{-\frac{\tau^2\cot^2\phi_0}{2}}\frac{\tau\cot\phi_0}{\pi}
        + \frac{\erf(\frac{\tau\cot\phi_0}{\sqrt{2}})}{\sqrt{2\pi}}
        -\sqrt{\frac{2}{\pi}}e^{-\frac{\tau^2}{2}}\erf(\frac{\tau\cot\phi_0}{\sqrt{2}})\\
        \leq& \left(1-\frac{\tau^2\cot^2\phi_0}{2}+\frac{\tau^4\cot^4\phi_0}{8}\right)\frac{\tau\cot\phi_0}{\pi}
        + \frac{1}{\sqrt{2\pi}}\frac{2}{\sqrt{\pi}}\frac{\tau\cot\phi_0}{\sqrt{2}}\\
        &-\sqrt{\frac{2}{\pi}}\left(1-\frac{\tau^2}{2}\right)\left(\frac{2}{\sqrt{\pi}}\frac{\tau\cot\phi_0}{\sqrt{2}}-\frac{2}{3\sqrt{\pi}}\frac{\tau^3\cot^3\phi_0}{2\sqrt{2}}\right)\\
        =& \left(\frac{\cot\phi_0}{\pi}-\frac{\cot^3\phi_0}{6\pi}\right)\tau^3 + \left(\frac{\cot^5\phi_0}{8\pi}-\frac{\cot^3\phi}{6\pi}\right)\tau^5
        =-\frac{3}{2\pi}\tau^3+\frac{207}{8\pi}\tau^5\leq 0,
    \end{align*}
    where we use $e^{-x}\leq 1-x+\frac{x^2}{2}$, $e^{-x}\geq 1-x$ and $\erf(x)\geq\frac{2x}{\sqrt{\pi}}-\frac{2x^3}{3\sqrt{\pi}}$ for $x\geq 0$ in the first inequality, and $\tau\leq c$, $\phi_0=\arccot 3$ in the last line.
    
    Therefore, we have $h(\tau,\phi_0)\leq h(0,\phi_0)=0$, when $\tau\leq c$. Together with $h_1(\tau,\phi)\geq 0$ when $\phi\leq\frac{\pi}{4}$,we know $h(\tau,\phi)\leq h(\tau,\phi_0)\leq 0$, when $\tau\leq c$ and $\phi\leq\frac{\pi}{4}$.
\end{proof}

\begin{lemma}\label{lem-calculation-1}
Consider $\alpha,\beta\in\mathbb{R}^3$ with $\angle(\alpha,\beta)=\phi$ and $\alpha^\top\beta\ge 0$. We have
\begin{align*}
    \E_x[\norm{x}^2 \mathbb{I}_{\sgn(\alpha^\top x)\neq \sgn(\beta^\top x)}]=O(\phi).
\end{align*}
\end{lemma}
\begin{proof}
    WLOG, assume $\alpha=(1,0,0)^\top$ and $\beta=(\cos\phi,\sin\phi,0)^\top$. We have
    \begin{align*}
        &\E_x[\norm{x}^2 \mathbb{I}_{\sgn(\alpha^\top x)\neq \sgn(\beta^\top x)}]\\
        =&\E_{x_1,x_2,x_3}[(x_1^2+x_2^2+x_3^2) \mathbb{I}_{\sgn(x_1)\neq \sgn(x_1\cos\phi+x_2\sin\phi)}]\\
        =&\E_{x_1,x_2}[(x_1^2+x_2^2) \mathbb{I}_{\sgn(x_1)\neq \sgn(x_1\cos\phi+x_2\sin\phi)}]
        + \E_{x_1,x_2}[\mathbb{I}_{\sgn(x_1)\neq \sgn(x_1\cos\phi+x_2\sin\phi)}]\\
        =& \frac{1}{2\pi}\int_0^\infty r^3 e^{-r^2/2}\dif r \int_0^{2\pi} \mathbb{I}_{\sgn(\cos\theta)\neq \sgn(\cos\phi\cos\theta+\sin\phi\sin\theta)} \dif \theta\\
        & +\frac{1}{2\pi}\int_0^\infty r e^{-r^2/2}\dif r \int_0^{2\pi} \mathbb{I}_{\sgn(\cos\theta)\neq \sgn(\cos\phi\cos\theta+\sin\phi\sin\theta)} \dif \theta\\
        =& O(\phi).
    \end{align*}
\end{proof}

\begin{lemma}\label{lem-sum-student-small-calculation-1}
Consider $w_j,w_k,w^*\in \mathbb{R}^d$ with $\norm{w_j}=\norm{w_k}=\norm{w^*}=1$ and $w_j^\top w^*,w_k^\top w^*\ge 0$. Denote $\phi_j=\angle(w_j,w^*)$ and $\phi_k=\angle(w_k,w^*)$. We have
\begin{align*}
    0\leq \E_{x} \left[w_j^\top x \cdot w_k^\top x \left(\sgn(w_{j}^\top x) - \sgn(w^{*\top} x)\right)  \left(\sgn(w_{k}^\top x) - \sgn(w^{*\top} x)\right)\right]
    = O(\phi_j\phi_k\min\{\phi_j,\phi_k\}).
\end{align*}
\end{lemma}
\begin{proof}
    Given that there are only three vectors $w_j,w_k,w^*$ and $x\sim N(0,I)$, it is equivalent to consider a three dimensional Gaussian $\tilde{x}\in\mathbb{R}^3$ that lies in the span of $w_j,w_k,w^*$. In the following, we slightly abuse the notation and use $x$ to denote this three dimensional Gaussian.
    
    For any $x$, it is easy to see that
    \begin{align*}
        w_j^\top x \left(\sgn(w_{j}^\top x) - \sgn(w^{*\top} x)\right)\geq 0  
        ,\quad 
        w_k^\top x   \left(\sgn(w_{k}^\top x) - \sgn(w^{*\top} x)\right)\geq 0.
    \end{align*}
    Therefore, we have the lower bound
    \begin{align*}
        \E_{x\sim N(0,I)} \left[w_j^\top x \cdot w_k^\top x \left(\sgn(w_{j}^\top x) - \sgn(w^{*\top} x)\right)  \left(\sgn(w_{k}^\top x) - \sgn(w^{*\top} x)\right)\right]
        \geq 0.
    \end{align*}
    
    For upper bound, note that when $\sgn(w_{j}^\top x) \neq \sgn(w^{*\top} x)$, we have $w_j^\top x =O(\phi_j\norm{x})$. Similarly, we have $w_k^\top x =O(\phi_k\norm{x})$ when $\sgn(w_{k}^\top x) \neq \sgn(w^{*\top} x)$. Thus, we have
    \begin{align*}
        &\E_{x\sim N(0,I)} \left[w_j^\top x \cdot w_k^\top x \left(\sgn(w_{j}^\top x) - \sgn(w^{*\top} x)\right)  \left(\sgn(w_{k}^\top x) - \sgn(w^{*\top} x)\right)\right]\\
        \leq& O(\phi_j\phi_k) \E_x[\norm{x}^2\mathbb{I}_{\sgn(w_{j}^\top x) \neq \sgn(w^{*\top} x), \sgn(w_{k}^\top x) \neq \sgn(w^{*\top} x)}]\\
        \leq& O(\phi_j\phi_k) \min\{\E_x[\norm{x}^2\mathbb{I}_{\sgn(w_{j}^\top x) \neq \sgn(w^{*\top} x)}], \E_x[\norm{x}^2\mathbb{I}_{\sgn(w_{k}^\top x) \neq \sgn(w^{*\top} x)}]\}\\
        =& O(\phi_j\phi_k\min\{\phi_j,\phi_k\}),
    \end{align*}
    where we use Lemma \ref{lem-calculation-1} in the last line.
\end{proof}

%%%%%%%%%%%%%%%%%%%%%%%%%%%%%%%%%%%%%%%%%%%%%%%%%%%%%%%%%%%%%%%%%%%%%%%
% Hermite Poly
%%%%%%%%%%%%%%%%%%%%%%%%%%%%%%%%%%%%%%%%%%%%%%%%%%%%%%%%%%%%%%%%%%%%%%%
\subsection{Some Properties of Hermite Polynomials}\label{sec-hermite}
In this section, we give several properties of Hermite Polynomials that are useful in our analysis. Let $H_k$ be the probabilists’ Hermite polynomial where
\begin{align*}
    H_k(x)=(-1)^k e^{x^2/2}\frac{\dif^k}{\dif x^k}(e^{-x^2/2})
\end{align*}
and $h_k=\frac{1}{\sqrt{k!}}H_k$ be the normalized Hermite polynomials.

The following lemma gives the Hermite coefficients for absolute value function.

\begin{lemma}\label{lem-hermite-coefficient}
Let $\sigma(x)=|x|$. Then, $\sigma(x)=\sum_{k=0}^\infty \hat{\sigma}_k h_k(x)$, where $\{h_k\}_{k=0}^\infty$ are Hermite polynomials and 
\begin{align*}
\hat{\sigma}_k
=\left\{\begin{array}{ll}
0 & \text{, $k$ is odd}\\
\sqrt{\frac{2}{\pi}} & \text{, $k=0$}\\
\sqrt{\frac{1}{\pi}} & \text{, $k=2$}\\
(-1)^{\frac{k}{2}-1}\sqrt{\frac{2}{\pi\cdot k!}}(k-3)!! & \text{, $k$ is even and $n\geq 4$}
\end{array}\right.
\end{align*}
are Hermite coefficients of $|x|$. Here, $\hat{\sigma}_k^2=\Theta(k^{-3/2})$.
\end{lemma}
\begin{proof}
    We calculate $\hat{\sigma}_k$ for the following cases.
    \noindent\textbf{Case 1:} $k$ is odd.
    
    Since $h_k(x)$ is odd when $k$ is odd, we have
    \begin{align*}
        \hat{\sigma}_k 
        = \langle \sigma(x), h_k(x)\rangle
        = \frac{1}{\sqrt{2\pi}}\int_{-\infty}^\infty |x|h_k(x) e^{-x^2/2}dx
        = 0.
    \end{align*}
    
    \noindent\textbf{Case 2:} $k=0$.
    \begin{align*}
        \hat{\sigma}_0
        = \langle \sigma(x), h_0(x)\rangle
        = \frac{1}{\sqrt{2\pi}}\int_{-\infty}^\infty |x| e^{-x^2/2}dx
        =\sqrt{\frac{2}{\pi}}.
    \end{align*}
    
    \noindent\textbf{Case 3:} $k=2$.
    \begin{align*}
        \hat{\sigma}_2
        = \langle \sigma(x), h_2(x)\rangle
        = \frac{1}{\sqrt{2\pi}}\int_{-\infty}^\infty |x| \cdot \frac{x^2 -1}{\sqrt{2}}\cdot e^{-x^2/2}dx
        = \frac{1}{\sqrt{\pi}}
    \end{align*}
    
    \noindent\textbf{Case 4:} $k$ is even and $k\geq 4$.
    
    When $k$ is even, we know $h_k$ is even. Thus,
    \begin{align*}
        \hat{\sigma}_k
        &= \langle \sigma(x), h_k(x)\rangle
        = \frac{1}{\sqrt{2\pi}}\int_{-\infty}^\infty |x| \cdot \frac{H_k}{\sqrt{k!}}\cdot e^{-x^2/2}\dif x\\
        &= \frac{1}{\sqrt{2\pi\cdot k!}}\int_{-\infty}^\infty |x| \cdot (-1)^k e^{x^2/2}\frac{\dif^k}{\dif x^k}(e^{-x^2/2})\cdot e^{-x^2/2}\dif x\\
        &= \frac{(-1)^k}{\sqrt{2\pi\cdot k!}} \cdot 2\int_0^\infty x \cdot \frac{\dif^k}{\dif x^k}(e^{-x^2/2}) \dif x\\
        &= \frac{2(-1)^k}{\sqrt{2\pi\cdot k!}} \cdot \left(\left.x \cdot \frac{\dif^{k-1}}{\dif x^{k-1}}(e^{-x^2/2})\right|_0^\infty - \int_0^\infty \frac{\dif^{k-1}}{\dif x^{k-1}}(e^{-x^2/2}) \dif x\right)\\
        &= \frac{2(-1)^k}{\sqrt{2\pi\cdot k!}} \cdot \left( - \left. \frac{\dif^{k-2}}{\dif x^{k-2}}(e^{-x^2/2}) \right|_0^\infty\right)\\
        &= \sqrt{\frac{2}{\pi\cdot k!}} \cdot H_{k-2}(0)\\
        &= \sqrt{\frac{2}{\pi\cdot k!}} (-1)^{\frac{k}{2}-1} (k-3)!!,
    \end{align*}
    where we use $H_k(x) = (-1)^k e^{x^2/2}\frac{\dif^k}{\dif x^k}(e^{-x^2/2})$ and $H_k(0) = (-1)^{k/2}(k-1)!!$ when $k$ is even \citep{abramowitz1948handbook}.
\end{proof}

The following is another helpful property of Hermite polynomial.
\begin{claim}[\citep{o2014analysis}, Section 11.2]\label{lem-hermite-product}
Let $(x,y)$ be $\rho$-correlated standard normal variables (that is,
both $x,y$ have marginal distribution $N(0,1)$ and $\E[xy] = \rho$). Then,
$\E[h_m(x)h_n(y)] = \rho^n \delta_{mn}$.
\end{claim}

%%%%%%%%%%%%%%%%%%%%%%%%%%%%%%%%%%%%%%%%%%%%%%%%%%%%%%%%%%%%%%%%%%%%%
% local convergence
%%%%%%%%%%%%%%%%%%%%%%%%%%%%%%%%%%%%%%%%%%%%%%%%%%%%%%%%%%%%%%%%%%%%%
\section{Proof of Local Convergence (Theorem \ref{lem-repara-high-dim-converge})}
In this section, we first show that the loss function is ``smooth" (Lemma \ref{lem-high-dim-smooth}) and Lipshcitz (Lemma \ref{lem-grad-upper-bound}). At the end, we give the proof of local convergence theorem (Theorem \ref{lem-repara-high-dim-converge}).

%%%%%%%%%%%%%%%%%%%%%%%%%%%%%%%%%%%%%%%%%%%%%%%%%%%%%%%%%%%%%%%%%%%%%
% smoothness
%%%%%%%%%%%%%%%%%%%%%%%%%%%%%%%%%%%%%%%%%%%%%%%%%%%%%%%%%%%%%%%%%%%%%
\subsection{Proof of Theorem \ref{lem-high-dim-smooth}}\label{sec-pf-smooth}
As the first step to prove local convergence theorem, we show that the loss function satisfies ``smoothness"-type conditions as below. The proof is involved and relies on algebraic computations to carefully bound each term. 

\lemHighDimSmooth*

\begin{proof}
    Denote the residual as
    \begin{align*}
        R_W(x) = \sum_{j=1}^{m}\norm{w_{j}}|w_{j}^\top x| - \sum_{i=1}^r |w_i^{*\top} x|.
    \end{align*}

    Then, we have
    \begin{align*}
        &L(W+U) - L(W) - \langle \nabla_W L(W), U\rangle\\
        =& \frac{1}{2}\E_x\left[R_{W+U}^2(x)\right]
        - \frac{1}{2}\E_x\left[R_W^2(x)\right]
        - \E_x\left[R_W(x) \sum_{j=1}^{m} \norm{w_{j}}u_{j}^\top(I+\bar{w}_{j}\bar{w}_{j}^\top) x\sgn(w_{j}^\top x)\right]\\
        =& \frac{1}{2} \underbrace{ \E_x\left[(R_{W+U}(x)-R_W(x))^2\right]}_{I_1}\\
        &+ \underbrace{ \E_x\left[R_W(x) \left( R_{W+U}(x)-R_W(x) 
        - \sum_{j=1}^{m} \norm{w_{j}}u_{j}^\top(I+\bar{w}_{j}\bar{w}_{j}^\top) x\sgn(w_{j}^\top x)\right)\right] }_{I_2}.
    \end{align*}
    For term $I_1$, let $\gamma_j=\norm{w_{j}+u_{j}}(w_{j}+u_{j}) - \norm{w_{j}}w_{j}$. Note that
    \begin{align*}
        \norm{\gamma_j}^2
        \leq& 2\left(\norm{\norm{w_{j}+u_{j}}w_{j} - \norm{w_{j}}w_{j}}^2 + \norm{\norm{w_{j}+u_{j}}u_{j}}^2\right)\\
        \leq& 2\left(\norm{w_{j}}^2 \norm{u_{j}}^2 + \norm{w_{j}+u_{j}}^2\norm{u_{j}}^2\right)\\
        \leq& 2\left(3\norm{w_{j}}^2+2\norm{u_{j}}^2\right) \norm{u_{j}}^2.
    \end{align*}
    
    Hence, we have
    \begin{align*}
        I_1 =& \E_x\left[(R_{W+U}(x)-R_W(x))^2\right]
        = \E_x\left[\left(\sum_{j=1}^{m}\left(|\norm{w_{j}+u_{j}}(w_{j}+u_{j})^\top x| - |\norm{w_{j}}w_{j}^\top x|\right)\right)^2\right]\\
        \leq& \E_x\left[\left(\sum_{j=1}^{m}|\gamma_j^\top x|\right)^2\right]\\
        \leq& c_0 \left(\sum_{j=1}^{m}\norm{\gamma_j}\right)^2\\
        \leq& c_0 \left(\sum_{j=1}^{m}\left(6\norm{w_{j}}^2+4\norm{u_{j}}^2\right)^{1/2} \norm{u_{j}}\right)^2\\
        \leq& c_0 \sum_{j=1}^{m}\left(6\norm{w_{j}}^2+4\norm{u_{j}}^2\right) \cdot \sum_{j=1}^{m}\norm{u_{j}}^2\\
        =& O(rw_{max})\sum_{j=1}^{m}\norm{u_{j}}^2
        + O(1) \left(\sum_{j=1}^{m}\norm{u_{j}}^2\right)^2,
    \end{align*}
    where we use Lemma \ref{lem-calculation-3} in the third line, and the last line is because of Lemma \ref{lem-sum-norm-bound}.
    
    For term $I_2$, we have
    \begin{align*}
        I_2=& \E\left[R_W(x) \left( R_{W+U}(x)-R_W(x) 
        - \sum_{j=1}^{m} \norm{w_{j}}u_{j}^\top(I+\bar{w}_{j}\bar{w}_{j}^\top) x\sgn(w_{j}^\top x)\right)\right]\\
        =& \E\left[R_W(x) \sum_{j=1}^{m}
        \underbrace{\left(|\norm{w_{j}+u_{j}}(w_{j}+u_{j})^\top x| - |\norm{w_{j}}w_{j}^\top x| - \norm{w_{j}}u_{j}^\top(I+\bar{w}_{j}\bar{w}_{j}^\top) x\sgn(w_{j}^\top x)\right)}_{I_{j}(x)}\right]\\
        \leq& \sum_{j=1}^{m} \norm{R_W}\norm{I_{j}}.
    \end{align*}
    Denote $\alpha = \norm{w_{j}+u_{j}}(w_{j}+u_{j})$ and $\beta = \norm{w_{j}}\left(w_{j}+(I+\bar{w}_{j}\bar{w}_{j}^\top)u_{j}\right)$. In the following, we drop the subscript $j$ for simplicity. We have
    \begin{align*}
        \norm{I_{j}}^2 
        =& \E_x\left[\left(|\alpha^\top x| - \beta^\top x\sgn(w^\top x)\right)^2\right]\\
        \leq& 2\E_x\left[\left(|\alpha^\top x| - |\beta^\top x|\right)^2 
        + \left(\beta^\top x(\sgn(\beta^\top x) - \sgn(w^\top x)\right)^2\right]\\
        \leq& 2\E_x\left[\left((\alpha-\beta)^\top x\right)^2\right] 
        + 2\E_x\left[\left(\beta^\top x\right)^2\mathbb{I}_{\sgn(\beta^\top x) \neq \sgn(w^\top x)}\right]\\
        =& 2 \norm{\alpha-\beta}^2
        + 2\underbrace{\E_x\left[\left(\beta^\top x\right)^2\mathbb{I}_{\sgn(\beta^\top x) \neq \sgn(w^\top x)}\right]}_{I_3}.
    \end{align*}

    For first term $\norm{\alpha-\beta}$, let $g(w)=\norm{w}w$. By Lemma \ref{lem-smooth-g}, we know
    \begin{align*}
        g(w+u)= g(w) +\langle\nabla_w g(w), u\rangle + \Delta,
    \end{align*}
    where $\norm{\Delta}\leq c_1\norm{u}^2$ for some constant $c_1$. Note that $g(w+u)=\alpha$ and $g(w) +\langle\nabla_w g(w), u\rangle=\beta$. Hence, we have $\norm{\alpha-\beta}\leq c_1\norm{u}^2$.
    
    For second term $I_3$, by Lemma \ref{lem-expectation}, we have
    \begin{align*}
        I_3 = \E_x\left[\left(\beta^\top x\right)^2\mathbb{I}_{\sgn(\beta^\top x) \neq \sgn(w^\top x)}\right]
        = \frac{\norm{\beta}^2}{\pi}(\phi-\sin\phi\cos\phi)
        \leq \frac{\norm{\beta}^2}{\pi}\phi^3,
    \end{align*}
    where $\phi=\arccos(\bar{\beta}^\top\bar{w})$ and we use $\phi-\sin\phi\cos\phi\leq\phi^3$ for $\phi\geq 0$.
    
    Denote $\theta = \arccos(\bar{w}^\top \bar{u})$. We have
    \begin{align*}
        \cos\phi =& \frac{\beta^\top w}{\norm{\beta}\norm{w}}
        = \frac{\norm{w}^3 + 2\norm{w}w^\top u}{\norm{\beta}\norm{w}}
        = \frac{\norm{w} + 2\norm{u}\cos\theta}{\sqrt{\norm{w}^2+\norm{u}^2+3\norm{u}^2\cos^2\theta+4\norm{w}\norm{u}\cos\theta}}\\
        =& \frac{\norm{w} + 2\norm{u}\cos\theta}{\sqrt{(\norm{w} + 2\norm{u}\cos\theta)^2 + \norm{u}^2\sin^2\theta}}.
    \end{align*}
    Note that $\cos\phi\leq 1-\frac{\phi^2}{5}$ for $0\leq\phi\leq \pi$. We have
    \begin{equation}\label{eq-high-dim-smooth-1}
    \begin{aligned}
        I_3 \leq \frac{\norm{\beta}^2}{\pi}5^{3/2}\left(1-\frac{\norm{w} + 2\norm{u}\cos\theta}{\sqrt{(\norm{w} + 2\norm{u}\cos\theta)^2 + \norm{u}^2\sin^2\theta}}\right)^{3/2}.
    \end{aligned}    
    \end{equation}

    Note that when $\norm{u}=0$, it is easy to verify that $I_3=0$. In the following, we assume $\norm{u}>0$.
    \paragraph{Case 1:} $\norm{w} \geq 4\norm{u}$.
    
    In this case, we have $\norm{w} + 2\norm{u}\cos\theta\geq \frac{\norm{w}}{2}\geq 2\norm{u}>0$. Using $1-\frac{1}{\sqrt{1+x^2}}\leq \frac{x^2}{2}$ with \eqref{eq-high-dim-smooth-1}, we have
    \begin{align*}
        I_3\leq& \frac{5^{3/2}\norm{\beta}^2}{\pi}\left(\frac{\norm{u}^2\sin^2\theta}{2(\norm{w} + 2\norm{u}\cos\theta)^2}\right)^{3/2}\\
        =& \frac{5^{3/2}}{2^{3/2}\pi}\frac{\norm{w}^2\left((\norm{w} + 2\norm{u}\cos\theta)^2 + \norm{u}^2\sin^2\theta\right)}{|\norm{w} + 2\norm{u}\cos\theta|^3}\norm{u}^3|\sin\theta|^3\\
        =& \frac{5^{3/2}}{2^{3/2}\pi}\frac{\norm{w}}{|\norm{w} + 2\norm{u}\cos\theta|}\frac{(\norm{w} + 2\norm{u}\cos\theta)^2 + \norm{u}^2\sin^2\theta}{|\norm{w} + 2\norm{u}\cos\theta|^2}\norm{w}\norm{u}^3|\sin\theta|^3\\
        \leq& \frac{5^{3/2}}{2^{3/2}\pi}2(1+\frac{\sin^2\theta}{4})
        \norm{w}\norm{u}^3|\sin\theta|^3\\
        \leq& 4 \norm{w}\norm{u}^3    
        \end{align*}
    
    \paragraph{Case 2:} $\norm{w} < 4\norm{u}$.
    
    In this case, we have $\norm{\beta}^2=\norm{w}^2\left((\norm{w} + 2\norm{u}\cos\theta)^2 + \norm{u}^2\sin^2\theta\right)\leq 592\norm{u}^4$. With \eqref{eq-high-dim-smooth-1}, we have $I_3 \leq \frac{592\norm{u}^4}{\pi}5^{3/2}2^{3/2}$.
    
    Combine above two cases, we have $I_3\leq c_2\max\{\norm{w}\norm{u}^3,\norm{u}^4\}$, where $c_2$ is a constant. Therefore, we have $\norm{I_{j}}\leq c_3\max\{\norm{w_{j}}^{1/2}\norm{u_{j}}^{3/2},\norm{u_{j}}^2\}$, where $c_3$ is a constant. This leads to
    \begin{align*}
        I_2 \leq& \sum_{j=1}^{m} \norm{R_W}\norm{I_{ij}}
        \leq \norm{R_W}\sum_{j=1}^{m} c_3\max\{\norm{w_{j}}^{1/2}\norm{u_{j}}^{3/2},\norm{u_{j}}^2\}\\
        \leq& \norm{R_W}\sum_{j=1}^{m} c_3\norm{w_{j}}^{1/2}\norm{u_{j}}^{3/2}
        + \norm{R_W}\sum_{j=1}^{m} c_3\norm{u_{j}}^2\\
        \leq& c_3 \norm{R_W}\left(\sum_{j=1}^{r}\norm{w_{j}}^2\right)^{1/4}\left(\sum_{j=1}^{m}\norm{u_{ij}}^2\right)^{3/4}+c_3 \norm{R_W}\sum_{j=1}^{m}\norm{u_{ij}}^2,
    \end{align*}
    where we use Holder's inequality in the last line.
    
    Denote $\norm{U}_F^2=\sum_{j=1}^{m}\norm{u_{j}}^2$. With Lemma \ref{lem-sum-norm-bound}, we have
    \begin{align*}
        &L(W+U) - L(W) - \langle \nabla_W L(W), U\rangle\\
        \leq& I_1 + I_2\\
        \leq&  O(rw_{max})\norm{U}_F^2
        + O(1) \norm{U}_F^4 + O(r^{1/4}w_{max}^{1/4})L^{1/2}(W)\norm{U}_F^{3/2}+O(1)L^{1/2}(W)\norm{U}_F^2\\
        =& O(r^{1/4}w_{max}^{1/4})L_2^{1/2}(W)\norm{U}_F^{3/2} + O(rw_{max})\norm{U}_F^2
        + O(1)\norm{U}_F^4.
    \end{align*}
\end{proof}

%%%%%%%%%%%%%%%%%%%%%%%%%%%%%%%%%%%%%%%%%%%%%%%%%%%%%%%%%%%%%%%%%%%%%
% gradient upper bound
%%%%%%%%%%%%%%%%%%%%%%%%%%%%%%%%%%%%%%%%%%%%%%%%%%%%%%%%%%%%%%%%%%%%%

\subsection{Proof of Lemma \ref{lem-grad-upper-bound}}
The second step is to show the loss function is Lipschitz, i.e., the gradient norm is upperbound. The proof follows from simple computations.

\lemGradUpperBound*
\begin{proof}
    Denote $R(x)= \sum_{i=1}^{m}\sum_{j\in T_i}\norm{w_{j}}|w_{j}^\top x| - \sum_{i=1}^r |w_i^{*\top} x|$ as the residual. For any $k\in[m]$, we have 
    \begin{align*}
    \norm{\nabla_{w_{k}}L(W)} 
    =& \norm{\E_x\left[R(x)\norm{w_{k}}(I+\bar{w}_{k}\bar{w}_{k}^\top) x\sgn(w_{k}^\top x)\right]}\\
    =& \norm{w_{k}}\norm{\E_x\left[\left(\sum_{i=1}^{r}\sum_{j\in T_i}\norm{w_{j}}|w_{j}^\top x| - \sum_{i=1}^r |w_i^{*\top} x|\right)(I+\bar{w}_{k}\bar{w}_{k}^\top) x\sgn(w_{k}^\top x)\right]}\\
    \leq& \norm{w_{k}}\left(\sum_{i=1}^{r}\sum_{j\in T_i}\norm{\E_x\left[\norm{w_{j}}|w_{j}^\top x|(I+\bar{w}_{k}\bar{w}_{k}^\top) x\sgn(w_{k}^\top x)\right]}\right.\\
    &+\left. \sum_{i=1}^r\norm{\E_x\left[ |w_i^{*\top} x| (I+\bar{w}_{k}\bar{w}_{k}^\top) x\sgn(w_{k}^\top x)\right]}\right)\\
    \leq& \norm{w_{k}}\left(\sum_{i=1}^{r}\sum_{j\in T_i}\norm{w_{j}}^2\cdot O(1) + \sum_{i=1}^r \norm{w_i^*}\cdot O(1)\right)\\
    \leq& O(rw_{max})\norm{w_{k}},
    \end{align*}
    where we use Lemma \ref{lem-calculation-2} in the second to last line, and Lemma \ref{lem-sum-norm-bound} in the last line. This leads to
    \begin{align*}
        &\norm{\nabla_W L(W)}_F^2=\sum_{j=1}^{m}\norm{\nabla_{w_{j}}L_2(W)}^{2}
        \leq O(r^2w_{max}^2)\sum_{j=1}^{m}\norm{w_{j}}^{2}
        =O(r^{3}w_{max}^3),
    \end{align*}
    where we again use Lemma \ref{lem-sum-norm-bound}.
\end{proof}

\subsection{Proof of Theorem \ref{lem-repara-high-dim-converge}}
\label{sec:proof_main_theorem}
Now we are ready to prove the local convergence result given the smoothness and Lipschitz of loss function.

\lemHighDimLocalConvergence*   

\begin{proof}
    Using Theorem \ref{lem-high-dim-smooth} with $u_{j}=-\eta\nabla_{w_{j}}L(W)$, we have
    \begin{align*}
        L(W+U)\leq& L(W) -\eta \norm{\nabla_{W}L(W)}_F^2 + O(\eta^{3/2}r^{1/4}w_{max}^{1/4})L^{1/2}(W)\norm{\nabla_{W}L(W)}_F^{3/2}\\
        &+ O(\eta^2rw_{max})\norm{\nabla_{W}L(W)}_F^{2} + O(\eta^4)\norm{\nabla_{W}L(W)}_F^4\\
        \leq& L(W) -\left(\eta - O\left(\eta^{3/2}r^{1/4}w_{max}^{1/4}\kappa^{-1/2} + \eta^2rw_{max} + \eta^4r^3w_{max}^3\right)\right) \norm{\nabla_{W}L(W)}_F^2\\
        \leq& L(W) -\frac{\eta\kappa^2}{4} L^2(W),
    \end{align*}
    where in the second line we use gradient norm upper bound (Lemma \ref{lem-grad-upper-bound}) and gradient norm lower bound (Theorem \ref{lem-repara-high-dim-grad-norm}), in the last line we use Theorem \ref{lem-repara-high-dim-grad-norm} again and $\eta\leq\eta_0=O(r^{-1}w_{max}^{-1})$.
    This implies that
    \begin{align*}
        L(W^{(t+1)})\leq L(W^{(t)}) -\frac{\eta\kappa^2}{4} L^2(W^{(t)}).
    \end{align*}
    Then, note that $0 < L(W^{(t+1)})\leq L(W^{(t)})$, we could have
    \begin{align*}
        \frac{1}{L(W^{(t)})} \leq \frac{1}{L(W^{(t+1)})} - \frac{\eta\kappa^2}{4}\frac{L(W^{(t)})}{L(W^{(t+1)})}
        \leq \frac{1}{L(W^{(t+1)})} - \frac{\eta\kappa^2}{4}. 
    \end{align*}
    This implies that
    \begin{align*}
        0<\frac{1}{L(W^{(0)})} 
        \leq \frac{1}{L(W^{(t)})} - \frac{t\eta\kappa}{4},
    \end{align*}
    which leads to
    \begin{align*}
        L(W^{(t)})\leq \min\left\{\frac{4}{t\eta\kappa^2},L(W^{(0)})\right\}=\min\left\{O\left(\frac{r w_{max}}{t\eta}\right),\epsilon_0\right\}.
    \end{align*}
    Therefore, we have $L(W^{(T)})\leq \epsilon$ for some $T=O\left(\frac{r w_{max}}{\epsilon\eta}\right)$.
\end{proof}

%%%%%%%%%%%%%%%%%%%%%%%%%%%%%%%%%%%%%%%%%%%%%%%%%%%%%%%%%%%%%%%%%%%%%
% Technical Lemmas
%%%%%%%%%%%%%%%%%%%%%%%%%%%%%%%%%%%%%%%%%%%%%%%%%%%%%%%%%%%%%%%%%%%%%
\subsection{Technical Lemmas}
\begin{lemma}\label{lem-smooth-g}
    Let $g(w)=\norm{w}w$. Then, $g$ is $(1+\sqrt{3})$-smooth.
\end{lemma}
\begin{proof}
    The gradient is
    \begin{align*}
        \nabla_w g(w) = \norm{w}(I+\bar{w}\bar{w}^\top).
    \end{align*}
    Note that when $w=0$, $\nabla_w g(w) = 0$.
    
    For any $u$, we have
    \begin{align*}
        &\norm{\nabla_w g(w+u)-\nabla_w g(w)}_F\\
        =& \norm{\norm{w+u}(I+\overline{w+u}\ \overline{w+u}^\top) - \norm{w}(I+\bar{w}\bar{w}^\top)}_F\\
        \leq& \norm{(\norm{w+u} - \norm{w})I}_F + \norm{\norm{w+u}\overline{w+u}\ \overline{w+u}^\top - \norm{w}\bar{w}\bar{w}^\top}_F\\
        \leq& \norm{u} + \underbrace{ \norm{\norm{w+u}\overline{w+u}\ \overline{w+u}^\top - \norm{w}\bar{w}\bar{w}^\top}_F }_{I_1}.
    \end{align*}
    Since $\nabla_w g(0) = 0$, when $w=0$ or $w+u=0$, it is clear that $\norm{\nabla_w g(w+u)-\nabla_w g(w)}_F\leq 2\norm{u}$.
    
    For $I_1$, we have
    \begin{align*}
        I_1^2 =& \norm{\norm{w+u}\overline{w+u}\ \overline{w+u}^\top - \norm{w}\bar{w}\bar{w}^\top}_F^2\\
        =& tr\left(\left(\norm{w+u}\overline{w+u}\ \overline{w+u}^\top - \norm{w}\bar{w}\bar{w}^\top\right)^2\right)\\
        =& \norm{w+u}^2+\norm{w}^2 - 2\norm{w+u}\norm{w}\left(\overline{w+u}^\top \bar{w}\right)^2\\
        =& (\norm{w+u} - \norm{w})^2 + 2\norm{w+u}\norm{w}\left(1-\left(\overline{w+u}^\top \bar{w}\right)^2\right)\\
        \leq& \norm{u}^2+ 2\norm{w+u}\norm{w}\left(1-\left(\overline{w+u}^\top \bar{w}\right)^2\right).
    \end{align*}
    Denote $\theta=\arccos(\bar{w}^\top\bar{u})$. Then,
    \begin{align*}
        \overline{w+u}^\top \bar{w} =& \frac{\norm{w}^2+\norm{w}\norm{u}\cos\theta}{\norm{w}\norm{w+u}}
        =\frac{\norm{w}+\norm{u}\cos\theta}{\sqrt{\norm{w}^2+\norm{u}^2+2\norm{w}\norm{u}\cos\theta}}\\
        =& \frac{\norm{w}+\norm{u}\cos\theta}{\sqrt{(\norm{w}+\norm{u}\cos\theta)^2+\norm{u}^2\sin^2\theta}}.
    \end{align*}
    
    Hence,
    \begin{align*}
        \norm{w+u}\norm{w}\left(1-\left(\overline{w+u}^\top \bar{w}\right)^2\right)
        = \frac{\norm{w+u}\norm{w}\norm{u}^2\sin^2\theta}{(\norm{w}+\norm{u}\cos\theta)^2+\norm{u}^2\sin^2\theta}
        = \frac{\norm{w}\sin^2\theta}{\norm{w+u}}\norm{u}^2.
    \end{align*}
    It is easy to check $\frac{\norm{w}\sin\theta}{\norm{w+u}}\leq 1$. Thus, we have
    $I_1^2 \leq (1+2\sin\theta)\norm{u}^2\leq 3\norm{u}^2$. Therefore, for any $u$
    \begin{align*}
        \norm{\nabla_w g(w+u)-\nabla_w g(w)}_F\leq (1+\sqrt{3})\norm{u}.
    \end{align*}
\end{proof}

\begin{lemma}\label{lem-expectation}
    Consider $\alpha,\beta\in\mathbb{R}^d$ with $\norm{\alpha}=\norm{\beta}=1$. Denote $\phi=\arccos(\alpha^\top\beta)$, then
    \begin{align*}
        \E_{x\sim N(0,I)}[(\beta^\top x)^2\mathbb{I}_{\sgn(\alpha^\top x)\neq\sgn(\beta^\top x)}] &= \frac{1}{\pi}(\phi - \sin\phi\cos\phi)=\Theta(\phi^3).
    \end{align*}
\end{lemma}

\begin{proof}
    WLOG, assume $\beta=(1,0,\ldots,0)^\top$ and $\alpha=(\cos\phi,\sin\phi,0,\ldots,0)^\top$. We have
    \begin{align*}
        &\E_{x\sim N(0,I)}[(\beta^\top x)^2\mathbb{I}_{\sgn(\alpha^\top x)\neq\sgn(\beta^\top x)}] \\
        =& \E_{x\sim N(0,I)}[x_1^2\mathbb{I}_{\sgn(x_1\cos\phi+x_2\sin\phi)\neq\sgn(x_1)}]\\
        =& \frac{1}{2\pi}\int_0^\infty r^3 e^{-r^2/2}\dif r \int_0^{2\pi} \cos^2\theta\mathbb{I}_{\sgn(\cos\theta\cos\phi+\sin\theta\sin\phi)\neq\sgn(\cos\theta)}\dif\theta\\
        =& \frac{1}{2\pi}\cdot 2 \cdot 2\int_{\pi/2}^{\pi/2+\phi} \cos^2\theta\dif\theta\\
        =& \frac{2}{\pi} \cdot \frac{1}{2}(\phi-\sin\phi\cos\phi)
        = \frac{1}{\pi}(\phi-\sin\phi\cos\phi)=\Theta(\phi^3).
    \end{align*}
\end{proof}

\begin{lemma}\label{lem-calculation-2}
    Consider $\alpha,\beta\in\mathbb{R}^d$ with $\norm{\alpha}=\norm{\beta}=1$. We have
    \begin{align*}
        \norm{\E_{x\sim N(0,I)}\left[|\alpha^\top x| (I+\beta\beta^\top) x\sgn(\beta^\top x)\right]}^2=O(1).
    \end{align*}
\end{lemma}

\begin{proof}
    WLOG, assume $\beta=(1,0,\ldots,0)^\top$ and $\alpha=(\alpha_1,\alpha_2,0,\ldots,0)^\top$.
    
    We have
    \begin{align*}
        &\norm{\E_{x\sim N(0,I)}\left[|\alpha^\top x|(I+\beta\beta^\top) x\sgn(\beta^\top x)\right]}^2\\
        =& \norm{\E_{x\sim N(0,I)}\left[|\alpha_1 x_1 + \alpha_2 x_2|(I+\beta\beta^\top) x\sgn(x_1)\right]}^2\\
        =& \left(\E_{x\sim N(0,I)}\left[|\alpha_1 x_1 + \alpha_2 x_2| \cdot 2x_1 \cdot \sgn(x_1)\right]\right)^2
        +\left(\E_{x\sim N(0,I)}\left[|\alpha_1 x_1 + \alpha_2 x_2| \cdot x_2 \cdot \sgn(x_1)\right]\right)^2\\
        \leq& \E_{x\sim N(0,I)}\left[\left(|\alpha_1 x_1 + \alpha_2 x_2| \cdot 2x_1 \cdot \sgn(x_1)\right)^2\right]
        +\E_{x\sim N(0,I)}\left[\left(|\alpha_1 x_1 + \alpha_2 x_2| \cdot x_2 \cdot \sgn(x_1)\right)^2\right]\\
        \leq& 4\E_{x\sim N(0,I)}\left[|\alpha_1 x_1 + \alpha_2 x_2|^2 \cdot (x_1^2+x_2^2)\right]\\
        \leq& 4\E_{x\sim N(0,I)}\left[(x_1^2+x_2^2)^2\right]\\
        =& O(1),
    \end{align*}
    where we use Jensen's inequality in the first inequality.
\end{proof}

\begin{lemma}\label{lem-calculation-3}
Consider $\alpha_i\in\mathbb{R}^d$ for $i\in[n]$. We have
    \begin{align*}
        \E_{x\sim N(0,I)}\left[\left(\sum_{i=1}^{n}|\alpha_i^\top x|\right)^2\right]
        \leq c_0 \left(\sum_{i=1}^{n}\norm{\alpha_i}\right)^2,
    \end{align*}
    where $c_0$ is a constant.
\end{lemma}

\begin{proof}
    We have
    \begin{align*}
        \E_{x}\left[\left(\sum_{i=1}^{n}|\alpha_i^\top x|\right)^2\right]
        =\E_{x}\left[\sum_{i,j=1}^{n}|\alpha_i^\top x||\alpha_j^\top x|\right]
        \leq \sum_{i,j=1}^{n} c_0\norm{\alpha_i}\norm{\alpha_j}
        = c_0 \left(\sum_{i=1}^{n}\norm{\alpha_i}\right)^2.
    \end{align*}
    To prove the first inequality above, it suffices to prove the following
    \begin{align*}
        \E_{x\sim N(0,I)}\left[|\alpha_i^\top x||\alpha_j^\top x|\right]
        \leq c_0\norm{\alpha_i}\norm{\alpha_j}.
    \end{align*}
    Note that the LHS above only depends on two vectors, so it is equivalent to consider the expectation over a 2-dimensional Gaussian $\tilde{x}$. We have
    \begin{align*}
        \E_{x\sim N(0,I)}\left[|\alpha_i^\top x||\alpha_j^\top x|\right]
        =\E_{\tilde{x}\sim N(0,I)}\left[|\alpha_i^\top \tilde{x}||\alpha_j^\top \tilde{x}|\right]
        \leq \norm{\alpha_i}\norm{\alpha_j}\E_{\tilde{x}}\left[ \norm{\tilde{x}}^2\right]
        = c_0\norm{\alpha_i}\norm{\alpha_j},
    \end{align*}
    where $c_0$ is a constant.
\end{proof}

%%%%%%%%%%%%%%%%%%%%%%%%%%%%%%%%%%%%%%%%%%%%%%%%%%%%%%%%%%%%%%%%%
% initialization
%%%%%%%%%%%%%%%%%%%%%%%%%%%%%%%%%%%%%%%%%%%%%%%%%%%%%%%%%%%%%%%%%
\section{Initialization}\label{sec-pf-init}
We present the details of two initialization algorithms: (1) random initialization (Algorithm~\ref{alg:init-1}) and (2) subspace initalization (Algorithm~\ref{alg:init-2}) and prove their correctness respectively in the following two subsections. At the end of this section, we give the proof of Theorem~\ref{thm:init}.

\subsection{Random Initialization (Algorithm \ref{alg:init-1})}
Random Initialization (Algorithm \ref{alg:init-1}) initializes the direction of neurons randomly and adjust the norm of neurons by least-squares. In the proof we show that as long as every teacher neuron has at least one close student neuron in direction, the solution returned by least-squares will be small.

\begin{algorithm}[ht]
    \caption{Random Initialization}\label{alg:init-1}
    \KwIn{number of student neurons $m$.}
    
    {\Indp
    Initialize student neurons as $w_1,w_2,\ldots,w_m\sim N(0,I_d)$
    
    Set new student neurons $W' =(w'_1, \ldots, w'_m)$ as $w'_i = \sqrt{z_i^* /\norm{w_i}} \cdot w_i$ where $\{z_i^*\}_{i=1}^m$ is the solution of the following least-squares problem
    \begin{equation*} %\label{eq:initialization}
        \min_{\{z_i\}_{i=1}^m: \forall i \in [m], z_i\geq 0} \frac{1}{2}
        \E_{x}\left[\left(\sum_{i=1}^m z_i|w_i^\top x| - \sum_{i=1}^r|w_i^{*\top} x| \right)^2\right]
    \end{equation*}
    }
    
    \KwOut{$W^{\prime}$}
\end{algorithm}

\begin{lemma}[Random Initialization]\label{lem-high-dim-init-1}
Under Assumption \ref{as:norm}, for any $\epsilon_{init}>0$, if we use Algorithm \ref{alg:init-1} with $m\geq m_0=O\left((rw_{max}/\sqrt{\epsilon_{init}})^d \cdot r\log(1/\delta)\right)$, then with probability $1-\delta$, we have $L(W') \le \epsilon_{init}$.
\end{lemma}
\begin{proof}
    For simplicity, we use $\epsilon$ instead of $\epsilon_{init}$ in the proof.
    Recall that after random initialization, we adjust the norm of student neurons by
    setting new student neurons $W' =(w'_1, \ldots, w'_m)$ as $w'_i = \sqrt{z_i^* /\norm{w_i}} \cdot w_i$ where $\{z_i^*\}_{i=1}^m$ is the solution of the following least-squares problem
    \begin{equation*} %\label{eq:initialization}
        \min_{\{z_i\}_{i=1}^m: \forall i \in [m], z_i\geq 0} \frac{1}{2}
        \E_{x}\left[\left(\sum_{i=1}^m z_i|w_i^\top x| - \sum_{i=1}^r|w_i^{*\top} x| \right)^2\right].
    \end{equation*}
    Therefore, it is equivalent to consider $w_1,w_2,\ldots,w_m\sim \text{Uniform}(S_{d-1})$. First, we are going to show that for every teacher neuron $w_i^*$, with probability $1-\delta$ and with $O(\gamma^{-d}\log\frac{1}{\delta})$ neurons, at least one of these neurons satisfy $\angle(w,w_i^*)\leq \gamma$.
    
    Note that for a single neuron $w$, 
    \begin{align*}
        \Prob(\angle(w,w_i^*)\leq \gamma)=\frac{2\int_0^\gamma \sin^{d-2}\theta_1\dif \theta_1}{\int_0^\pi \sin^{d-2}\theta_1\dif \theta_1}=\Omega(\gamma^d).
    \end{align*}
    Hence, with $O(\gamma^{-d}\log\frac{1}{\delta})$ neurons, the probability of none of these neurons satisfy $\angle(w,w_i^*)\leq \gamma$ is at most $(1-\Omega(\gamma^d))^{O(\gamma^{-d}\log\frac{1}{\delta})}\leq \delta$. Therefore, with $m\geq m_0=O(r\gamma^{-d}\log\frac{1}{\delta})$ neurons, we guarantee that with probability $1-\delta$, every teacher neuron has at least one student neuron satisfy $\angle(w_i,w_i^*)\leq\gamma$.
    
    Then, we show that if every teacher neuron has at least one student neuron satisfy $\angle(w_i,w_i^*)\leq \gamma$, then after doing least square to fit the norms $\norm{w_i}$, the loss $L(W)\leq O(r^2w_{max}^2\gamma^2)$. We prove this by constructing a feasible solution that satisfy the constraints.
    
    Denote student neuron $w_i$ as the one that satisfies $\angle(w_i,w_i^*)\leq\gamma$. Let $z_i=\norm{w_i^*}/\norm{w_i}$ for $i\in[r]$ and $z_i=0$ otherwise. Consider the least-squares objective under this case, we have
    \begin{align*}
        L(W) =& \frac{1}{2}\E_{x}\left[\left(\sum_{i=1}^m z_i|w_i^\top x| - \sum_{i=1}^r|w_i^{*\top} x| \right)^2\right]
        = \frac{1}{2}\E_{x}\left[\left(\sum_{i=1}^r \norm{w_i^*}\left(|\bar{w}_i^\top x| - |\bar{w}_i^{*\top} x| \right)\right)^2\right]\\
        \leq& \frac{1}{2}\E_{x}\left[\left(\sum_{i=1}^r \norm{w_i^*}\left|\left(\bar{w}_i - \bar{w}_i^*\right)^{\top} x\right| \right)^2\right]\\
        \leq& \frac{c_0}{2} \left(\sum_{i=1}^r \norm{w_i^*}\norm{\bar{w}_i-\bar{w}_i^*}\right)^2
        =O(r^2w_{max}^2\gamma^2),
    \end{align*}
    where we use Lemma \ref{lem-calculation-3} and $\angle(w_i,w_i^*)\leq\gamma$ in the last line.
    
    Therefore, let $\gamma=O(r^{-1}w_{max}^{-1}\epsilon^{1/2})$, we know after least square $L(W)\leq\epsilon$. In this case, we need $m\geq m_0=O\left(r\left(\frac{rw_{max}}{\epsilon^{1/2}}\right)^{d}\log\frac{1}{\delta}\right)$ neurons at initialization.
\end{proof}

\subsection{Subspace Initialization (Algorithm \ref{alg:init-2})}
Subspace Initialization (Algorithm \ref{alg:init-2}) first uses samples to estimate the subspace space spanned by the teacher neuron. Then following the same argument in Random Initialization (Algorithm \ref{alg:init-1}), if we random initializes the direction of neurons in this subspace and adjust the norms by least-squares, we could have a small initial loss.

\begin{algorithm}[ht]
    \caption{Subspace Initialization}\label{alg:init-2}
    \KwIn{number of student neurons $m$, number of teacher neuron $r$, $N$ samples $(x_i,y_i)$ where $x_i\sim N(0,I)$ and $y_i=\sum_{i=1}^r |w_i^{*\top}x_i|$.}
    
    {\Indp Let $\widehat{M}=\frac{1}{N}\sum_{i=1}^N y_i (x_ix_i^\top - I)$.
    
    Let $S$ be the subspace spanned by the top-$r$ eigenvectors of $\widehat{M}$ and $Q\in\mathbb{R}^{d\times r}$ be a matrix formed by an orthonormal basis of $S$. 
    
    Initialize student neurons as $w_1,w_2,\ldots,w_m\sim N(0,QQ^\top)$.
    
    Set new student neurons $W' =(w'_1, \ldots, w'_m)$ as $w'_i = \sqrt{z_i^* /\norm{w_i}} \cdot w_i$ where $\{z_i^*\}_{i=1}^m$ is the solution of the following least-squares problem
    \begin{equation*} %\label{eq:initialization}
        \min_{\{z_i\}_{i=1}^m: \forall i \in [m], z_i\geq 0} \frac{1}{2}
        \E_{x}\left[\left(\sum_{i=1}^m z_i|w_i^\top x| - \sum_{i=1}^r|w_i^{*\top} x| \right)^2\right]
    \end{equation*}
    }
    
    \KwOut{$W^{\prime}$}
\end{algorithm}

We need the following results for our second initialization that show the span of top eigenvectors of $M$ is approximately the span of teacher neurons. Recall that $f^*(x)=\sum_{i=1}^r|w_i^{*\top}x|$.
\begin{lemma} [Claim 5.2, \citep{zhong2017recovery}, with our notation]\label{lem-init-2-1}
\begin{equation}\label{eq:init-2-matrix}
\begin{aligned}
    M\triangleq\E_{x\sim N(0.I)}\left[f^*(x)\left(xx^\top-I\right)\right] 
    = c\sum_{i=1}^r \norm{w_i^*}\bar{w}_i^*\bar{w}_i^{*\top}
\end{aligned}
\end{equation}
\end{lemma}

\begin{lemma}[Lemma E.2, \citep{zhong2017recovery}, with our notation]\label{lem-init-2-2}
For $M$ defined in \eqref{eq:init-2-matrix}, denote
$\widehat{M}=\frac{1}{N}\sum_{i=1}^N f^*(x_i) (x_ix_i^\top-I)$. If there are $N=\widetilde{O}(d/\epsilon^2)$ samples $(x_i,y_i)$, where $y_i=f^*(x_i)$, then with probability $1-\delta$ we have 
\begin{align*}
    \norm{\widehat{M} - M}_2=O( rw_{max}\epsilon).
\end{align*}
\end{lemma}

\begin{lemma}[Lemma 3.6, \citep{diakonikolas2020algorithms}, with our notation]\label{lem-init-2-approx-subspace}
For $M$ defined in \eqref{eq:init-2-matrix}, let $\widehat{M}\in \mathbb{R}^{d\times d}$ be a matrix such that $\norm{M-\widehat{M}}_2 \leq\epsilon$ and let $\mathcal{V}$ be the subspace of $\mathbb{R}^d$ that is spanned by the top-$r$ eigenvectors of $\widehat{M}$. There exist $r$ vectors $v^{(i)}\in\mathcal{V}$ such that 
\begin{align*}
    \E_{x\sim N(0,I)} \left[\left(\sum_{i=1}^r|v^{(i)\top}x|-\sum_{i=1}^r|w_i^{*\top}x|\right)^2\right] = O(r^2w_{max}\epsilon).
\end{align*}
\end{lemma}

Now we are ready to prove the lemma for Subspace Initialization.
\begin{lemma}[Subspace Initialization]
\label{lem-high-dim-init-2}

Under Assumption \ref{as:norm}, for any $\epsilon_{init}>0$, if we use Algorithm \ref{alg:init-2} with $\epsilon_{init}\leq O(r^2w_{max}^2)$, $N\geq N_0=\widetilde{O}(dr^6w_{max}^4/\epsilon_{init}^2)$ and $m\geq m_0=O\left((rw_{max}/\sqrt{\epsilon_{init}})^{r}\cdot r\log(1/\delta)\right)$,
then with probability $1-\delta$, we have $L(W') \le \epsilon_{init}$.
\end{lemma}
\begin{proof}
    For simplicity, we use $\epsilon$ instead of $\epsilon_{init}$ in the proof. We use $\widetilde{d}$ to represent an upperbound of $r$. From Lemma \ref{lem-init-2-2} and Lemma \ref{lem-init-2-approx-subspace} and $N\geq N_0=\widetilde{O}(dr^6w_{max}^4/\epsilon^2)$, we know with probability $1-\delta$ there exists $r$ vectors $v^{(i)}\in S$ such that 
    \begin{align*}
        \E_{x\sim N(0,I)} \left[\left(\sum_{i=1}^r|v^{(i)\top}x|-\sum_{i=1}^r|w_i^{*\top}x|\right)^2\right] \leq \epsilon/2.
    \end{align*}
    
    Since $w_1,w_2,\ldots,w_m\sim N(0,QQ^\top)$ and $Q\in\mathbb{R}^{d\times\widetilde{d}}$ is a matrix formed by an orthonormal basis of $S$, we are effectively sampling in $S$. That is, $w_i=Qu_i$ with $u_i\sim N(0,I_{\widetilde{d}})$. Then, following the similar arguments in proof of Lemma \ref{lem-high-dim-init-1}, we know using $m\geq m_0=O\left(r\gamma^{-d}\log\frac{1}{\delta}\right)$ neurons, with probability $1-\delta$ every $v^{(i)}$ has at least one close enough neuron $w_i$ in the sense of $\angle(v^{(i)},w_i)\leq \gamma$ and when $z_i=\norm{v^{(i)}}/\norm{w_i}$ for $i\in[r]$ and $z_i=0$ otherwise, we have
    \begin{align*}
        \frac{1}{2}\E_{x}\left[\left(\sum_{i=1}^m z_i|w_i^\top x| - \sum_{i=1}^r|v^{(i)\top} x| \right)^2\right]
        =\left(\sum_{i=1}^r\norm{v^{(i)}}\right)^2O(\gamma^2).
    \end{align*}
    
    By Lemma \ref{lem-sum-norm-bound}, we know $\sum_{i=1}^r \norm{v^{(i)}}=O(rw_{max})$. Therefore, let $\gamma=O(r^{-1}w_{max}^{-1}\epsilon^{1/2})$, we have
    \begin{align*}
        \E_{x}\left[\left(\sum_{i=1}^m z_i|w_i^\top x| - \sum_{i=1}^r|v^{(i)\top} x| \right)^2\right]
        =O(r^2w_{max}^2\gamma^2)\leq \epsilon/2.
    \end{align*}
    
    Hence, the loss after least square is bounded by
    \begin{align*}
        L(W)\leq \E_{x\sim N(0,I)} \left[\left(\sum_{i=1}^r|v^{(i)\top}x|-\sum_{i=1}^r|w_i^{*\top}x|\right)^2\right]+\E_{x}\left[\left(\sum_{i=1}^m z_i|w_i^\top x| - \sum_{i=1}^r|v^{(i)\top} x| \right)^2\right]
        \leq \epsilon.
    \end{align*}
    This implies when  $m\geq m_0=O\left(r\left(\frac{rw_{max}}{\epsilon^{1/2}}\right)^{\widetilde{d}}\log\frac{1}{\delta}\right)$, the loss after least square is at most $\epsilon$.
\end{proof}

\subsection{Proof of Theorem~\ref{thm:init}}
Now we are ready to prove Theorem~\ref{thm:init} given the above two initialization procedures.
\thmInit*

\begin{proof}
    Set $\epsilon_{init}=\epsilon_0=poly(\Delta,r^{-1},w_{max}^{-1},w_{min})$. By Lemma \ref{lem-high-dim-init-1} and Lemma \ref{lem-high-dim-init-2}, we know the initial loss returned by Algorithm \ref{alg:init-1} and Algorithm \ref{alg:init-2} satisfies $L(W^\prime)\leq \epsilon_0$. Then, combining Theorem \ref{lem-repara-high-dim-converge} we finishes the proof.
\end{proof}

\subsection{Student-Teacher Neuron Matching and the Lottery Ticket Hypothesis}\label{sec:lottery}
 In this subsection, we give a more detailed discussion about the connection between student-teacher neuron matching as indicated in our local convergence result and the lottery ticket hypothesis \citep{frankle2018lottery}. The lottery ticket hypothesis suggests that one can prune a neural network such that even if we only train the neural network with a small subset of randomly initialized neurons (the pruned network), we can still have a model with good performance. Our local convergence result shows if one can maintain at least one student neuron close to every teacher neuron at every step, gradient descent will eventually converge to the global optimal solution and recover the teacher neurons. This gives a partial explanation for lottery ticket hypothesis int the two-layer teacher/student setting \--- as long as the initialization contains student neurons that are close to each teacher neuron and one can prune away neurons that were not close to teacher neurons at initialization, then the training process will converge to the global optima. Intuitively, the two initialization algorithms in Section~\ref{sec-pf-init} use least-squares to prune away the useless neurons given enough number of student neurons at initialization.

%%%%%%%%%%%%%%%%%%%%%%%%%%%%%%%%%%%%%%%%%%%%%%%%%%%%%%%%%%%%%%%%%%%%%%%%%
% sample complexity
%%%%%%%%%%%%%%%%%%%%%%%%%%%%%%%%%%%%%%%%%%%%%%%%%%%%%%%%%%%%%%%%%%%%%%%%%%
\section{Sample Complexity}\label{sec-pf-sample-complexity}
Recall we have the empirical loss
\begin{align*}
    \widehat{L}(W) = \frac{1}{2N} \sum_{k=1}^N \left(\sum_{i=1}^r\sum_{j\in T_i} \norm{w_j}|w_j^\top x_k| - \sum_{i=1}^r|w_i^{*\top} x_k|\right)^2,
\end{align*}
where $N$ sample $\{(x_k,y_k)\}$ are i.i.d. sampled from $x_k\sim N(0,I)$ and $y_k=f^*(x_k)=\sum_{i=1}^r|w_i^{*\top} x_k|$.

We first give the following concentration result that shows when the number of sample is large enough, the gradient on empircal loss is close to the gradient on population loss. 

\begin{restatable}{lemma}{lemSampleComplexity}
\label{lem-sample-complexity}
Under Assumption \ref{as:norm}, for any fixed $W$, if loss $L(W)=O(r^2w_{max}^2)$ and the number of data $N\geq  O\left(r^3w_{max}^3d^2\epsilon^{-2}\delta^{-1}\right)$, with probability $1-\delta$ we have
\begin{align*}
    \norm{\nabla_W \widehat{L}(W) - \nabla_W L(W)}_F\leq \epsilon.
\end{align*}
\end{restatable}

%\lemSampleComplexity*

\begin{proof}    
    We have
    \begin{align*}
        &\Prob\left(\norm{\nabla_W \widehat{L}(W) - \nabla_W L(W)}_F\geq \epsilon\right)\\
        \leq& \epsilon^{-2}\E_{x_1,x_2,\ldots,x_N}\left[\norm{\nabla_W \widehat{L}(W) - \nabla_W L(W)}_F^2\right]
        = \epsilon^{-2}\sum_{j=1}^m\E_{x_1,x_2,\ldots,x_N}\left[\norm{\nabla_{w_j} \widehat{L}(W) - \nabla_{w_j} L(W)}^2\right]\\
        =& \epsilon^{-2}N^{-1}\sum_{j=1}^m\E_{x}\left[\norm{\left(\sum_{i=1}^m \norm{w_i}|w_i^\top x| - \sum_{i=1}^r|w_i^{*\top} x|\right) \norm{w_j}(I+\bar{w}_j\bar{w}_j^\top)x\sgn(w_j^\top x) - \nabla_{w_j} L(W)}^2\right]\\
        =& 2\epsilon^{-2}N^{-1}\sum_{j=1}^m\E_{x}\left[\norm{\left(\sum_{i=1}^m \norm{w_i}|w_i^\top x| - \sum_{i=1}^r|w_i^{*\top} x|\right)\norm{w_j}(I+\bar{w}_j\bar{w}_j^\top)x\sgn(w_j^\top x)}^2\right]\\
        &+ 2\epsilon^{-2}N^{-1}\sum_{j=1}^m\norm{\nabla_{w_j} L(W)}^2.
    \end{align*}
    
    We now bound the first term above. We have
    \begin{align*}
        &\E_{x}\left[\norm{\left(\sum_{i=1}^m \norm{w_i}|w_i^\top x| - \sum_{i=1}^r|w_i^{*\top} x|\right)\norm{w_j}(I+\bar{w}_j\bar{w}_j^\top)x\sgn(w_j^\top x)}^2\right]\\
        =& 4\norm{w_j}^2\E_{x}\left[\norm{x}^4\left(\sum_{i=1}^m \norm{w_i}|w_i^\top \bar{x}| - \sum_{i=1}^r|w_i^{*\top} \bar{x}|\right)^2\right]\\
        \leq& 4\norm{w_j}^2\E_{x}\left[\norm{x}^4\left(\sum_{i=1}^m \norm{w_i}^2 + \sum_{i=1}^r\norm{w_i^*}\right)^2\right]\\
        =& O(r^2w_{max}^2d^2)\norm{w_j}^2,
    \end{align*}
     where in the last line we use Lemma \ref{lem-sum-norm-bound}, which gives $\sum_{i=1}^m \norm{w_i}^2=O(rw_{max})$. Hence, use Lemma \ref{lem-sum-norm-bound} again, we upper bound the first term by 
     $O(\epsilon^{-2}N^{-1}r^3w_{max}^3d^2)$.

    For the second term, from Lemma \ref{lem-grad-upper-bound}, we know
    \begin{align*}
        \sum_{j=1}^m\norm{\nabla_{w_j} L(W)}^2=O(r^3w_{max}^3).
    \end{align*}
    Therefore, we have
    \begin{align*}
        \Prob\left(\norm{\nabla_W \widehat{L}_2(W) - \nabla_W L_2(W)}\geq \epsilon\right)=O\left(\frac{r^3w_{max}^3d^2}{\epsilon^2 N}\right).
    \end{align*}
    By our choice of $N$, this finishes the proof.
\end{proof}

Given the above concentration result, we are ready to prove Theorem~\ref{thm-sample}. The proof is very similar with the proof of Theorem \ref{lem-repara-high-dim-converge}.

\thmSample*

\begin{proof}
    By Lemma \ref{lem-grad-upper-bound}, we know
    \begin{align*}
        &\norm{\nabla_W L(W)}_F^{2}
        =O(r^{3}w_{max}^3).
    \end{align*}
    
    By Theorem \ref{lem-repara-high-dim-grad-norm}, we know
    \begin{align*}
        \norm{\nabla_W L(W)}_F
        \geq \kappa L(W),
    \end{align*}
    where $\kappa=\Theta(r^{-1/2}w_{max}^{-1/2})$.
    
    Now, using Theorem \ref{lem-high-dim-smooth} with $u_{j}=-\eta\nabla_{w_{j}}\widehat{L}(W)$ , we have
    \begin{align*}
        L(W+U)
        \leq& L(W) -\eta \langle\nabla_W L(W),\nabla_W \widehat{L}(W)\rangle + O(\eta^{3/2}r^{1/4}w_{max}^{1/4})L^{1/2}(W)\norm{\nabla_W \widehat{L}(W)}_F^{3/2}\\
        &+ O(\eta^2rw_{max})\norm{\nabla_W\widehat{L}(W)}_F^{2} + O(\eta^4)\norm{\nabla_W \widehat{L}(W)}_F^4.
    \end{align*}
    
    Using Lemma \ref{lem-sample-complexity} with $N\geq O(r^3w_{max}^3d^2\epsilon_g^{-2}\delta_g^{-1})$ samples where $\epsilon_g^2\leq \frac{1}{4}\norm{\nabla_W L(W)}^2$, then with probability $1-\delta_g$ we have
    $\norm{\nabla_W \widehat{L}(W)-\nabla_W L(W)}_F\leq\epsilon_g$. Hence, 
    \begin{align*}
        \langle\nabla_W L(W),\nabla_W\widehat{L}(W)\rangle 
        =& \norm{\nabla_W L(W)}_F^2
        +\langle\nabla_W L(W),\nabla_W\widehat{L}(W)-\nabla_W L(W)\rangle\\
        \geq& \norm{\nabla_W L(W)}_F^2
        - \norm{\nabla_W L(W)}_F\norm{\nabla_W \widehat{L}(W)-\nabla_W L(W)}_F\\
        \geq& \frac{1}{2}\norm{\nabla_W L(W)}_F^2,
    \end{align*}
    \begin{align*}
        \norm{\nabla_W\widehat{L}(W)}_F^{3/2}
        \leq& \left(\norm{\nabla_W L(W)}_F + \norm{\nabla_W \widehat{L}(W)-\nabla_W L(W)}_F\right)^{3/2}
        \leq O(1)\norm{\nabla_W L(W)}_F^{3/2},\\
        \norm{\nabla_W \widehat{L}(W)}_F^{2}
        \leq& 2\norm{\nabla_W L(W)}_F^2 + 2\norm{\nabla_W \widehat{L}(W)-\nabla_W L(W)}_F^2
        \leq O(1) \norm{\nabla_W L(W)}_F^2.
    \end{align*}
    
    With the above bounds, we have
    \begin{align*}
        L(W+U)
        \leq& L(W) -\left(\eta - O\left(\eta^{3/2}r^{1/4}w_{max}^{1/4}\kappa^{-1/2} + \eta^2rw_{max} + \eta^4r^3w_{max}^3\right)\right) \norm{\nabla_W L(W)}_F^2\\
        \leq& L(W) -\frac{\eta\kappa^2}{4} L^2(W),
    \end{align*}
    where in the last line we use $\eta\leq\eta_0=O(\min\{r^{-1/2}w_{max}^{-1/2}\kappa,r^{-1}w_{max}^{-1}\})$. 
    This implies that with probability $1-\delta_g$
    \begin{align*}
        L(W^{(t+1)})\leq L(W^{(t)}) -\frac{\eta\kappa^2}{4} L^2(W^{(t)}).
    \end{align*}
    Then, note that $0 < L(W^{(t+1)})\leq L(W^{(t)})$, we could have
    \begin{align*}
        \frac{1}{L(W^{(t)})} \leq \frac{1}{L(W^{(t+1)})} - \frac{\eta\kappa^2}{4}\frac{L(W^{(t)})}{L(W^{(t+1)})}
        \leq \frac{1}{L(W^{(t+1)})} - \frac{\eta\kappa^2}{4}. 
    \end{align*}
    This implies that
    \begin{align*}
        0<\frac{1}{L(W^{(0)})} 
        \leq \frac{1}{L(W^{(t)})} - \frac{t\eta\kappa^2}{4},
    \end{align*}
    which leads to
    \begin{align*}
        L(W^{(t)})\leq \min\left\{\frac{4}{t\eta\kappa^2},L(W^{0})\right\}=\min\left\{O\left(\frac{r w_{max}}{t\eta}\right),\epsilon_0\right\}.
    \end{align*}
    Therefore, with probability $1-T\delta_g$, we have $L(W^{(T)})\leq \epsilon$ in $T=O\left(\frac{r w_{max}}{\epsilon\eta}\right)$. Setting $\delta_g=\delta/T$ and $\epsilon_g=\frac{1}{2}\kappa\epsilon\leq  \frac{1}{2}\norm{\nabla_W L(W)}_F$ finishes the proof.
\end{proof}

\end{document}